\documentclass{CPP}
\pdfoutput=1
\usepackage{bm}
\usepackage{amsmath}
\usepackage{amssymb}
\usepackage{amsthm} 
\usepackage{mathcomp}
\usepackage{graphicx}
\usepackage{booktabs}
\usepackage{multirow}
\usepackage[page]{appendix}
\usepackage{url}
\usepackage{caption}
\usepackage[intoc]{nomencl}
\usepackage{tikz}
\usepackage{tikzducks}
\usepackage[lined,boxed,algochapter]{algorithm2e}
\usepackage[numbers]{natbib}
\usetikzlibrary{matrix,chains,positioning,decorations.pathreplacing,arrows}
\usepackage[pdftex]{hyperref}

\newcommand{\cov}{\text{cov}}
\newcommand{\Sd}{\mathbb{S}^{d-1}}
\newcommand{\Rd}{\mathbb{R}^d}
\newcommand{\x}{\mathbf{x}}
\newcommand{\calX}{\mathcal{X}}
\newcommand{\R}{\mathbb{R}}
\newcommand{\Ss}{\mathbb{S}}
\newcommand{\N}{\mathbb{N}}
\newcommand{\calH}{\mathcal{H}}
\newcommand{\MVN}{\mathcal{MVN}}
\newcommand{\kGaus}{k_{Gaus}}

\DeclareMathOperator*{\argmin}{\arg\min}

\newtheorem{definition}{Definition}[section]
\newtheorem{theorem}{Theorem}[section]

\newtheorem{prop}{Proposition}

\makenomenclature
\makeindex
\hypersetup{
    colorlinks   = true, 
    urlcolor     = blue, 
    linkcolor    = blue, 
    citecolor   = red 
}

\SetAlCapSkip{1.73\abovecaptionskip}
\SetAlCapFnt{\small}
\SetAlCapNameFnt{\small\setlinespacing{1.1}}

\counterwithout{footnote}{chapter}

\begin{document}

    \titleone{An Empirical Analysis of the}
    \titletwo{Laplace and Neural Tangent Kernels}
    \titlethree{}
    \doctype{Thesis}
    \doctypeUp{Thesis}
    \degree{Master of Science}
    \field{Mathematics}
    \Author{Ronaldas Paulius Lencevi\v{c}ius}
    \Advisor{James Risk}
    \MemberA{Manuchehr Aminian}
    \MemberB{Adam King}
    \Semester{Summer}
    \Year{2022}
    
    \Abstract{
        The neural tangent kernel is a kernel function defined over the parameter distribution of an infinite width neural network.  Despite the impracticality of this limit, the neural tangent kernel has allowed for a more direct study of neural networks and a gaze through the veil of their black box.  More recently, it has been shown theoretically that the Laplace kernel and neural tangent kernel share the same reproducing kernel Hilbert space in the space of $\Sd$ alluding to their equivalence.  In this work, we analyze the practical equivalence of the two kernels.  We first do so by matching the kernels exactly and then by matching posteriors of a Gaussian process.  Moreover, we analyze the kernels in $\Rd$ and experiment with them in the task of regression.
    }
    
    \Acknowledgments{
        \vfill
        \begin{center}
            \textit{To Ada, Percy, and Sabrina,} \\
            \textit{to Raimis, Diana, Boson, and Curie,} \\
            \textit{to all my friends, family, and faculty,} \\
            \textit{...and to Dr. Risk's care, patience, and mentorship.} \\ 
            \vspace{1em}
            \begin{tikzpicture}
                \duck[
                    scale=0.5,
                    graduate=gray!20!black,
                    tassel=yellow
                ]
            \end{tikzpicture}
        \end{center}
        \vfill
    }
    
    \titlepage
    \signaturepage{\addcontentsline{toc}{chapter}{Signature Page}}
    \acknowledgmentspage{\addcontentsline{toc}{chapter}{Acknowledgements}}
    \abstractpage{\addcontentsline{toc}{chapter}{Abstract}}
    \tableofcontents{}

    \listoftables{\addcontentsline{toc}{chapter}{List of Tables}}
    \listoffigures{\addcontentsline{toc}{chapter}{List of Figures}}
    \listofalgorithms{\addcontentsline{toc}{chapter}{List of Algorithms}}
    \printnomenclature[2.5cm]
    
    \nomenclature[1]{$\Rd$}{$d$-dimensional real space}
    \nomenclature[1]{$\Sd$}{Unit $d$-sphere space}
    \nomenclature[1]{$\mathcal{X}$}{Data space}
    \nomenclature[1]{$\mathcal{H}$}{Hilbert space}
    \nomenclature[1]{$\mathcal{H}_0$}{Pre-Hilbert space}
    \nomenclature[1]{$\mathcal{H}_k$}{Reproducing kernel Hilbert space (RKHS)}
    \nomenclature[1]{$L_2$}{Square-integrable function space}
    \nomenclature[2]{$\mathcal{GP}$}{Gaussian process (GP)}
    \nomenclature[2]{$\mathcal{N}$}{Normal}
    \nomenclature[2]{$\mathcal{MVN}$}{Multivariate normal}
    \nomenclature[2]{$\mu$}{Metric}
    \nomenclature[3]{$\mathbf{x}$}{Vector}
    \nomenclature[3]{$x$}{Scalar}
    \nomenclature[3]{$X$}{Matrix}
    \nomenclature[3]{\textvisiblespace$_*$}{Test/predictions}
    \nomenclature[4]{$\ell$}{Length-scale parameter (Mat\'ern kernel)}
    \nomenclature[4]{$\nu$}{Smoothness parameter (Mat\'ern kernel)}
    \nomenclature[4]{$\beta$}{Bias parameter (neural tangent kernel)}
    \nomenclature[4]{$D = L+1$}{Depth (neural tangent kernel)}
    \nomenclature[4]{$\theta$}{General parameter set}
    \nomenclature[4]{$\epsilon$}{Noise}
    \nomenclature[4]{$\sigma^2$}{Variance}
    \nomenclature[4]{$\rho$}{Pearson correlation coefficient}
    \nomenclature[4]{$R^2$}{Coefficient of determination}
    \nomenclature[5]{$k$}{Kernel}
    \nomenclature[5]{$\mathbb{E}$}{Expected value}
    \nomenclature[5]{$\cov$/var}{Covariance/variance}
    \nomenclature[5]{$T_k$}{Integral operator}
    \nomenclature[5]{$L$}{Loss function}
    \nomenclature[6]{$\sigma(\cdot)$}{Activation function}
    \nomenclature[6]{$d_\theta$}{Absolute difference function given kernel parameters $\theta$}
    \nomenclature[6]{$\langle\cdot,\cdot\rangle$}{Inner product}
    \nomenclature[6]{$\lVert\,{\cdot}\,\rVert$}{Norm}
    \nomenclature[7]{$\phi_i$}{Eigenfunction}
    \nomenclature[7]{$\lambda_i$}{Eigenvalue}

    \newpage
    \pagenumbering{arabic} \setcounter{page}{1} \thispagestyle{empty}
    
    \AddChap
    
    \chapter{Introduction}
        
    The relation between artificial neural networks and Gaussian processes has been established since the early 1990s.  In 1989 researchers determined that a single layer neural network can, under the right conditions, approximate any continuous function as the layer width tends to infinity \citep{cybenko1989approximation,funahashi1989approximate,stinchombe1989universal,girosi1990networks}.  With this idea, \citet{neal1996bayesian} later found that a single hidden layer neural network with normally distributed weights and biases at initialization can be represented as a closed form Gaussian process when the layer width approach infinity and hence converges to a normal distribution.  It is then possible to inform neural networks via specified priors and analyze them over the neural network parameter distribution.  This was further expanded in 2018 to neural networks of many layers by \citet{lee2018deep}.
        
    This perspective opened new doors to analysis but did not explain the effectiveness of training neural networks using the widely used gradient descent approach. \citet{jacot2018neural} developed the solution to this by further generalizing infinite-width neural networks via a recursively defined kernel called the neural tangent kernel. The neural tangent kernel can be used to represent and analyze a given infinite-width neural network during training with a specific depth, activation, and variance initialization. Future works utilized this kernel and expanded on the 1989 single layer results by showing that wide neural networks trained under gradient descent work as linear models and that empirically, finite networks also share those attributes \citep{lee2019wide}.  In addition, further neural tangent kernel parameterizations were discovered for convolutional, recurrent, and graph neural network architectures \citep{arora2019cntk,alemohammad2021rntk,du2019gntk}. Furthermore, the neural tangent kernel was also shown to generalize with neural networks that allow for regularization and gradient noise during training \citep{chen2020noisentk}. 
        
    During the same time, \citet{belkin2018kernel} empirically found that overfitted kernel methods display a similar phenomenon to overparameterized deep models where, despite reaching zero training loss, the test data would show good performance.  Moreover, their work showed parallels between rectified linear unit (ReLU) activated neural networks and the Laplace kernel in the task of fitting random labels.  Since then it has been shown theoretically that the Laplace and neural tangent kernels do in fact perform similarly to their neural network counterparts since both kernels share the same reproducing kernel Hilbert space $\mathcal{H}_k$ of predictions in the $\Sd$ unit $d$-sphere \citep{chen2021rkhs,geifman2020rkhs}.  This is consequential because the Laplace kernel has a very simple, well understood formulation whereas the neural tangent kernel has an unwieldy recursive formulation.  In essence, the Laplace kernel can be used to analyze deep neural networks without the computational or theoretical difficulties of the neural tangent kernel.
        
    However, this leaves some questions unanswered such as the practical equivalence of the kernels, the ability to find matching elements from the $\calH_k$ of the kernels, whether the kernels share these similarities in the space of $\Rd$, and if the Gaussian kernel, which comes from the same family as the Laplace, provides similar results.  We attempt to provide answers by analyzing the neural tangent, Laplace, and Gaussian kernels via the framework of Gaussian process regression.
    
    Chapter \ref{ch:regression} defines the general data fitting problem, Gaussian processes for regression, and neural networks and their equivalence to Gaussian processes.  Chapter \ref{ch:rkhs} develops the theory for reproducing kernel Hilbert spaces and their relevance to the task of regression.  Chapter \ref{ch:types} introduces the kernels used in our analysis and results showcasing the conditions for which the Laplace and neural tangent kernels can be made equal.  Chapter \ref{ch:synthexp} provides synthetic experiments for matching Gaussian process regression results (i.e.\ posterior means) between the various kernels, comparing the kernels in the space of $\Rd$, and improving regression results for the neural tangent kernel.  Lastly, Chapter \ref{ch:realexp} showcases real world experiments of the kernels using the lessons learned from the previous chapters.  We summarize the contributions of this thesis as follows:
        
    \begin{itemize}
        \item We find empirical evidence for the importance of the neural tangent kernel's bias parameter in equating it to the Laplace kernel.
        \item We derive the partial derivative of the neural tangent kernel with respect to its bias parameter which can be used to optimize the bias during regression fitting.
        \item We develop a method for matching the Laplace and neural tangent kernels' Gaussian process regression posteriors, which are elements of $\mathcal{H}_k$.
        \item We implement a \texttt{Python} programming language package called \href{https://github.com/392781/scikit-ntk}{\texttt{scikit-ntk}} which is an implementation of the neural tangent kernel that can be used directly with the popular \texttt{scikit-learn} machine learning toolkit.
    \end{itemize}

    \chapter{Regression}\label{ch:regression}
    
    In this chapter we will introduce the basic data modeling problem and present two relevant modeling tools: Gaussian process regression and neural networks. We also present the relation between these two tools that allows practitioners to study the black box architecture of neural networks through the lens of Gaussian processes.
    
    \section{Data Fitting Problem}\label{sec:datafitting}
    
    A single output data fitting problem begins with a set of $n$ data points $\{(\mathbf{x}_i, y_i)\ |\ i = 1,\dots, n\}$ where $\mathbf{x}_i = [x_1, \dots, x_d]^\top\in\mathcal{X}\subseteq \Rd$ is a single input vector and $y_i\in\mathbb{R}$ is a output value usually referred to as a target or response\footnote{We will be using response to refer to $y_i$ from here on out.}.  It is important to note that $\mathcal{X}$ can be a metric space but for our purpose it is sufficient to assume that it is a subset of the $d$-dimensional real space. In general, $\mathcal{X}$ may be determined by the type of data used or a transformation that is applied to the data (e.g.\ the unit d-sphere space $\Sd := \{\mathbf{x} \in \Rd : \|\mathbf{x}\| = 1\}$).
    
    We can form an $n \times d$ size matrix $X = [\x_1, \dots, \x_n]^\top$ of $n$ observations and $d$ independent variables and a vector $\mathbf{y} = [y_1, \dots, y_n]^\top$ of dependent outputs which combine into a set $(X, \mathbf{y})$ which we call a \textit{training set}.  The training set can be seen as a snapshot of some phenomenon that relates the training set inputs to response values via some function $f$:
    \begin{equation}\label{eq:dataproblem}
        y = f(\mathbf{x}) + \epsilon
    \end{equation}
    where $\epsilon$ is some additive noise assumed to be independent of $f$ and other noise terms.  Depending on the problem, we may have $\epsilon=0$ which indicates exact observations.  Otherwise, we assume that
    \begin{equation}\label{eq:dataproblemnoise}
        \epsilon\sim \mathcal{N}(0, \sigma^2)
    \end{equation}
    indicating normally distributed noise with mean zero and fixed variance $\sigma^2\in\mathbb{R}$.

    The goal of the data fitting problem is to find a function $f$ that best fits the training set while also best generalizing to unseen data of the phenomenon we are trying to model.  This is done by minimizing an empirical loss functional between the true response values $y_i$ and estimated values $f(\x_i)$:
    \begin{equation}\label{eq:generalloss}
        f_{opt} = \argmin_f \left\{\sum_{i=1}^n L(y_i, f(\mathbf{x}_i))\right\}.
    \end{equation}
    For regression, the loss function is usually absolute error $L(y_i, f(\x_i)) = |y_i - f(\x_i)|$, squared error $L(y_i, f(\x_i)) = (y_i - f(\x_i))^2$, or some variation of either.
    
    Without any restrictions or further assumptions placed on this process it is possible to perfectly interpolate (or ``overfit'') over the training set thus losing any generalization of unseen data.  This is not a desirable outcome since the purpose of the problem is to capture additional information about the underlying phenomenon that generated the training set.  As such, it is reasonable to place further assumptions on things including but not limited to the dataset (i.i.d., transformations), function properties (continuity, time dependence), model type (nonlinear regressor, support vector machine), and model properties (parameterization type, regularization).
    These assumptions allow for the data fitting process to more productively use the training set to explain the underlying phenomenon. An example of a commonly used set of assumptions is a linear parametric model:
    \begin{equation}\label{eq:linearmodel}
        \begin{aligned}
            f(\mathbf{x}) &= \beta_0 + \beta_1 x_1 + \cdots + \beta_d x_d \\
            &= \beta_0 + \bm{\beta}^\top \mathbf{x},
        \end{aligned}
    \end{equation}
    where $\bm{\beta} = [\beta_1, \ldots, \beta_d]^\top$.     With the squared error loss function, Equation \eqref{eq:generalloss} reduces to
    \begin{equation}
        \begin{aligned}
            f_{opt} &= \argmin_f \left\{ \sum_{i=1}^n L(y_i, f(\mathbf{x}_i)) \right\} \\
            &= \argmin_{\beta_0,\dots,\beta_d} \left\{\sum_{i=1}^n \left(y_i - \beta_0 - \bm{\beta}^\top \mathbf{x}_i\right)^2 \right\}.
        \end{aligned}
    \end{equation}
    However, this is a highly restrictive assumption which is not applicable in many cases. 
    
    On the other hand, regularization is a modification that is applicable to all models. Regularization helps prevent overfitting by placing a penalty on the model's objective function (Equation \eqref{eq:generalloss}):
    \begin{equation}\label{eq:penalizedloss}
        f_{opt} = \argmin_f \left\{\sum_{i=1}^n L(y_i, f(\mathbf{x}_i)) + P_\lambda(f)\right\}
    \end{equation}
    where $P_\lambda(\cdot)\in\mathbb{R}$ is a penalization term for a given predictive function $f$ that is scaled by $\lambda\in\mathbb{R}$ and added to the summation.  Effectively, $P_\lambda(\cdot)$ can be seen as a restriction on \textit{smoothness} (i.e.\ properties of differentiability).  The idea is to change the search space of the objective function from all possible models, including sophisticated ones that interpolate through the data, to simpler models that attempt to capture the underlying phenomenon instead of attaining zero training loss.  Penalization for linear models (Equation \eqref{eq:linearmodel}) includes ridge regression where $P_\lambda(f) = \lambda \sum_{j=1}^d \beta_j^2$ and lasso regression where $P_\lambda(f) = \lambda \sum_{j=1}^d |\beta_j|$.  The regularization framework in Equation \eqref{eq:penalizedloss} is very general. For example, if instead we are searching through the space of all functions $f$ with two continuous derivatives and utilizing the residual sum of squares loss, Equation \eqref{eq:penalizedloss} becomes
    \begin{equation}\label{eq:smoothingspline}
        \begin{aligned}
            f_{opt} &= \argmin_f \left\{\sum_{i=1}^n L(y_i, f(\mathbf{x}_i)) + P_\lambda(f)\right\} \\
            &= \argmin_f \left\{\sum_{i=1}^n (y_i - f(\mathbf{x}_i))^2 + \lambda \int f''(\mathbf{x}_i)^2 dx\right\}
        \end{aligned}
    \end{equation}
    with the unique solution being a natural cubic spline \citep[pg. 110]{hastie2017generalized}.  
    
    The data fitting problem is one of compromises. While adding additional assumptions helps by reducing search time and/or limiting the search space, this still results in a complex task since we are going from one set of infinite functions to another. In addition, there may be fundamental issues with our training set since it is just a small sample of the overall phenomenon and thus sensitive to sampling methods and low signal to noise ratio.  As for the model, we are limited by ``no free lunch'' theorem's \citep{wolpert1997no, wolpert1996lack} implication that there is no best way to choose a model or even a best model for a specific task.  There are many ways to approach the data fitting problem so our choices in models, regularization, data augmentation, etc. are then informed by both the task and quality of data. In this paper we focus on Gaussian processes which are a nonparametric regression method highly related to kernel ridge regression.
    
    \section{Gaussian Processes}\label{sec:gp}
    
    At its simplest, a Gaussian process is a set of random variables indexed by time where any finite set composes a multivariate normal distribution. At the heart of a Gaussian process is a positive definite covariance function.
    \begin{definition}[Positive Definite Function]
        Let $k : \mathcal{X} \times \mathcal{X} \to \mathbb{R}$ be a symmetric function.  Given $n\in\mathbb{N}$, inputs $\mathbf{x}_1,\dots, \mathbf{x}_n\in\mathcal{X}$, and constants $c_1, \dots, c_n\in\mathbb{R}$ where $\mathbf{c} = [c_1, \dots, c_n]^\top$ we say that $k$ is positive definite and forms a positive definite matrix $K$ of all pairwise evaluations of the inputs on $k$ if 
        \begin{equation*}
            \sum_{i=1}^n\sum_{j=1}^n c_i c_j k(\mathbf{x}_i, \mathbf{x}_j) = \mathbf{c}^\top K \mathbf{c} \geq 0
        \end{equation*}
    \end{definition}
    Symmetry in the inputs is required in order for a covariance function to be real valued.  Covariance functions are also referred to as kernel functions which will be expanded on in Chapters \ref{ch:rkhs} and \ref{ch:types}.  From here we can define a Gaussian process from a functional context which gives us the ability to build regression models through them \citep{rasmussen2006gpml,kanagawa2018gaussian}.
    \begin{definition}[Gaussian Process]
        \label{def:gp}
        Let $f : \calX \to \R$ be a real valued random function.  We then say that $f(\mathbf{x})$ is a Gaussian process (GP) defined by mean function $m : \calX \to \R$ and positive definite covariance function $k : \calX \times \calX \to \R$ if for any finite set $X = \{\mathbf{x}_1, \dots, \mathbf{x}_n\} \subset \calX$ of any size $n\in\N$ we have the random vector $\mathbf{f}_X = [f(\mathbf{x}_1),\dots, f(\mathbf{x}_n)]^\top \in \mathbb{R}^n$ such that 
        \begin{equation}
            \mathbf{f}_X \sim \mathcal{MVN}(\mathbf{m}_X, K_{XX})
        \end{equation}
        where $\mathbf{m}_X = \left[m(\mathbf{x}_1), \dots, m(\mathbf{x}_n)\right]^\top\in\R^n$ is the mean vector and $K_{XX} = \left[k(\mathbf{x}_i, \mathbf{x}_j)\right]^n_{i,j=1}$ $\in \R^{n\times n}$ is the covariance matrix. We define $m$ and $k$ for inputs $\mathbf{x},\mathbf{z}\in\calX$ as follows:
        \begin{equation}
            \begin{aligned}
                m(\mathbf{x}) &= \mathbb{E}[f(\mathbf{x})] \\
                k(\mathbf{x},\mathbf{z}) &= \mathbb{E}[(f(\mathbf{x}) - m(\mathbf{x}))(f(\mathbf{z}) - m(\mathbf{z}))]
            \end{aligned}
        \end{equation}
        and denote such a process as 
        \begin{equation}
            f(\mathbf{x}) \sim \mathcal{GP}(m(\mathbf{x}), k(\mathbf{x}, \mathbf{z})).
        \end{equation}
    \end{definition}
    The implication of this definition is if a GP exists, then so does a corresponding $m$ and $k$.  However, it is not necessary to begin with a process as shown via the Kolmogorov Existence Theorem \citep[Theorem 12.1.3, pg. 443]{dudley2018real}.
    \begin{theorem}[Kolmogorov Existence Theorem for Gaussian Processes]
        Let $\mathcal{X}$ be any set such that there exists a function $m : \mathcal{X} \to \mathbb{R}$ and a positive definite function $k : \mathcal{X} \times \mathcal{X} \to \mathbb{R}$.  Then there exists a GP on $\mathcal{X}$ with mean function $m$ and covariance function $k$.
    \end{theorem}
    
    As a result, there is a one-to-one relationship between a GP and its mean and covariance functions.  Properties such as continuity, periodicity, and differentiability are tied to the mean and covariance functions which allows for an open view of the inner workings of a model built on a GP.  Furthermore, it allows for users to insert prior knowledge into the data fitting process making GPs a powerful Bayesian modeling tool.
    
    To apply a GP to the task of regression it is necessary to build a joint distribution over a training set and a previously ignored \textit{test set}, $[\x_{*_1}, \dots, \x_{*_m}]^\top =  X_*\subseteq\R^{m\times d}$ where each $\x_{*_i}\in\mathcal{X}\subseteq\Rd$.  The test set mirrors the training set in Section \ref{sec:datafitting} with the only differences being that the test set is usually disjoint from the training set and potentially contains a different number of observations compared to the training set ($m \neq n$).  The test set will be used analyze unseen data $\mathbf{f_*}\in\mathbb{R}^m$ through a posterior distribution conditional on the known $\mathbf{y}$.  First, we choose a covariance function $k$ and build covariance matrices $K$ to inform our data fitting task from Equation \ref{eq:dataproblem} following the same noise assumptions as in Equation \ref{eq:dataproblemnoise}.  We can then build the following joint distribution of the training data and testing data:
    \begin{equation}\label{eq:priordist}
        \begin{bmatrix}
            \mathbf{y} \\
            \mathbf{f_*}
        \end{bmatrix}
        \sim
        \mathcal{MVN}
        \left(
        \begin{bmatrix}
            \bm{0} \\ \bm{0}
        \end{bmatrix}
        ,
        \begin{bmatrix}
            K_{XX} + \sigma^2 I_n & K_{XX_*} \\ 
            K_{X_*X} & K_{X_*X_*}
        \end{bmatrix}
        \right)
    \end{equation}
    where
    \begin{equation}
        \begin{aligned}
            K_{XX} &= \left[k(\x_i, \x_j)\right]^n_{i,j=1} \\
            K_{X_*X_*} &= \left[k(\x_{*_i}, \x_{*_j})\right]^m_{i,j=1} \\
            K_{XX_*} &= \left[k(\x_i, \x_{*_j})\right]^{i=n, j=m}_{i,j=1} \\
            K_{X_*X} &= K_{XX_*}^\top
        \end{aligned}
    \end{equation}
    
    \begin{figure}[t]
        \centering
        \includegraphics[width=\textwidth]{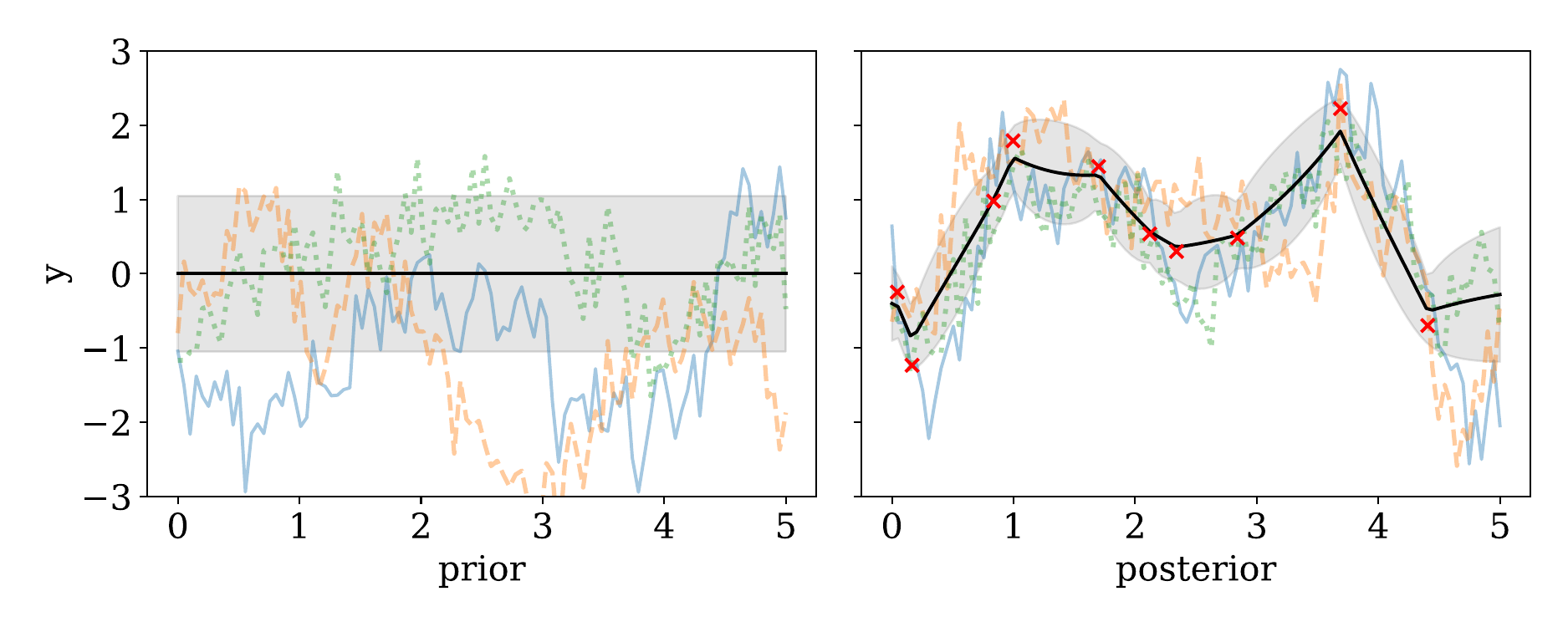}
        \caption{\setlinespacing{1.1} Sample paths of the Laplace covariance function from a GP prior and posterior. The solid line represents $m(\mathbf{x})$ and $\mathbf{\bar{f}}_*$ in the left and right panel respectively.  The shaded bands represent $\text{cov}(\mathbf{f}_*)$.  The red \texttt{x}'s in the right panel are the training points.}
        \label{fig:priorposterior}
    \end{figure}
    
    Equation \eqref{eq:priordist} forms the prior distribution for our GP regressor. To get the posterior distribution, we condition $\mathbf{f_*}$ on the known data which by properties of $\MVN$ is itself $\MVN$ \citep{guan2020introduction}:
    \begin{equation}\label{eq:posteriordist}
        \begin{aligned}
            \mathbf{f}_* | X, y, X_* &\sim \mathcal{MVN}(\mathbf{\bar{f}}_*, \text{cov}(\mathbf{f_*})) \\
            \mathbf{\bar{f}}_* &= K_{X_*X}\left[K_{XX} + \sigma^2 I_n\right]^{-1}\mathbf{y} \\
            \text{cov}(\mathbf{f}_*) &= K_{X_*X_*} - K_{X_*X}\left[K_{XX} + \sigma^2 I_n\right]^{-1} K_{XX_*}
        \end{aligned}
    \end{equation}
    
    Viewing the procedure from the practical modeling perspective, training the GP regressor can be thought of as pre-computation of $\left[K_{XX} + \sigma^2 I_n\right]^{-1}$ while prediction computes the remainder of the values involving the test set $X_*$.  In addition, with knowledge of the full distribution we can sample the prior and posterior distributions to view the bounds and functions generated by the GP as seen in Figure \ref{fig:priorposterior}.  The matrix operations involved in calculating the posterior are $O(n^3)$ and thus require approximate methods for large scale datasets.  Rasmussen and Williams calculate the inverse using Cholesky decomposition and provide a number of approximate methods for computation \citep[pg. 19, pg. 171]{rasmussen2006gpml}.
    
    \section{Neural Networks}
    \label{sec:nn}
    
    A neural network is one of the key tools used in machine learning.  They provide a way to model highly nonlinear phenomena and have been found to generalize exceptionally to a variety of complex tasks.  One drawback of neural networks is the difficulty of analyzing their black box nature.  Traditional tools for inference and uncertainty estimation cannot be easily applied to them; however, there are equivalences that aid in this task.  To start, we adopt notation from \citep{Goodfellow2016deep} and define the \textit{neuron} which is the basic building block of a neural network.
    
    \begin{definition}[Neuron]
        Let $\mathbf{x} = [x_1, \dots, x_d]^\top \in \mathcal{X}$ be an input vector with $d\in\mathbb{N}$.  Then the $k$-th \textbf{neuron} $a_k$ is defined as:
        \begin{equation}
            a_k = \sigma\left(\sum_{i=1}^d w_{k_i} x_i + b_k \right) = \sigma\left(\mathbf{w}_k^\top \x + b_k\right)
        \end{equation}
        where $\mathbf{w}_k = [w_{k_1},\dots, w_{k_d}]^\top$ is a vector of weights and $b_k\in\mathbb{R}$ is the bias value.  $\sigma$ is an \textbf{activation function} which is a fixed transformation that is used to inject nonlinear behavior to the resulting outputs.
    \end{definition}
    
    \begin{figure}[t]
        \centering
        \begin{tikzpicture}[
            init/.style={
              draw,
              circle,
              inner sep=2pt,
              font=\huge
            },
            squa/.style={
              draw,
              inner sep=2pt,
              font=\large,
              join = by -latex
            },
            start chain=2,node distance=13mm
            ]
            \node[on chain=2] 
              (x2) {$\vdots$};
            \node[on chain=2] 
              {$\ \ \vdots$};
            \node[on chain=2,init] (sigma) 
              {$\displaystyle\Sigma$};
            \node[on chain=2,squa,label=above:{\parbox{2cm}{\centering Activation \\ function}}]   
              {$\sigma$};
            \node[on chain=2,label=above:Output,join=by -latex] 
              {$a_k$};
            \begin{scope}[start chain=1]
                \node[on chain=1] at (0,1.5cm) 
                  (x1) {$x_1$};
                \node[on chain=1,join=by o-latex] 
                  (w1) {$w_1$};
            \end{scope}
            \begin{scope}[start chain=3]
                \node[on chain=3] at (0,-1.5cm) 
                  (x3) {$x_d$};
                \node[on chain=3,label=below:Weights,join=by o-latex] 
                  (w3) {$w_d$};
            \end{scope}
            \node[label=above:\parbox{2cm}{\centering Bias \\ $b_k$}] at (sigma|-w1) (b) {};
            
            \draw[-latex] (w1) -- (sigma);
            \draw[-latex] (w3) -- (sigma);
            \draw[o-latex] (b) -- (sigma);
            
            \draw[decorate,decoration={brace,mirror}] (x1.north west) -- node[left=10pt] {Inputs} (x3.south west);
            \end{tikzpicture}
        \caption{\setlinespacing{1.1} A visualization of a neuron courtesy of \citep{neuralTeX}.}
        \label{fig:neuron}
    \end{figure}
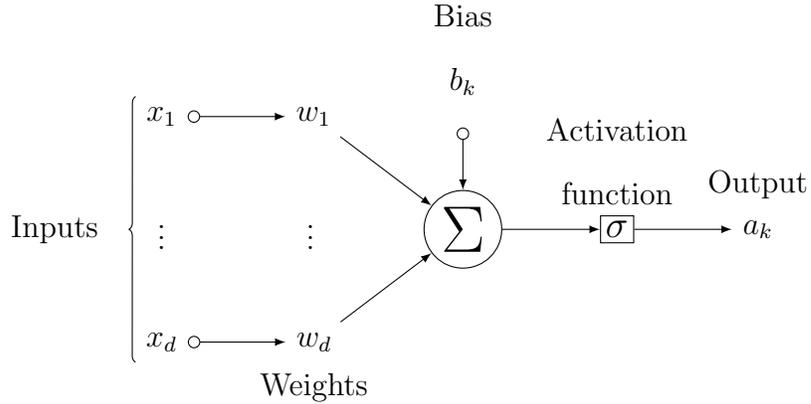
    
    \noindent Nonlinear activations allow neural networks to model complex behavior that traditional linear models cannot.  A combination of $l\in\mathbb{N}$ neurons into a vector is interpreted as $l$-width layer.  By combining many layers of varying widths, we can build the simplest architecture of a neural network called the \textit{multilayer perceptron}.
    
    \begin{definition}[Multilayer Perceptron]
        \label{def:MLP}
        Let $\mathbf{x} \in \mathcal{X}$ and $h^{(i)}$ be the $i$-th \textbf{hidden layer} where $i\in \{1,\dots,L\}$ represents a finite amount of layers with $L\in\mathbb{N}$. Then a \textbf{multilayer perceptron architecture} is defined as follows:
        
        \begin{equation}\label{eq:neuralnet}
            \begin{aligned}
                h^{(1)} &= \sigma^{(1)}\left(\bm{W}^{(1)}\mathbf{x} + \beta^{(1)}\right) \\
                h^{(2)} &= \sigma^{(2)}\left(\bm{W}^{(2)}h^{(1)} + \beta^{(2)}\right) \\
                &\hspace{1em}\vdots \\
                h^{(L)} &= \sigma^{(L)}\left(\bm{W}^{(L)}h^{(L-1)} + \beta^{(L)}\right) \\
                f(\mathbf{x};\theta) &= \bm{w}h^{(L)} + \beta^{(L + 1)}
            \end{aligned}
        \end{equation}
        \noindent where $\bm{W}^{(i)}$ are weight matrices for each layer, $\bm{w}$ is the vector of weights for the desired output dimension of the function $f$, $\beta^{(i)}$ is the bias vector, and $\sigma^{(i)}$ is the layer dependant activation function. The 1st dimension of $\bm{W}^{(i)}$ represents the number of neurons (i.e.\ the width of the hidden layer $h^{(i)}$) whilst the 2nd dimension is determined by the dimension of $h^{(i-1)}$.  $h^{(1)}$ is called the \textbf{input layer} while $f(\mathbf{x};\theta)$ is the \textbf{output layer} where $\theta$ is a set of all parameters $\bm{W}^{(1)}, \bm{W}^{(2)}, \dots, \bm{W}^{(L)}, \bm{w}, \beta^{(1)}, \beta^{(2)}, \dots, \beta^{(L+1)}$ in the network. The resulting network is considered to be $L+1$ layers deep.
    \end{definition}
    
    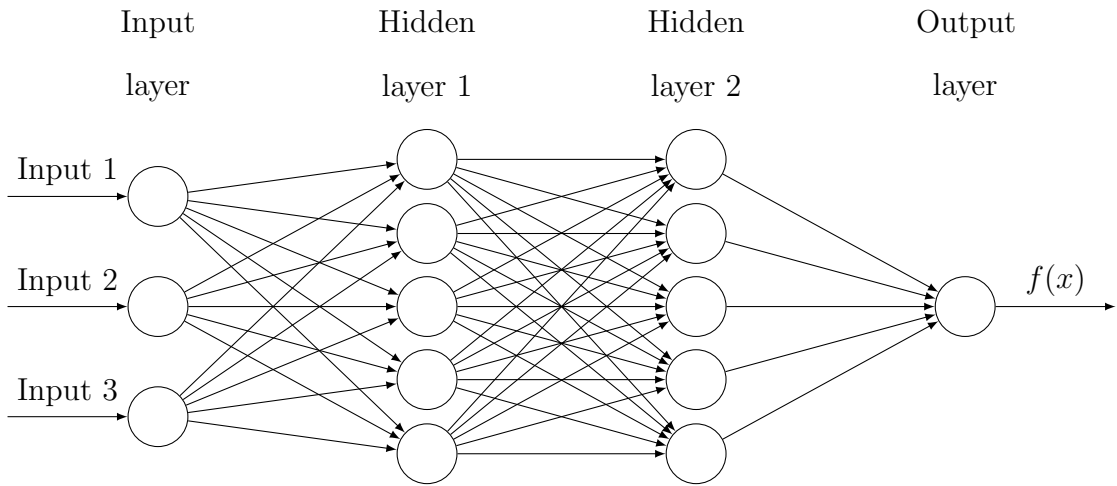
\begin{figure}[b!]
    \centering
        \begin{tikzpicture}[
            plain/.style={
              draw=none,
              fill=none,
              },
            net/.style={
              matrix of nodes,
              nodes={
                draw,
                circle,
                inner sep=8pt
                },
              nodes in empty cells,
              column sep=1cm,
              row sep=-9pt
              },
            >=latex
            ]
            \matrix[net] (mat)
            {
            |[plain]| \parbox{1.3cm}{\centering Input\\layer} & |[plain]| \parbox{1.3cm}{\centering Hidden\\layer 1} & |[plain]| \parbox{1.3cm}{\centering Hidden\\layer 2} & |[plain]| \parbox{1.3cm}{\centering Output\\layer} \\
            |[plain]| &           &            \\
                      & |[plain]| & |[plain]|  \\
            |[plain]| &           &            \\
            |[plain]| & |[plain]| & |[plain]|  \\
                      &           &           & \\
            |[plain]| & |[plain]| & |[plain]|  \\
            |[plain]| &           &            \\
                      & |[plain]| & |[plain]|  \\
            |[plain]| &           &            \\    };
            \foreach \ai [count=\mi ]in {3,6,9}
              \draw[<-] (mat-\ai-1) -- node[above] {Input \mi} +(-2cm,0);
            \foreach \ai in {3,6,9}
            {\foreach \aii in {2,4,...,10}
              \draw[->] (mat-\ai-1) -- (mat-\aii-2);
            }
            \foreach \ai in {2,4,...,10}
            {\foreach \aii in {2,4,...,10}
              \draw[->] (mat-\ai-2) -- (mat-\aii-3);
            }
            \foreach \ai in {2,4,...,10}
              \draw[->] (mat-\ai-3) -- (mat-6-4);
            \draw[->] (mat-6-4) -- node[above] {$f(x)$} +(2cm,0);
        \end{tikzpicture}
        \caption{\setlinespacing{1.1} A visualization of a 2 layer fully connected neural network courtesy of \citep{neuralTeX}.}
    \end{figure}
    
    \noindent A multilayer perceptron can be modified to include transformation layers (e.g.\ a \textit{dropout layer} which randomly drops weights from the previous layer), differing activation functions per layer, differing layer widths, etc.  
    The training procedure for such networks follows the data fitting procedure in Equation \eqref{eq:generalloss}
    \begin{equation}
         \theta_{opt} = \argmin_\theta \left\{ \sum_{i=1}^n L(y_i, f(\x_i; \theta))\right\},
    \end{equation}
    so the optimized network is $f_{opt}(\mathbf{x}) := f(\mathbf{x}; \theta_{opt})$ with optimal parameters $\theta_{opt}$ found during the minimization procedure over the unfixed network parameters $\theta$.  The minimization is done by computing the gradient of the loss function with respect to $\theta$.  This is done efficiently via \textit{back-propagation} \citep{rumelhart1986learning} which uses information provided by the outputs $f(\cdot)$ in a single gradient step and computes the gradient of the loss over the parameters starting from the output layer and ending at the input layer.  On the other hand, prediction is done by \textit{forward propagation} which uses an input to calculate an output $f_*$ on the optimized network through a forward pass using Equation \ref{eq:neuralnet}.
    
    The specific neural networks we consider in this thesis are \textit{infinite width fully connected rectified linear unit (ReLU) activated neural networks}. \textit{Infinite width} refers to the size of all hidden layers.  \textit{Fully connected} means that every individual neuron in every layer $i \in \{1,\dots, L\}$ sends its output to every neuron in layer $i+1$.  Lastly, the \textit{ReLU activation function} is defined as follows for $x\in\R$:
    \begin{equation}
        \begin{gathered}
            \sigma(x) = \max(0, x) \\
            \text{where } \sigma(\cdot) : \R \to \R.
        \end{gathered}
    \end{equation}
    
    \citet{neal1996bayesian} showed that neural networks in the infinite width limit converge to a GP while \citet{williams1996computing} developed the computation of the covariance function. Let us assume a single hidden layer network with weights $w_{k_i}$ and biases $b_k$, each independent and identically distributed (iid) normal with mean 0 and respective variances $\sigma^2_w$ and $\sigma^2_b$. The gist of this equivalence is that each neuron in the hidden layer then becomes iid normal with mean zero and finite variance.  By the Central Limit Theorem, as the width $H$ of the hidden layer tends to infinity, it also forms a normal distribution.   By following Definition \ref{def:gp}, we form a stochastic process over the indexed set of $n$ inputs and corresponding outputs which form a multivariate normal distribution with mean zero and finite covariance for all inputs:  
    \begin{equation}
        \begin{gathered}
        \begin{aligned}
            m(\mathbf{x}) &= 0 \\
            k(\mathbf{x},\mathbf{z}) &= \sigma^2_b + \sigma^2_w H \mathbb{E}[h_j(\mathbf{x})h_j(\mathbf{z})] \\
            &=\sigma^2_b + \omega_w^2 \mathbb{E}[h_j(\mathbf{x})h_j(\mathbf{z})]
        \end{aligned} \\
        \text{where } \sigma_w = \omega_w H^{-1/2} \text{ and } j\in\{1,\dots, L\}.
        \end{gathered}
    \end{equation}
    This single layer infinite width network satisfies the definition of a GP.  This can be further extended to networks of multiple layers by applying the same procedure by over all network weights and biases.  It should be noted that this formulation describes an untrained network at initialization.
    
    Although somewhat impractical in practice, infinite width networks provide a way to study the properties of neural networks. \citet{yang2019wide} expanded on this and showed that the GP and neural network equivalence applies to all modern architectures.  The result described in this section is related to the neural tangent kernel \citep{jacot2018neural} which can be used as a GP covariance function to implement an infinite width neural network which is further discussed in Chapter \ref{ch:types}.
    
    \chapter{Reproducing Kernel Hilbert Spaces}\label{ch:rkhs}
    
    Kernels exist in a variety of contexts throughout statistics, probability, and mathematics.  They can be thought of as the kernel of a probability density (mass) function used in kernel density estimation, a positive definite kernel used in a variety of kernel methods, a null space in linear algebra, an integral transform in calculus, and a reproducing kernel considered in functional analysis.
    
    Although kernels span a large variety of subjects, we are going to focus on them as they pertain to functional analysis via \textit{reproducing kernel Hilbert spaces}. In this chapter, we will define kernels, their reproducing kernel Hilbert spaces, and properties of these spaces in the context of data fitting.  
    
    \section{Reproducing Kernels}\label{sec:reproducingkernels}
    
    A \textit{kernel} is a class of functions that map two values to the real line:
    \begin{equation}\label{kernel}
        k:\mathcal{X}\times\mathcal{X}\to\mathbb{R}
    \end{equation}
    \noindent In our context, we are interested in kernels that are positive definite and symmetric because this allows them to be real valued and used in methods such as kernel ridge regression, support vector machines, and Gaussian processes.  As mentioned in Chapter \ref{ch:regression}, a positive definite kernel and covariance function are one and the same in the context of GPs. In order to further develop kernels, we need to work within a general vector space, namely a Hilbert space.
    
    \begin{definition}[Hilbert space]
        A \textbf{Hilbert space} $\mathcal{H}$ is a vector space with the following properties:
        \begin{enumerate}
            \item $\calH$ contains an inner product $\langle \cdot, \cdot \rangle_{\calH} : \calH \times \calH \to \R$ which induces the norm $\langle x,x \rangle_\calH = ||x||^2_\calH$ for $x\in\calH$,
            \item $\calH$ is complete (i.e.\ every Cauchy sequence in $\calH$ converges to some element of $\calH$).
        \end{enumerate}
    \end{definition}
    
    Although our definition of a Hilbert space constitutes a real valued inner product, it can be defined over more general spaces (see \citet[Chapter 8 pg. 211]{axler2020measure}).  In our case it is sufficient to work in the space of $\R$.  We are now ready to define a key space regarding kernels.
    
    \begin{definition}[Reproducing Kernel Hilbert Space]
        Let $\calH_k$ be a Hilbert space of real functions defined on $\calX$ and norm $\|f\|_{\calH_k}^2 = \langle f,f\rangle_{\calH_k}$ for $f\in\calH_k$.  The function $k:\calX\times\calX\to\R$ is called a reproducing kernel of $\calH_k$ if:
        \begin{enumerate}
            \item For all $\mathbf{x}\in\calX$, $k(\cdot, \mathbf{x}) \in \calH_k$,
            \item\label{rkhs2} For all $\mathbf{x}\in\calX$ and for all $f\in\calH_k$, $\langle f(\cdot), k(\cdot, \mathbf{x}) \rangle = f(\mathbf{x})$.
        \end{enumerate}
    \end{definition}
    
    Property \ref{rkhs2} in the definition above is called the \emph{reproducing property}. Reproducing kernel Hilbert spaces (RKHSs) do not require for kernel functions to be positive definite explicitly; however, positive definite kernels have a nice property regarding RKHS:
    
    \begin{theorem}[Moore-Aronszajn Theorem \citep{aronszajn1950theory}]
        \label{thm:moorearon}
        For a positive definite function $k(\cdot, \mathbf{x})$ on $\calX\times\calX$ there exists only one RKHS.
    \end{theorem}
    
    This theorem guarantees that for any symmetric and positive definite kernel, there exists a unique RKHS and vice-versa.  Using this idea we can build an RKHS from a positive definite kernel alone. To illustrate this given a symmetric and positive definite kernel on $\calX$, we begin by defining a \textit{pre-Hilbert space} $\mathcal{H}_0$.
    \begin{multline}
        \mathcal{H}_0 := \text{span}\left\{  k(\cdot, \mathbf{x}) : \mathbf{x} \in \mathcal{X}\right\} \\
        = \left\{ f(\cdot) = \sum_{i=1}^n c_i k(\cdot, \mathbf{x}_i) : n\in\mathbb{N}, c_1, \dots, c_n \in \mathbb{R}, \mathbf{x}_1, \dots, \mathbf{x}_n \in \mathcal{X} \right\}.
    \end{multline}
    with a valid inner product (see \citep[pg. 20]{berlinet2011reproducing}) defined for any $f := \sum_{i=1}^n a_i k(\cdot, \mathbf{x}_i)$ and $g := \sum_{j=1}^m b_j k(\cdot, \mathbf{x}_j)$
    \begin{equation}
        \langle f, g \rangle_{\mathcal{H}_0} = \sum_{i=1}^n\sum_{j=1}^m a_i b_j k(\mathbf{x}_i, \mathbf{x}_j)
    \end{equation}
    and a norm induced by the inner product
    \begin{equation}
        \langle f, f \rangle_{\mathcal{H}_0} = ||f||^2_{\mathcal{H}_0} = \sum_{i=1}^n\sum_{j=1}^n a_i a_j k(\mathbf{x}_i, \mathbf{x}_j) = \mathbf{a}^\top K \mathbf{a}
    \end{equation}
    where $\mathbf{a} = [a_1, \dots, a_n]^\top$ is a vector and $K = [k(\mathbf{x}_i, \mathbf{x}_j)]^n_{i,j = 1}$ is a matrix. To form the RKHS $\mathcal{H}_k$ we must bundle the rest of the possible elements within $\mathcal{H}_0$ by defining its closure with respect to $||\cdot||_{\mathcal{H}_0}$
    \begin{multline}\label{eq:fullRKHS}
        \mathcal{H}_k = \left\{ f(\cdot) = \sum_{i=1}^\infty c_i k(\cdot, \mathbf{x}_i) : n\in\mathbb{N}, c_1, c_2,\dots \in \R, \mathbf{x}_1, \mathbf{x}_2,\dots \in \calX \text{ where } \right. \\
        \left. \|f\|^2_{\mathcal{H}_k} := \lim_{n\to\infty}\left\| \sum^n_{i=1} c_i k(\cdot, \mathbf{x}_i) \right\|^2_{\mathcal{H}_0} = 
        \sum_{i=1}^\infty\sum_{j=1}^\infty c_i c_j k(\mathbf{x}_i, \mathbf{x}_j) < \infty \right\},
    \end{multline}
    thus completing our construction. For more details refer to \citet[pg. 11]{kanagawa2018gaussian} and \citet[Theorem 3 pg. 19-21]{berlinet2011reproducing}.
    
    \section{Mercer Representation}
    
    There is another way to represent RKHSs that provides a further route of analyzing kernels via spectral decomposition. This will require a deeper dive into functional analysis for which further details can be found in \citet[Chapter 10 pg. 280]{axler2020measure}, \citet[Appendix A.5 pg 497]{steinwart2008support}, and \citet{rudin1991functional}.
    
    We begin with a measurable space $\mathcal{X}$ with $\mu$ being its finite Borel measure with $\mathcal{X}$ being its support.  We then consider the space of square-integrable functions $L_2(\mathcal{X}, \mu)$ that exist on $\mathcal{X}$ with respect to metric $\mu$ and the kernel $k$ on $\mathcal{X}$.  Then we define an integral operator $T_k : L_2(\mathcal{X}, \mu) \to L_2(\mathcal{X}, \mu)$ such that for $f\in L_2(\mathcal{X}, \mu)$
    \begin{equation}
        T_k f := \int_\mathcal{X} k(\cdot, \mathbf{x}) f(\mathbf{x}) d\mu(\mathbf{x})
    \end{equation}
    which is compact, positive, and self-adjoint \citep[Theorem 4.27]{steinwart2008support}.  As a result, we can apply the Spectral Theorem \citep[Theorem A.5.13]{steinwart2008support} which shows that there exists $(\phi_i, \lambda_i)_{i\in I}$ where $(\phi_i)_{i\in I} \subset L_2(\mathcal{X}, \mu)$ is the orthonormal system of countable eigenfunctions and $(\lambda_i)_{i\in I}$ is the corresponding family of eigenvalues for indices $I\subseteq \mathbb{N}$ such that for strictly positive eigenvalues $\lambda_1 \geq \lambda_2 \geq \cdots > 0$:
    \begin{equation}
        T_k f = \sum_{i\in I} \lambda_i \langle \phi_i, f \rangle_{L_2} \phi_i.
    \end{equation}
    With this we can now develop an expansion of kernels via orthonormal functions:
    \begin{theorem}[Mercer's theorem \citep{steinwart2008support, mercer1909functions}]
        \label{thm:mercer}
        Let $\mathcal{X}$ be a compact metric space, $k : \mathcal{X} \times \mathcal{X} \to \mathbb{R}$ be a continuous kernel, and $\mu$ be a finite Borel measure with support on $\mathcal{X}$.  Then for $(\phi_i, \lambda_i)_{i\in I}$ and inputs $\mathbf{x},\mathbf{z}\in\mathcal{X}$ we have
        \begin{equation}
            k(\mathbf{x},\mathbf{z}) = \sum_{i\in I} \lambda_i \phi_i(\mathbf{x}) \phi_i(\mathbf{z}),
        \end{equation}
        where the convergence is absolute and uniform over the inputs.
    \end{theorem}
    \noindent Lastly, we can redevelop RKHSs using orthonormal eigenfunctions to represent kernels:
    \begin{theorem}[Mercer Representation of RKHSs \citep{steinwart2008support}]
        \label{thm:mercerrepresentation}
        Let $\mathcal{X}$ be a compact metric space, $k : \mathcal{X} \times \mathcal{X} \to \mathbb{R}$ be a continuous kernel, $\mu$ be a finite Borel measure with support on $\mathcal{X}$ and $(\phi_i, \lambda_i)_{i\in I}$. We then define the RKHS $\mathcal{H}_k$
        \begin{equation}\label{eq:mercerRKHS}
            \mathcal{H}_k := \left\{ f := \sum_{i\in I} a_i \sqrt{\lambda_i}\phi_i : ||f||_{\mathcal{H}_k}^2 := \sum_{i\in I} a_i^2 < \infty \right\},
        \end{equation}
        with an inner product defined for any $f := \sum_{i\in I} a_i\sqrt{\lambda_i}\phi_i$ and $g := \sum_{i\in I} b_i\sqrt{\lambda_i}\phi_i$
        \begin{equation}
            \langle f, g \rangle_{\mathcal{H}_k} = \sum_{i\in I} a_i b_i
        \end{equation}
        where $(\sqrt{\lambda_i}\phi_i)_{i\in I}$ is an orthonormal basis of $\mathcal{H}_k$ and operator $T_k^{1/2} : L_2(\mathcal{X}, \mu) \to H$ is an isometric isomorphism.  The $\calH_k$ in Equation \eqref{eq:mercerRKHS} is the same as in Equation \eqref{eq:fullRKHS}.
    \end{theorem}
    
    As a result, Theorem \ref{thm:mercer} gives a general way of analyzing kernels via their resulting eigenfunctions and corresponding eigenvalues while Theorem \ref{thm:mercerrepresentation} connects this analysis back to RKHSs.  One way to do this is by observing and comparing the rate of eigenvalue decay between two separate kernels \citep{geifman2020rkhs}. Furthermore, both theorems indicate that an RKHS $\mathcal{H}_k$ is a subset of the $L_2(\mathcal{X}, \mu)$ space.  Because of this, the Mercer notion of kernel and RKHS relies on a choice of measure $\mu$ to work; however, as shown in Section \ref{sec:reproducingkernels}, both kernel and RKHS are independent of any measure.  As a result, regardless of the choice of measure we still end up with the same RKHS \citep{kanagawa2018gaussian}.  The only difference here is that the new measure will result in a new eigensystem.
    
    \section{Representer Theorem}
    
    Let us backtrack to the data fitting problem established in Section \ref{sec:datafitting}.  Consider the dataset $\{(\mathbf{x}_i, y_i) : i = 1, \dots, n\}$ where $\mathbf{x}_i\in\mathcal{X}$ and $y_i\in\mathbb{R}$. From such a setup, it is natural to wish to determine whether there exists a function $f$ that is generating some sort of signal in our dataset.  We can do so by regression using a variety of methods to accomplish this task; however, in general we wish to minimize an empirical risk functional \citep{scholkopf2001generalized}.
    
    \begin{theorem}[Representer Theorem]\label{thm:represent}
        Let $\{\mathbf{x}_i\}^n_{i=1}\subset\mathcal{X}$ and $\{y_i\}^n_{i=1}\subset\mathbb{R}$.  Consider functions defined as $f : \mathcal{X}\to\mathbb{R}$ defined in a RKHS $\mathcal{H}_k$ where $k$ represents a kernel.  Now consider the following empirical risk minimization problem:
        \begin{equation}
            f_{opt} = \argmin_{f\in\mathcal{H}_k} \left\{\sum_{i=1}^{n} L(y_i, f(\mathbf{x}_i)) + \lambda ||f||^2_k \right\}
        \end{equation}
        $L : \mathbb{R}^2\to\mathbb{R}$ represents the data fitting term (usually referred to as a loss function) and $\lambda ||f||^2_k$ is a regularization term with factor $\lambda \geq 0$.  The minimizer $f_{opt} \in\mathcal{H}_k$ can be represented pointwise as follows: 
        \begin{equation}\label{eq:minimizer}
            f_{opt}(\mathbf{x}) = \sum^n_{i=1}\alpha_i k(\mathbf{x}, \mathbf{x}_i) = K_{\mathbf{x} X} \bm{\alpha}, \qquad \mathbf{x} \in \mathcal{X}
        \end{equation}
        where $\bm{\alpha} = [\alpha_1, \dots, \alpha_n]^\top = (K_{XX} + n\lambda I_n)^{-1}\mathbf{y}$.
    \end{theorem}
    
    As a result, the potentially infinite dimensional problem of finding $f_{opt}$ only depends on a finite sum and on the data in question.  In addition, it is guaranteed to exist and is unique. It also gives us a generic way of looking various regression tools from the more abstract perspective of RKHSs. This includes deep neural networks as shown by \citet{unser2019representer}. It is important to note that Equation \eqref{eq:minimizer} shows that the minimizer belongs to the RKHS regardless of the dataset used.
    
    Regularization is one form of the representer theorem. By applying Theorem \ref{thm:represent} and slightly adjusting the regularized regression problem presented in Equation \eqref{eq:smoothingspline}, we can view the penalty term as one determined by a RKHS endowed with a kernel $k$:
    \begin{equation}
        \begin{aligned}
            f_{opt} &= \argmin_{f\in\mathcal{H}_k} \left\{\dfrac{1}{n}\sum_{i=1}^n (y_i - f(\mathbf{x}_i))^2 + \lambda ||f||_{\mathcal{H}_k}^2\right\}
        \end{aligned}
    \end{equation}
    with a unique solution
    \begin{equation}
        K_{\mathbf{x} X} (K_{XX} + n\lambda I_n)^{-1}\mathbf{y}.
    \end{equation}
    This form of regression is called \textit{kernel ridge regression} (KRR) \citep{kanagawa2018gaussian} and only depends on functions defined by the kernel $k$ and its corresponding RKHS $\mathcal{H}_k$.  Alternatively, a GP regressor defined by the same kernel $k$ has a unique posterior mean function computed by the marginal log likelihood
    \begin{equation}
        K_{\mathbf{x} X} (K_{XX} + \sigma^2 I_n)^{-1}\mathbf{y}.
    \end{equation}
    Given this, \citet{kanagawa2018gaussian} concisely summarize the following known result:
    \begin{prop}
        If $\sigma^2 = n\lambda$ then the GP posterior mean function and the KRR solution are the same.
    \end{prop}
    
    This result ties a GP posterior mean to a KRR solution and the resulting RKHS $\mathcal{H}_k$.  Although any valid kernel is tied to an RKHS, our definition of a GP regressor does not directly rely on that connection.  By relating GP regression to KRR we can better analyze kernels via the GP framework while having confidence in the underlying RKHS theory.  In particular interest is attempting to find kernel hyperparameters of a GP for the Laplace and neural tangent kernels so that the resulting posterior means are the same.  Through this matching of posterior means (and underlying hyperparameters) we can gain insight to the underlying RKHS of each kernel which we explore later in Chapter \ref{ch:synthexp}.
    
    \chapter{Types of Kernels and their Equivalences}\label{ch:types}
    
    Having developed the key theory behind kernels we turn to define the kernels used in this thesis.  Our empirical analysis will focus on the Laplace, Gaussian, and neural tangent kernels.  We define the kernels over their respective inputs $\mathbf{x},\mathbf{z}\in\mathcal{X}$ and set of parameters $\theta$; however, when it is clear from context we may drop the inputs so that
    \begin{equation}
        k(\theta) := k(\mathbf{x},\mathbf{z}; \theta).
    \end{equation}
    Moving forward it should be noted that by the reproducing property of a given kernel $k$, the elements of its RKHS can be represented as
    \begin{equation}
        f(\cdot) = \sum_{i=1}^\infty c_i k(\cdot, \mathbf{x}_i) \text{ where } c_i\in\R, \mathbf{x}_i \in \calX
    \end{equation}
    which shows that RKHS member functions inherit properties that are dependent on the kernel.
    
    \section{Mat\'{e}rn Class of Kernels}
    
    \begin{definition}[Mat\'{e}rn Kernel]
        Let $\mathbf{x},\mathbf{z} \in \Rd$ be inputs, $\ell > 0$ be the length-scale parameter, and $\nu > 0$ be the smoothness parameter such that
        \begin{equation}
            k_{Mat}(\mathbf{x},\mathbf{z}; \nu, \ell) = \dfrac{2^{1-\nu}}{\Gamma(\nu)}\left(\dfrac{\sqrt{2\nu}}{\ell}\|\mathbf{x}-\mathbf{z}\|)\right)^\nu K_\nu\left(\dfrac{\sqrt{2\nu}}{\ell}\|\mathbf{x}-\mathbf{z}\|)\right)
        \end{equation}
        where $||\cdot||$ is the $L_2$-norm, $\Gamma$ is the Gamma function, and $K_\nu$ is a modified Bessel function of the 2nd kind \citep[Section 9.6]{abramowitz1964handbook}.
    \end{definition}
    
    The Mat\'{e}rn class of kernels defines a set of functions dependent on their smoothness $\nu$.  By varying $\nu$ we define our two kernels of interest. As $\nu \to \infty$, the Mat\'{e}rn kernel becomes equivalent to the well known \textit{Gaussian kernel}\footnote{Also known as the \textit{radial basis function} or the \textit{squared exponential kernel}.}:
    \begin{equation}
        k_{Gaus}(\mathbf{x},\mathbf{z}; \ell) = \exp\left(-\dfrac{\|\mathbf{x} - \mathbf{z}\|^2}{2\ell^2} \right).
    \end{equation}
    By setting $\nu = \frac{1}{2}$, the Mat\'{e}rn kernel becomes equivalent to the \textit{Laplace kernel}:
    \begin{equation}
        k_{Lap}(\mathbf{x},\mathbf{z}; \ell) = \exp\left(-\dfrac{\|\mathbf{x} - \mathbf{z}\|}{\ell^2} \right).
    \end{equation}
    Both kernels are very similarly defined and as such fall under the umbrella of exponential class of kernels as well.  Despite their similar forms, the two kernels have some stark differences.  The Gaussian kernel produces an RKHS of continuous, infinitely differentiable functions for all possible length-scales \citep[Section 4.2.1]{rasmussen2006gpml}, and its eigenvalues decay exponentially \citep{minh2006mercer}.  In contrast, elements of the Laplace kernel's RKHS are continuous, nowhere differentiable for all possible length-scales \citep[Section 4.2.1]{rasmussen2006gpml}, and its eigenvalues decay polynomially \citep{geifman2020rkhs}.  A process defined by the Laplace kernel is called an \textit{Ornstein--Uhlenbeck process} \citep{uhlenbeck1930theory} which was shown to describe the velocity of a Brownian particle.
    
    \section{Neural Tangent Kernel}
    
    As mentioned in Section \ref{sec:nn}, infinite width neural networks behave as GPs and can be studied through the function space. One hindrance with this approach is that in order to develop a GP via the method outlined previously, one must determine the covariance kernel of the GP using all the parameters of that architecture.  \citet{cho2009kernel} found that instead of this approach, it is possible to use the \textit{arc-cosine kernel} to build various finite network architectures via a single kernel. Our focus is on the \textit{neural tangent kernel} defined by \citet{jacot2018neural} which describes the behavior of a neural network trained by gradient descent.
    
    \begin{definition}[Finite Neural Tangent Kernel]
        \label{def:NTK}
        Let $f(\cdot; \theta)$ be a neural network with finite number of parameters $\theta$.  Then, the \textbf{finite neural tangent kernel} for the neural network and inputs $\mathbf{x},\mathbf{z}\in\Sd$ is defined as a sum containing tensor products of partials with respect to the $p$-th parameter of $f$: 
        \begin{equation}
            k_{NTK}(\mathbf{x},\mathbf{z}; \theta) = \sum_{p=1}^P \partial_{\theta_p} f(\mathbf{x}; \theta) \otimes \partial_{\theta_p} f(\mathbf{z}; \theta),
        \end{equation}
        where $P$ is the total number of network parameters
        and $\Sd$ is the unit $d$-sphere space defined as
        \begin{equation}
            \Sd := \left\{\mathbf{x}\in\Rd : \|\mathbf{x}\| = 1\right\}.
        \end{equation}
    \end{definition}
    
    Definition \ref{def:NTK} refers to the neural tangent kernel for finite width and depth neural networks. However, this kernel can also be used with kernel methods to represent infinitely wide neural network architectures through an explicit recursively defined kernel. The neural network and GP equivalence discussed in Section \ref{sec:nn} makes this possible by making the network in question independent of the parameters $\theta$ and instead dependent on the resulting GP.  The details of this are further discussed in the appendix of \citet{jacot2018neural}. In our definitions we adopt their notation alongside notation from \citet{bietti2019inductive}.
    
    \begin{definition}[Infinite Neural Tangent Kernel]\label{def:recurNTK}
    Given a fully connected infinite width network with $L+1$ layers, $\beta \geq 0$ bias, and with $h\in \{1,\dots, L\}$ we define the deterministic \textbf{infinite neural tangent kernel} recursively for inputs $\mathbf{x},\mathbf{z}\in\Sd$ as
    \begin{equation}
        \begin{gathered}
            k_{NTK}(\mathbf{x}, \mathbf{z}; L+1, \beta) := \Theta^{(L)}(\mathbf{x},\mathbf{z}) \\
            \Theta^{(h)}(\mathbf{x},\mathbf{z}) = \Theta^{(h-1)}(\mathbf{x},\mathbf{z})\dot{\Sigma}^{(h)}(\mathbf{x},\mathbf{z}) + \Sigma^{(h)}(\mathbf{x},\mathbf{z}) + \beta^2,
        \end{gathered}
    \end{equation}
    with the base case
    \begin{equation}\label{NTKinit}
        \begin{aligned}
            \Sigma^{(0)}(\mathbf{x},\mathbf{z}) &= \mathbf{x}^\top \mathbf{z} \\
            \Theta^{(0)}(\mathbf{x},\mathbf{z}) &= \Sigma^{(0)}(\mathbf{x},\mathbf{z}) + \beta^2
        \end{aligned}
    \end{equation}
    and components
    \begin{equation}
        \begin{aligned}
            \Sigma^{(h)}(\mathbf{x},\mathbf{z}) &= \dfrac{c_\sigma}{2} \kappa_1\left(\lambda^{h-1}\right) \sqrt{\Sigma^{(h-1)}(\mathbf{x},\mathbf{x})\Sigma^{(h-1)}(\mathbf{z},\mathbf{z})} \\
            \dot{\Sigma}^{(h)}(\mathbf{x},\mathbf{z}) &= \dfrac{c_\sigma}{2}\kappa_0\left(\lambda^{h-1}\right),
        \end{aligned}
    \end{equation}
    where $c_\sigma = 2$ for ReLU activated networks.  We then define the cosine normalization \citep{ghojogh2021reproducing}:
    \begin{equation}\label{eq:cosnorm}
        \begin{gathered}
            \lambda^{(h-1)}(\mathbf{x}, \mathbf{z}) = \dfrac{\Sigma^{(h-1)}(\mathbf{x},\mathbf{z})}{\sqrt{\Sigma^{(h-1)}(\mathbf{x},\mathbf{x})\Sigma^{(h-1)}(\mathbf{z},\mathbf{z})}},
        \end{gathered}
    \end{equation}
    such that $|\lambda^{(h-1)}| \leq 1$. Lastly, we define the arc-cosine kernels of degree 0 and 1 respectively \citep{cho2009kernel}:
    \begin{equation}\label{eq:arccos}
        \begin{aligned}
            \kappa_0(u) &= \dfrac{1}{\pi}\left(\pi - \arccos(u)\right) \\
            \kappa_1(u) &= \dfrac{1}{\pi}\left(u(\pi - \arccos(u)) + \sqrt{1 - u^2}\right).
        \end{aligned}
    \end{equation}
    \end{definition}
    
    In this thesis, we refer to the infinite neural tangent kernel as the NTK for brevity.  It should be noted that Definitions \ref{def:NTK} and \ref{def:recurNTK} are also valid in the space of $\Rd$. \citet{geifman2020rkhs} further define the normalized kernel $\frac{1}{L+1}k_{NTK}(\mathbf{x}, \mathbf{z}; L+1, \beta)$ when $\beta = 0$.  We improve on this by empirically finding the more general case of normalization.  Let $\beta \geq 0$, then it can be shown that
    \begin{equation}
        \label{eq:normNTK}
        \ddot{k}_{NTK}(\mathbf{x}, \mathbf{x}; L+1, \beta) = \dfrac{1}{(L+1)(\beta^2 + 1)}k_{NTK}(\mathbf{x}, \mathbf{x}; L+1, \beta) = 1\ 
    \end{equation}
    for all $\mathbf{x} \in \mathbb{R}^d$.  We will be using this normalized form for the remainder of our work.
    
    This recursive formalization depends entirely on the depth of the network and bias $\beta$.  In practice, we want to be able to find the optimal $\beta$ parameter for the given regression problem.  As such, this requires the gradient of $\ddot{k}_{NTK}$ which we define independent of input for a given depth $L+1$ and bias $\beta$
    \begin{equation}
        \label{eq:betaNTK}
        \dfrac{\partial}{\partial\beta}\ddot{k}_{NTK}(\mathbf{x},\mathbf{z}) = \dfrac{1}{(L+1)(\beta^2 + 1)}\left(\dfrac{\partial}{\partial\beta}k_{NTK}(\mathbf{x},\mathbf{z}) - \dfrac{2\beta}{\beta^2 + 1}k_{NTK}(\mathbf{x},\mathbf{z})\right).
    \end{equation}
    The full derivation of the gradient can be found in Appendix \ref{betagrad}.  From here, we will utilize the notation $D = L+1$ to refer to the \textit{depth} or the number of layers of a given network defined by the NTK.

    \section{RKHS Inclusion}\label{sec:inclusion}
    
    Given two kernels, it is natural to compare the space of functions that they are capable of producing and seeing if there is any overlap.  This is made possible by a consequence of Theorem \ref{thm:moorearon} which allows us to determine if one RKHS is a subset of another.  In regards to this thesis, we are interested in determining equality between the RKHSs of two kernels and thus equality of the kernels themselves.
    
    \begin{theorem}[Subset Inclusion \citep{berlinet2011reproducing}, p. 30]\label{thm:subset}
        Let $k_1$ and $k_2$ be continuous positive definite kernels on $\mathcal{X}_1\times\mathcal{X}_1$ and $\mathcal{X}_2\times\mathcal{X}_2$ respectively with $\mathcal{H}_{k_1}$ and $\mathcal{H}_{k_2}$ denoting their respective RKHS.  Then, $\mathcal{H}_{k_1} \subset \mathcal{H}_{k_2}$ if and only if there exists a constant $B$ such that $B^2k_2 - k_1$ is a positive definite kernel.
    \end{theorem}
    
    While this is a powerful theorem, in practice it presents a problem when trying to take into consideration a kernel's set of parameters in relation to an RKHS.  In fact, a parameter can have a great effect on the functions that belong to the RKHS.  This is illustrated in \citet[Lemma 3.1.2, pg. 34]{walder2008efficient} where we have two Gaussian kernels defined by length-scale parameters $\ell_1, \ell_2\in\mathbb{R}$. We let $\kGaus(\cdot, \cdot; \ell)$ be the general Gaussian reproducing kernel defined by some parameter $\ell\in\mathbb{R}$ that defines an RKHS $\mathcal{H}_{k}$.  If $\ell_1 > \frac{1}{2}\ell$ and $\ell_2 > \frac{1}{2}\ell$, then by utilizing the inner product of $\mathcal{H}_{k}$ we can create a new reproducing kernel defined in that space for some inputs $\mathbf{x},\mathbf{z}\in\Rd$
    \begin{equation}
        \langle \kGaus(\cdot, \mathbf{x}; \ell_1), \kGaus(\cdot, \mathbf{z}; \ell_2) \rangle_{\mathcal{H}_k} = \kGaus(\mathbf{x},\mathbf{z}; \ell_1 + \ell_2 - \ell)
    \end{equation}
    which shows that there exists a new kernel function with length-scale $\ell_1 + \ell_2 - \ell$ that is a member of $\calH_k$.  However, if one of the inequalities is not satisfied, then the kernel corresponding to that length-scale is \textit{not} in the RKHS.
    
    Another way to look at this is that by scaling the parameter of the Gaussian kernel, you also scale the $\Rd$ input space \citep[Proposition 4.37, pg. 132]{steinwart2008support}.  Without the rescaling the input space, there is no guarantee that the scaled Gaussian kernel will maintain the same RKHS.
    
    This challenges the notion that a kernel can produces an all-encompassing RKHS independent of its set of parameters.  Hence, a kernel and its parameters must be observed together in the practical analysis of RKHSs.  There are two ways to think about this in terms of equality between two RKHSs: there are specific parameters for which the RKHSs produce the same set of functions \textit{or} there is a family of many RKHSs over all possible parameters for which all possible sets of functions are the same.  In the upcoming sections and chapters we analyze the practical equivalence of the Laplace kernel and the NTK by viewing their equivalence through parameter matching.
    
    \section{Equivalence of the Laplace and Neural Tangent Kernels}
    \label{sec:equivalence}
    
    It is clear that neural networks and kernel methods have some latent overlap. \citet{belkin2018kernel} empirically found similarities between the Laplace kernel and ReLU activated neural networks when used on the task of fitting random labels.  As such, it motivates the question: \textit{Do the Laplace and neural tangent kernels have the same RKHS and if so to what extent?}  This question is answered in theory by \citet{geifman2020rkhs} and \citet{chen2021rkhs} who showed subset equality between the Laplace RKHS $\mathcal{H}_{Lap}$ and the NTK RKHS $\mathcal{H}_{NTK}$ in the space of $\Sd$.  The forward direction \citep{geifman2020rkhs} was shown by eigenvalue analysis by way of Mercer's Theorem \ref{thm:mercer}.  The backward direction \citep{chen2021rkhs} was shown by utilizing the Subset Inclusion Theorem \ref{thm:subset} and singularity analysis. 
    
    This result begs a further question: \textit{What does the practical equivalence of these kernels look like?}  Since they are dual representations of each other, this poses a challenge because the NTK relies on depth $D$ that is in the natural numbers which is in great contrast to the Laplace parameter $\ell$ which is in the positive reals.  Due to the vast differences in parameterization, we hypothesize that the bias $\beta$ plays a role in bridging the gap between the depths.
    
    We begin by considering the matching of the neural tangent kernel $\ddot{k}_{NTK}$ for set depth $D\in\mathbb{N}$ and $\beta\in\mathbb{R}^+$ with the Laplace kernel $k_{Lap}$ for some $\ell\in\mathbb{R}^+ - \{0\}$.  It would then suffice that finding a matching kernel for some inputs $\mathbf{x},\mathbf{z}\in\Sd$ would be shown as follows:
    \begin{align}\label{eq:matching}
        k_{Lap}(\mathbf{x},\mathbf{z}; \ell) &= \ddot{k}_{NTK}(\mathbf{x},\mathbf{z}; D, \beta) \nonumber \\
        \exp\left(-\frac{\|\mathbf{x}-\mathbf{z}\|}{\ell}\right) &= \ddot{k}_{NTK}(\mathbf{x},\mathbf{z}; D, \beta) \nonumber \\
        \frac{\|\mathbf{x}-\mathbf{z}\|}{\ell} &= -\log\left(\ddot{k}_{NTK}(\mathbf{x},\mathbf{z}; D, \beta)\right) \nonumber \\
        \ell &= \dfrac{\|\mathbf{x}-\mathbf{z}\|}{-\log\left(\ddot{k}_{NTK}(\mathbf{x},\mathbf{z}; D, \beta)\right)}.
    \end{align}
    Further, we define $d_\theta$ as a measure of differences between the kernels given a set of parameters $\theta$:  
    \begin{equation}
        d_{D,\beta,\ell}(\mathbf{x},\mathbf{z}) = \left|\ddot{k}_{NTK}(\mathbf{x},\mathbf{z}; D, \beta) - k_{Lap}(\mathbf{x},\mathbf{z}; \ell)\right|.
    \end{equation}
    For the kernels to be the same, it is necessary that Equation \eqref{eq:matching} provides a length-scale where the kernels are indeed the same.  This is consequential because if this can be done consistently, then the kernels are interchangeable regardless of any optimization procedures that are done while utilizing kernel methods.  This brings us to the motivating question: \textit{Can we find a length-scale for which both kernels are identical?}
    
    \begin{table}[t]
        \centering
        \begin{tabular}{|c||c|c|}
            \cline{2-3}
            \multicolumn{1}{c|}{\multirow{2}{*}{Inputs}} & \multicolumn{2}{c|}{$d_{D=3, \beta, \ell}(\mathbf{x}_i, \mathbf{z}_i)$} \\
            \cline{2-3}
            \multicolumn{1}{c|}{} & \small $\beta=0,\ \ell\approx1.815$  & \small $\beta\approx2.122,\ \ell\approx2.036$ \\
            \hline
            \small $\mathbf{x}_1 = [0.8027\ 0.2299\ 0.5503]$ & \multirow{2}{*}{$\approx 0.0$} & \multirow{2}{*}{0.001296} \\
            \small $\mathbf{z}_1 = [0.7982\ 0.3818\ 0.4658]$ & & \\
            \hline
            \small $\mathbf{x}_2 = [0.0389\ 0.9663\ 0.2545]$ & \multirow{2}{*}{0.0980} & \multirow{2}{*}{0.0000187} \\
            \small $\mathbf{z}_2 = [0.6941\ 0.5958\ 0.4040]$ & & \\
            \hline
        \end{tabular}
        \caption{\setlinespacing{1.1} An illustration of the discrepancy between kernel differences $d_\theta$ while trying to match $\ell$ to $\ddot{k}_{NTK}$ of depth $D=3$ using Equation \eqref{eq:matching}.  \textit{Left column:} Random inputs in $\Sd$.  \textit{Right table:} Table of differences for specific parameters.  Column 1 fixes $\beta=0$ during matching. Column 2 optimizes $\beta$ and $\ell$ using Algorithm \ref{alg:length-scalebiasmatching}.}
        \label{tab:matching}
    \end{table}
    
    Table \ref{tab:matching} showcases this conundrum.  Using Equation \eqref{eq:matching} we find that for fixed $D$ and $\beta$ we can find $\ell$ such that $d_{\theta}$ of the two kernels is near zero for a \textit{single} input but not for any other input as illustrated in the 1st $d_{\theta}$ column in Table \ref{tab:matching}.  Here $\mathbf{x}_1, \mathbf{z}_1$ produce a difference near zero but $\mathbf{x}_2, \mathbf{z}_2$ are nearly 0.1 apart which indicates that this type of matching only works on an input by input basis.
    
    \begin{algorithm}[p]
        \DontPrintSemicolon
        \SetKwInOut{Input}{Input}
        \SetKwInOut{Return}{Return}
        
        \Input{$dim, n, D \in \mathbb{N}$ and $b, seed\in\R$}
        \BlankLine
        Initialize empty list for means $M$ and variances $V$\;
        Initialize list of $\beta$ values $B$ in range $[0, b]$\;
        \ForEach{$\beta$ in list $B$}{
            Set $seed$\;
            Initialize $\ddot{k}_{NTK}(D, \beta)$\;
            Initialize empty list for length-scales $L$\;
            \For{$i \leftarrow 0$ \KwTo $n$}{
                $\mathbf{x}_i \leftarrow $ vector size $dim$ with random $x$ entries normalized to $\Sd$\;
                $\mathbf{z}_i \leftarrow $ vector size $dim$ with random $z$ entries normalized to $\Sd$\;
                Append length-scale $\ell_i = \frac{||\mathbf{x}_i-\mathbf{z}_i||}{-\log(\ddot{k}_{NTK}(\mathbf{x}_i,\mathbf{z}_i; D, \beta))}$ to $L$\;
            }
            Append $\bar{\ell} = \mathbb{E}[L]$ to $M$\;
            Append $\text{var}[L]$ to $V$\;
        }
        \BlankLine
        \Return{$\bar{\ell}$ in $M$ and $\beta$ in $B$ corresponding to $\min V$}
        \caption{\setlinespacing{1.1} Length-scale and bias matching procedure.}
        \label{alg:length-scalebiasmatching}
    \end{algorithm}
    
    To improve on this, we attempt to vary $\beta$ and try to find one for a given $D$ such that the $\ell$ found using Equation \eqref{eq:matching} matches \textit{for all $\mathbf{x},\mathbf{z} \in \Sd$}. This procedure is described in Algorithm \ref{alg:length-scalebiasmatching}. The idea behind this procedure is that we cannot directly find an $\ell$ and $\beta$ that produce $d_\theta = 0$ using all possible $\mathbf{x},\mathbf{z}$ so instead we can approximate it by using $n\in\mathbb{N}$ points $(\mathbf{x}_i, \mathbf{z}_i)$ where each $\mathbf{x}_i, \mathbf{z}_i\in\Sd$.  We set $D$ of the NTK and calculate $\ell$ for some $\beta$ over every point using Equation \eqref{eq:matching}.  We then take the mean and variance of the resulting set of length-scales $L = \{\ell_1,\dots,\ell_n\}$ for that specific $\beta$ and record it.  We repeat this over $m\in\mathbb{N}$ number of $\beta$'s bounded between 0 and some upper bound (dependent on $D$).  The minimum variance in this experiment indicates that the corresponding mean length-scale $\bar{\ell}_j$ and $\beta_j$ where $j\in \left\{1,\dots,m\right\}$ are the optimal kernel parameters that produce a small or approximately zero $d_\theta$.  This can be summarized as follows:
    \begin{equation}
        \ell, \beta = \argmin_{\bar{\ell}_j, \beta_j} \left\{ \text{var}\left[L\right]\quad \forall\ \beta_1,\dots,\beta_m \right\}.
    \end{equation}
    As $n\to\infty$ the entire space gets utilized so the resulting $\ell$ and $\beta$ should be optimal.
    
    \begin{figure}[t]
        \centering
        \makebox[\textwidth]{\includegraphics[width=\textwidth]{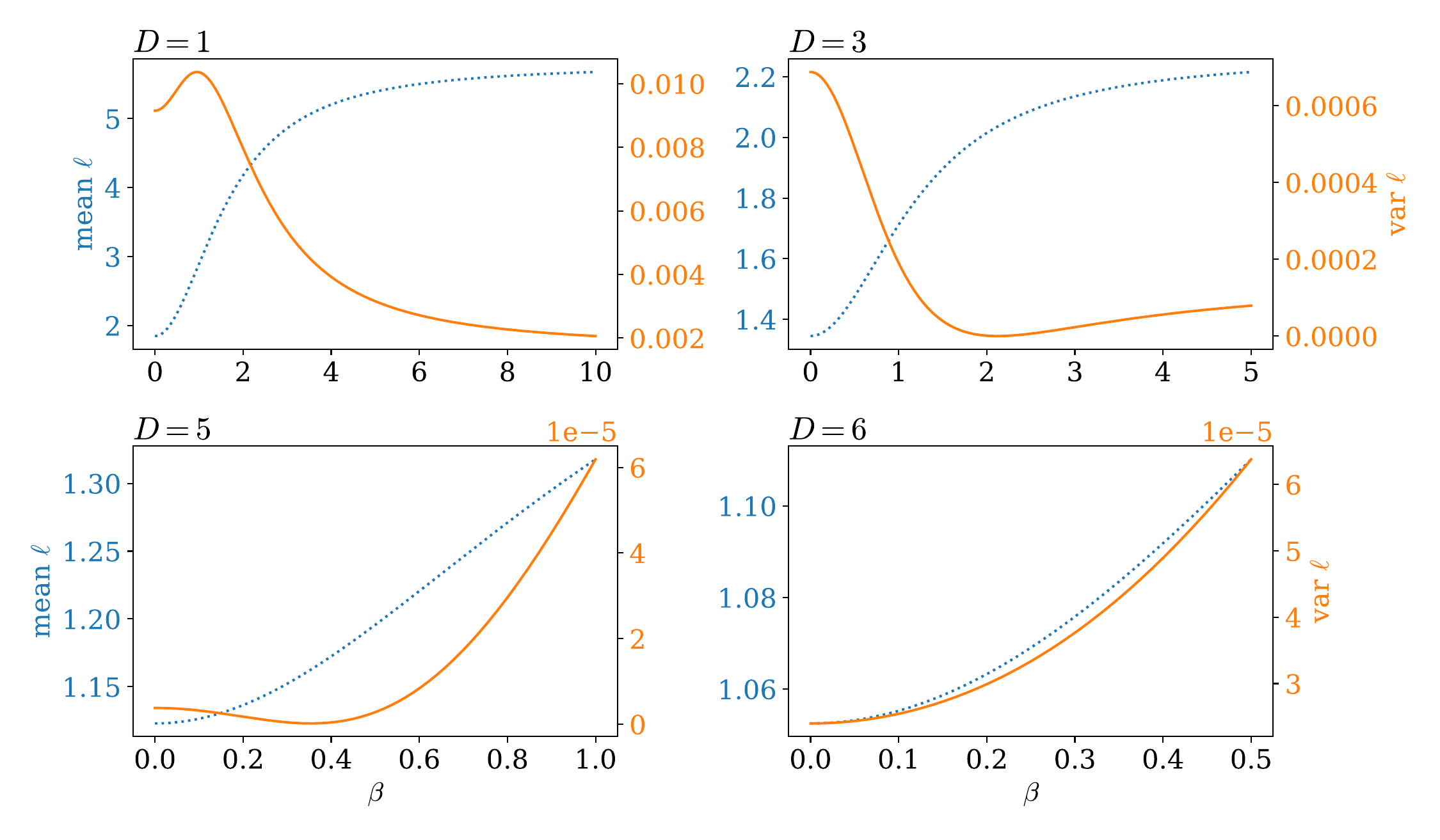}}
        \caption{\setlinespacing{1.1} Mean and variance plots of $\ell$ given specific $\beta$ calculated using $n=1000$ sample of input pairs for various depths.  The solid orange line represents the variance while the dotted blue line represents the mean.}
        \label{fig:varopt}
    \end{figure}  
    
    In our experiments, for depths 1 and 2, $\text{var}[L]\to 0$ as $\beta\to\infty$, depths 3, 4, and 5 attain non-trivial minimums at various $\beta$, and for depth $> 6$ the global minima is near zero. Figure \ref{fig:varopt} illustrates the global minima for depths 1, 3, 5, and 6. The 2nd $d_\theta$ column in Table \ref{tab:matching} illustrates the effect using the optimal $\ell$ and $\beta$ values.  We can see that both points have small $d_\theta$.
    
    \begin{figure}[t]
        \centering
         \makebox[\textwidth]{\includegraphics[width=0.65\textwidth]{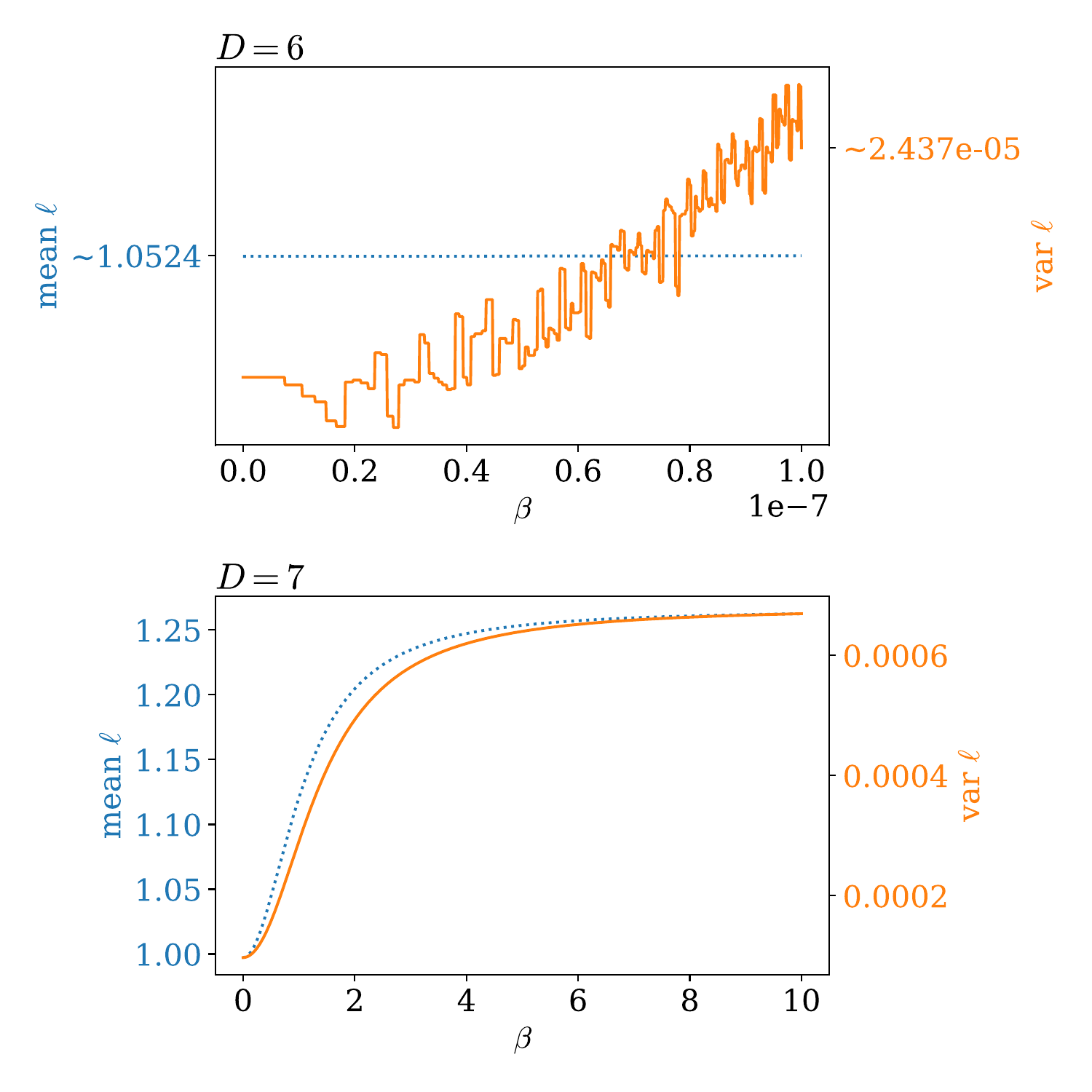}}
        \caption{\setlinespacing{1.1} Solid orange line represents the variance and the dotted blue line represents the mean. \textit{Top:} A zoom in of depth $D=6$ for $\beta\in[0, 10^{-7}]$.  Due to the zoom, the mean values are all concentrated around $\approx 1.0524$ and all variance values are near $\approx 2.437\cdot 10^{-5}$.  The difference between the minimum and maximum variance shown is approximately $10^{-18}$. \textit{Bottom:} A showcase of a typical plot past depth 6.}
        \label{fig:varopt2}
    \end{figure}
    
    Depth 6 is an interesting case because as seen in Figure \ref{fig:varopt2}, as we look closer to zero we are still unable to attain a global minimum for the variance despite it appearing that $\beta = 0$.  The rough looking nature of the plot is due to $\beta$ being squared in the NTK formulation which results in values less than $10^{-14}$.  The floating point precision of \texttt{numpy} for the machine used to compute this is $10^{-15}$.  Beyond depth 6, $\beta$ needs to be orders of magnitude smaller and thus cannot be accurately computed. 
    
    In summary, we gained significant insight to the empirical matching of these kernels.  In order to equate the kernels, that is, over all possible inputs, we are indeed dependent on both the Laplace kernel length-scale $\ell$ \textit{and} the NTK bias $\beta$.  In addition, through Figures \ref{fig:varopt} and \ref{fig:varopt2} we have evidence to support the idea that as the NTK depth increases, the optimal bias and length-scale both decrease.  From the context of machine learning methods, increasing the depth of a neural network also increases the generalization properties and the susceptibility for overfitting.  This is reflected in kernel methods, where a relatively low length-scale parameter produces functions that capture more fine grained detail by closer interpolating over the given dataset (e.g.\ overfitting).
    
    \chapter{Synthetic Experiments}\label{ch:synthexp}
    
    In this chapter we analyze the similarities and differences between the Laplace kernel and the NTK over a number of synthetically generated datasets using GP regression.  Specifically, we are interested in matching the posterior means generated by the GP regressor under kernel assumptions, determining the effectiveness of matching in $\Rd$ and $\Sd$, and analyzing the influence of data transformations on the quality of posterior mean predictions for the NTK.  Section \ref{sec:setup} describes the experimental setup used in this chapter,  Section \ref{sec:illustrative} focuses on posterior mean matching when $d=2$, Section \ref{sec:2D} showcases the the differences in posterior mean matching in $\R^2$ and $\Ss^1$ on more complex surfaces, and Section \ref{sec:HD} deals with the quality of posterior mean predictions using high dimensional input datasets.
    
    \section{Setup}\label{sec:setup}
    
    In our experiments, we utilize \texttt{scikit-learn} \citep{scikit-learn} which is a machine learning library for the \texttt{Python} programming language.  One contribution of this thesis is an implementation of the NTK\footnote{\url{https://github.com/392781/scikit-ntk}} that is directly compatible with \texttt{scikit-learn}'s GP module.  Using this implementation we are able to compute NTK values, optimize the NTK's bias, and train GP models that represent infinite width neural networks.  We also use their GP module for the Mat\'ern kernel which yields the Laplace $(\nu = \frac{1}{2})$ and Gaussian $(\nu \to \infty)$ kernels in our experiments.
    
    We generate both noisy and noiseless synthetic data for our experiments.  The added data noise $\epsilon$ is normally distributed with mean zero and variance $\sigma^2$ chosen depending on the dataset,
    \begin{equation}
        \epsilon \sim \mathcal{N}(0, \sigma^2).
    \end{equation}
    The number of samples we generate depends on the data.  In addition, the data is split 50-50 into a training set $(X, \mathbf{y})$ and a testing set $(X_*, \mathbf{y}_*)$ for each experiment.  Lastly, we normalize and/or rescale our datasets depending on experiment.  \textit{Normalizing} in this context means that we map $\Rd$ inputs (each observation or row vector in $X$) to the $d$-sphere space $\Sd$ by way of the $L_2$-norm:
    \begin{equation}
        \mathbf{x} \mapsto \dfrac{\mathbf{x}}{||\mathbf{x}||_{L_2}}.
        \label{eq:normalization}
    \end{equation}
    On the other hand, \textit{rescaling} refers to subtracting the sample mean $m$ and dividing sample variance $s^2$ for a variable (column vector in $X$) $\mathbf{x}$: 
    \begin{equation}
        \dfrac{\mathbf{x} - m}{s^2}.
        \label{eq:rescaling}
    \end{equation}
    
    Our GP regressor setup begins with the kernels; namely, the NTK, Laplace, and Gaussian kernels.  During GP regression, the kernel hyperparameters are trained via optimization to best fit the data.  Up until now we have been referring to kernel \textit{parameters} because those values are intrinsic to the definition of a kernel.  However, in the context of GPs, we refer to those kernel parameters as hyperparameters since they do not directly influence the definition of a GP model. Thus, we define \textit{hyperparameters} as the parameters that are optimized in order to tune a specific model but do not directly describe that model.  
    
    All kernels contain the constant value parameter $c$.  In the case when a dataset is noisy, we include a white noise variance $\sigma^2$.  There are additional hyperparameters that are used to define specific kernels: $D$ and $\beta$ for the NTK and $\nu$ and $\ell$ for the Mat\'ern kernel.  All the listed hyperparameters can be \textit{unfixed} meaning that they can be optimized during training by maximizing the marginal log likelihood of the GP as outlined by \citet{rasmussen2006gpml}.
    
    Furthermore, GPs contain optimization specific options: $n_{restart}$, $\mathbf{y}_{rescale}$, and $\alpha$.  $n_{restart}$ controls the number of optimizer restarts that are done during training in order to find the optimal hyperparameters.  Each restart randomly chooses a new initial value from within pre-specified bounds which helps combat what may be a complicated loss surface with many local minima.  $\mathbf{y}_{rescale}$ is a boolean value that when true, rescales the response variable during training in order to aid with fitting and optimization.  It should be noted that $\mathbf{y}_{rescale}$ is undone at time of prediction.  Finally, $\alpha$ is a small positive value that is added to the diagonal of $\left[K_{XX} + \sigma^2 I_n\right]^{-1}$ (Equation \eqref{eq:posteriordist}) to ensure positive definiteness in light of any numerical issues.  All relevant modifications and parameters are summarized in Table \ref{tab:expparams}.
    
    \begin{table}[t!]
        \centering
        \begin{tabular}{|c|c|c|}
            \hline
            Data & Description & Value \\
            \hline
            Normalization & Transforming data to $\Sd$ (Equation \eqref{eq:normalization}) & -- \\
            Rescaling & Subtracting $m$, dividing $s^2$ (Equation \eqref{eq:rescaling}) & -- \\
            \hline
            \hline
            Optimization & Description & Value\\
            \hline
            $n_{restart}$ & Number of optimizer restarts & $9$ \\
            $\mathbf{y}_{rescale}$ & Response variable rescaling during training & \texttt{true} \\
            $\alpha$ & Value to ensure positive definiteness & $10^{-5}$ \\
            \hline
            \hline
            Kernel & Description & Value \\
            \hline
            $D$   & NTK depth & $2, 3, 10$ \\
            $\beta$ & NTK bias & Unfixed \\
            $\nu$ & Mat\'ern kernel smoothness & $\dfrac{1}{2}, \infty$ \\
            $\ell$ & Length-scale & Unfixed \\
            $c$ & Constant value & Unfixed \\
            $\sigma^2$ & Data noise variance & Unfixed \\
            \hline
        \end{tabular}
        \caption{\setlinespacing{1.1} Summary of variables managed in all synthetic experiments.  Values left blank are determined per experiment.  Unfixed values are ones allowed to be optimized during experiments.}
        \label{tab:expparams}
    \end{table}
    
    Our main task throughout this chapter is to perfectly match the posterior means of the Laplace kernel and NTK.  We use the Gaussian kernel as a comparison that is outside the RKHS in question for our experiments. To accomplish this task we will be optimizing Mat\'ern kernels with $\nu=\frac{1}{2}$ (Laplace) and $\nu \to \infty$ (Gaussian). We begin with an objective function as outlined in Algorithm \ref{alg:objfunc} that we will minimize during the optimization process.
    
    \begin{algorithm}[p]
        \DontPrintSemicolon
        \SetKwInOut{Input}{Input}
        \SetKwInOut{Return}{Return}
        
        \Input{$k_{Mat}(\theta), \mathbf{\bar{f}}_{NTK_*}, X, \mathbf{y}, X_*$}
        \BlankLine
        \BlankLine
        $f_{Mat_{opt}}, \theta_{opt} \leftarrow$\textsc{optimize}$\left\{ f_{Mat}\sim \mathcal{GP}(0, k_{Mat}(\theta))|_{X,\mathbf{y}} : n_{restart}, \mathbf{y}_{rescale}, \alpha \right\}$\;
        $\mathbf{\bar{f}}_{Mat_*} \leftarrow f_{Mat_{opt}}(X_*)$\;
        \BlankLine
        \BlankLine
        \Return{$\left\|\mathbf{\bar{f}}_{NTK_*} - \mathbf{\bar{f}}_{Mat_*}\right\|_{L_2}$}
        
        \caption{\setlinespacing{1.1} Objective function \texttt{obj\_func} for posterior mean matching optimization.}
        \label{alg:objfunc}
    \end{algorithm}
    
    \begin{algorithm}[p]
        \DontPrintSemicolon
        \SetKwInOut{Define}{Define}
        \SetKwInOut{Input}{Input}
        \SetKwInOut{Return}{Return}
        
        \Define{$n_{restart}$, $\mathbf{y}_{rescale}$, $\alpha$, $noise$}
        \Input{$D, \nu, X, \mathbf{y}, X_*$ (Train and test data pre-processed the same way)}

        \BlankLine
        \BlankLine
        \emph{NTK fitting and optimization}\;
        $\ddot{k}_{NTK}(\theta) \leftarrow c\cdot \ddot{k}_{NTK}(D, \beta);\quad \theta := \{D, \beta, c\}$\;
        \lIf{$noise$ is True}{$\ddot{k}_{NTK}(\theta) \leftarrow \ddot{k}_{NTK}(\theta) + \sigma^2 \delta_{\mathbf{s} = \mathbf{t}};\quad \theta = \{\dots, \sigma^2\}$}
        $f_{NTK_{opt}}, \theta_{opt} \leftarrow$ \textsc{optimize}$\left\{ f_{NTK} \sim \mathcal{GP}\right(0, \ddot{k}_{NTK}(\theta)\left)|_{X, \mathbf{y}} : n_{restart}, \mathbf{y}_{rescale}, \alpha \right\}$\;
        $\mathbf{\bar{f}}_{NTK_*} \leftarrow f_{NTK_{opt}}(X_*)$\;
        $\beta_{opt}$, $c_{opt}, \sigma^2_{opt} \leftarrow \{\dots, \beta, c, \sigma^2\}_{\theta_{opt}}$\;
        \BlankLine
        \BlankLine
        \emph{Mat\'ern posterior matching}\;
        $k_{Mat}(\theta) \leftarrow c_{opt}\cdot k_{Mat}(\nu, \ell);\quad \theta := \{\nu, \ell, c_{opt}\}$\;
        \lIf{$noise$ is True}{$k_{Mat}(\theta) \leftarrow k_{Mat}(\theta) + \sigma^2 \delta_{\mathbf{s} = \mathbf{t}});\quad \theta = \{\dots, \sigma^2\}$}
        $f_{Mat_{opt}}, \theta_{opt} \leftarrow$\textsc{minimize}$\left\{\right.$\texttt{obj\_func}$\left. (k_{Mat}(\theta), \mathbf{\bar{f}}_{NTK_*}, X, y, X_*) \right\}$\;
        $\ell_{opt} \leftarrow \{\dots, \ell, \dots\}_{\theta_{opt}}$\;
        \BlankLine
        \BlankLine
        
        \Return{$c_{opt}, \sigma^2_{opt}, \beta_{opt}, \ell_{opt}, f_{NTK_{opt}}, f_{Mat_{opt}}$}
        
        \caption{\setlinespacing{1.1} Posterior mean matching for Mat\'ern kernels.  The optimized model is denoted as $f_{opt}(\cdot) := f(\cdot; \theta_{opt})$.  \textsc{optimize} refers to the GP optimization process of maximizing the marginal log likelihood.}
        \label{alg:hyperparam}
    \end{algorithm}
    
    The purpose of the objective function is to minimize the RMSE between the Mat\'ern and NTK posterior means by varying the Mat\'ern kernel parameter $\ell$,
    \begin{equation}
        \left\|\mathbf{\bar{f}}_{NTK_*} - \mathbf{\bar{f}}_{Mat_*}\right\|_{L_2}.
    \end{equation}
    The procedure is outlined below in Algorithm \ref{alg:hyperparam} utilizes parameters for the GPs and hyperparameters for the kernels as expressed in Table \ref{tab:expparams}.  It should be noted that \textsc{optimize} utilizes \texttt{sklearn.gaussian\_process.GaussianProcessRe}-\texttt{gressor}'s \texttt{fit} function and \textsc{minimize} uses the \texttt{scipy.optimize.minimize\_scalar} function to perform the key steps in our procedure.
    
    We control NTK's depth parameter $D$ and the choice of Mat\'ern kernel smoothness $\nu$.  It should be noted that for our experiments we only allow the constant value $c$ and error term $\sigma^2$ parameters to optimize for the NTK while the respective parameters for the Mat\'ern kernels inherit the NTK's optimized values.  The reasoning lies in the relation between a kernel and a GP covariance function.  For values $\mathbf{s},\mathbf{t}\in\calX$, the covariance function of a GP $f\sim \mathcal{GP}(0, k)$ is related directly to its kernel (and thus $f$ depends on choice of $k$):
    \begin{equation}
        \cov(f(\mathbf{s}), f(\mathbf{t})) = k(\mathbf{s},\mathbf{t})
    \end{equation}
    So, a scaled GP is expressed as $\sqrt{c}f\sim \mathcal{GP}(0, ck)$ because its covariance function simplifies as follows:
    \begin{equation}\label{eq:scaledGP}
        \cov(\sqrt{c} f(\mathbf{s}), \sqrt{c} f(\mathbf{t})) = (\sqrt{c})^2 \cov(f(\mathbf{s}), f(\mathbf{t})) = c k(\mathbf{s},\mathbf{t}).
    \end{equation}
    Lastly, a GP that includes an additive noise term is expressed as $f + \epsilon\sim \mathcal{GP}(0, k+\sigma^2\delta_{\mathbf{s}=\mathbf{t}})$ where $\sigma^2\delta_{\mathbf{s}=\mathbf{t}}$ is the \textit{white noise kernel} with a constant value represented by the noise variance $\sigma^2$ and the kernel itself being the Kronecker delta $\delta_{\mathbf{s}=\mathbf{t}}$.  This is justified because
    \begin{equation}
        \begin{gathered}
        \cov(f(\mathbf{s}) + \epsilon(\mathbf{s}), f(\mathbf{t}) + \epsilon(\mathbf{t})) = \\
        \cov(f(\mathbf{s}), f(\mathbf{t}))\ + \\
        \cov(f(\mathbf{s}), f(\mathbf{t}) + \epsilon(\mathbf{t})) + \cov(\epsilon(\mathbf{s}), f(\mathbf{t}) + \epsilon(\mathbf{t}))\ + \\
        \cov(f(\mathbf{s}) + \epsilon(\mathbf{s}), f(\mathbf{t})) + \cov(f(\mathbf{s}) + \epsilon(\mathbf{s}), \epsilon(\mathbf{s}))\ + \\
        \cov(\epsilon(\mathbf{s}), \epsilon(\mathbf{t})) \\
        = k(\mathbf{s}, \mathbf{t}) + \sigma^2 \delta_{\mathbf{s} = \mathbf{t}}.
        \end{gathered}
    \end{equation}
    As a result, if we wish to compare just the kernels while maintaining their expressivity through transformations, it is necessary to keep the constant value and noise variance terms the same between kernels.
    
    We find that our procedure is a difficult optimization task due to the objective function landscape containing many local minima.  Our \texttt{Python} implementation\footnote{Implementation details at \url{https://github.com/392781/neural-laplace}.} deals with this by modifying the minimization process to search by slowly expanding the search bounds for $\ell$ over multiple runs. In our experiments we consider two metrics for posterior mean matching: Root mean-squared error (RMSE) and Pearson correlation coefficient ($\rho$).  The former gives an idea of the ability to minimize the objective function's loss and latter is used to see how well the posterior means of different kernels overlap with each other.
    
    \section{Illustrative Example}\label{sec:illustrative}
    
    To motivate the remainder of this work, we will first attempt to answer the following question and analyze the results produced: \textit{Does our posterior mean matching procedure work?}  We begin by creating a simple 2D input parametric curve as follows:
    \begin{equation}\label{eq:parametric}
        \begin{gathered}
            x_1 = (y^2 + 1)\cdot \sin(t) \\
            x_2 = (y^2 + 1)\cdot \cos(t) \\ 
            y \in [-2, 2] \\
            \text{where } t\in[-2\pi, 2\pi].
        \end{gathered}
    \end{equation}
    
    We independently sample 100 values of $t\sim U[-2, 2]$ and $y\sim U[-2\pi, 2\pi]$ where $U[a,b]$ is the uniform distribution between points $a$ and $b$, inclusive.  We generate $x_1$ and $x_2$ using the $t$ and $y$ samples.      Although the curve is defined by starting with samples of $y$ and $t$, we treat points $(x_1,x_2)$ as the inputs and $y$ as the output during fitting. We chose this curve because it very clearly illustrates the posterior mean matching between the NTK and Mat\'ern kernels. The dependence on $t$ allows us to connect the points in a way that other surfaces and datasets do not.  We then split our inputs $(x_1,x_2)$ and output $y$ into a training set size $n=50$ and testing set size $m=50$.  We make two separate experiments: non-noisy and noisy with $\epsilon \sim \mathcal{N}(0,0.15^2)$.  We utilize the experiment setup outlined in Table \ref{tab:expparams}.  In addition, we normalize the data inputs $(x_1, x_2)$ to $\Ss^1$ prior to training.  We then perform the posterior mean matching as outlined in Algorithm \ref{alg:hyperparam}.
    \begin{figure}[t!]
        \centering
        \makebox[\textwidth]{\includegraphics[width=\textwidth]{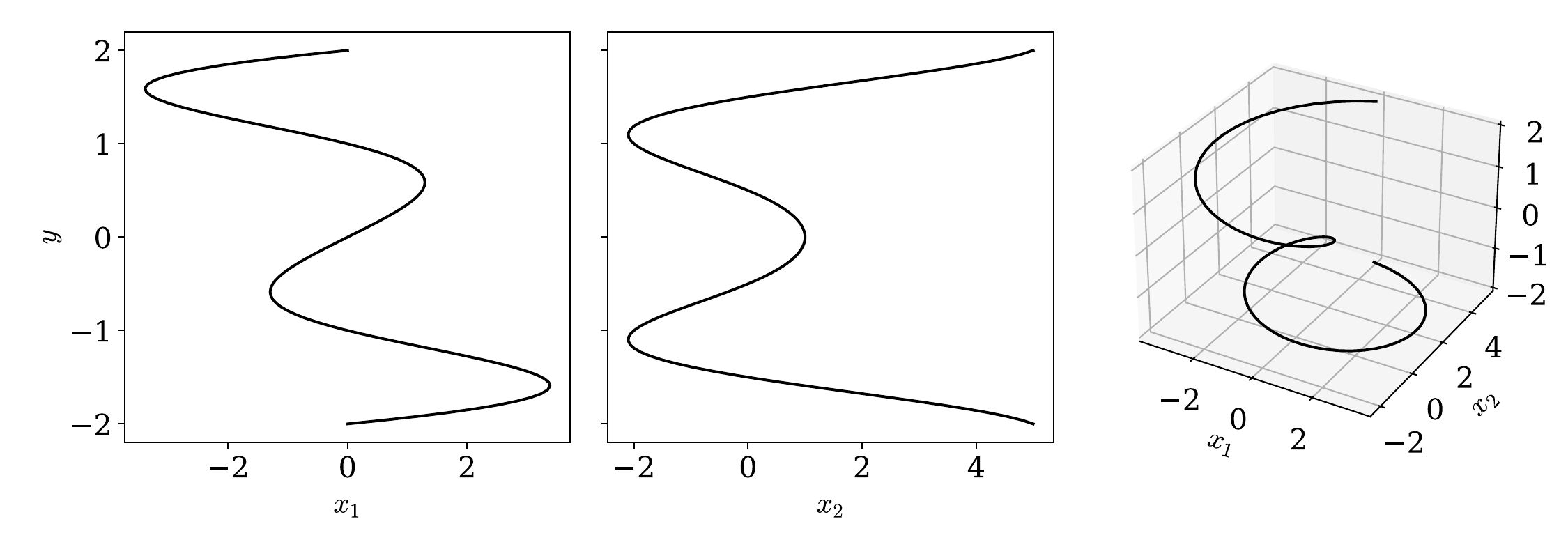}}
        \makebox[\textwidth]{\includegraphics[width=\textwidth]{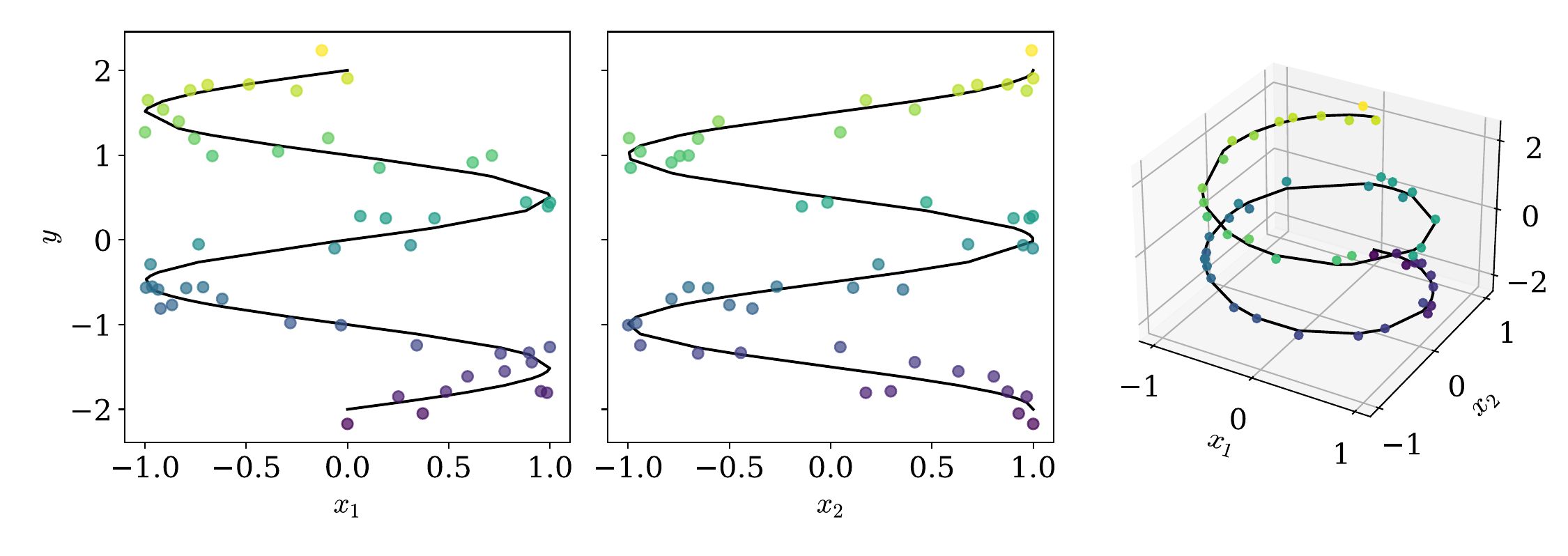}}
        \caption{\setlinespacing{1.1} \textit{Top:} The parametric curve defined in Equation \eqref{eq:parametric}. \textit{Bottom:} Equation \eqref{eq:parametric} with inputs normalized and noisy training points shown.}
        \label{fig:parametric}
    \end{figure}

    \begin{table}[!ht]
        \centering
        \begin{tabular}{|c|c||c|c||c|c||c|c|}
        \cline{3-8}
        \multicolumn{2}{c|}{} & \multicolumn{2}{c||}{$D=2$} & \multicolumn{2}{c||}{$D=3$} & \multicolumn{2}{c|}{$D=10$} \\
        \hline 
        Metric & Noise & Lap & Gaus & Lap & Gaus & Lap & Gaus \\
        \hline
        \multirow{2}{*}{\small RMSE} & No & 0.0860 & 0.3352 & 0.0250 & 0.3610 & 0.0079 & 0.3670 \\
        & Yes & 0.0563 & 0.0228 & 0.0314 & 0.0412 & 0.0494 & 0.0915 \\
        \hline
        \hline
        \multirow{2}{*}{$\rho$} & No & 0.9974 & 0.9590 & 0.9998 & 0.9509 & 0.9999 & 0.9462 \\
        & Yes & 0.9954 & 0.9991 & 0.9984 & 0.9980 & 0.9929 & 0.9783 \\
        \hline
        \end{tabular}
        \caption{\setlinespacing{1.1} The results of posterior mean matching the NTK to Mat\'ern kernels for the parametric dataset in $\Ss^1$.}
        \label{tab:parametricsd}
    \end{table}
    \begin{figure}[!b]
        \centering
        \includegraphics[width=0.85\textwidth]{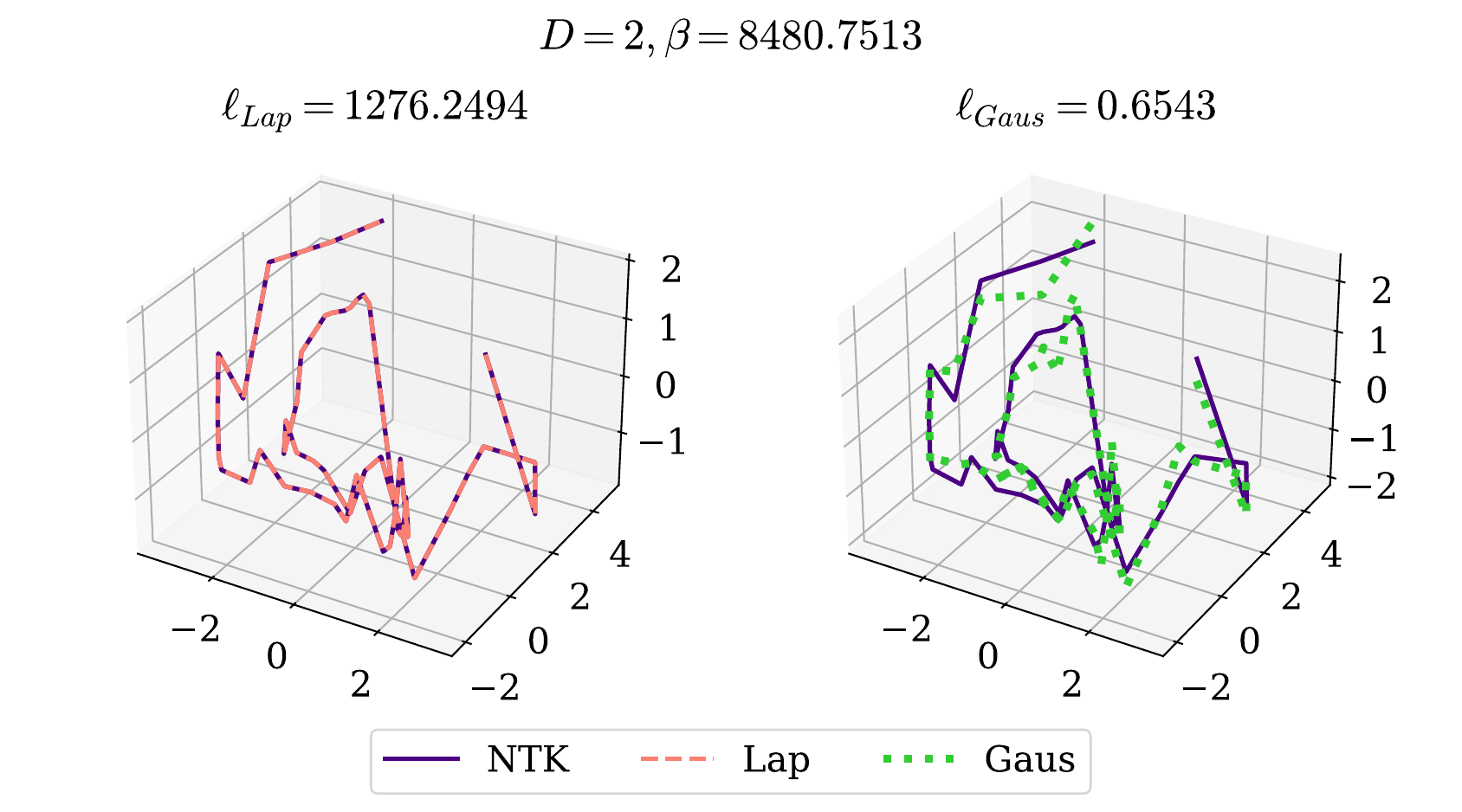}
        \caption{\setlinespacing{1.1} Posterior means generated by fitting non-noisy parametric curve data in $\Ss^1$ and predicted on out of sample data in $\Ss^1$.  For visualization purposes we set $(x_1,x_2) \in \R^2$.}
        \label{fig:parametricSd}
    \end{figure}
    
    Fitting the parametric curve shows a number of interesting results.  The optimization procedure performs well and can accurately match the mean posteriors of the Laplace and NTK GPs.  Table \ref{tab:parametricsd} indicates this with very small RMSE values and Figure \ref{fig:parametricSd} visually shows an exact match.  More than that, the posterior means correlate almost perfectly ($\rho > 0.99$) regardless if the GPs were trained on noisy data or not.

    On the other hand, the Gaussian kernel and NTK fail to perfectly match up in the case where no noise is present as can be seen in Table \ref{tab:parametricsd}.  Although there is a high correlation between the posterior means, Figure \ref{fig:parametricSd} illustrates how the Gaussian kernel cannot match the NTK's predictions to the exact effect that the Laplace kernel does.  This is also indicated by a higher resulting RMSE during optimization.
    
    One important thing of note is the low RMSE and high correlation values in Table \ref{tab:parametricsd} for both the Laplace and Gaussian kernels when noisy data is present.  This can be explained by the fact that adding the noise term $\sigma^2$ to our kernels gives the resulting regression additional smoothness and increased bounds for error while fitting.  When matching posterior means, the regressors can use this inferred noise to their advantage to better minimize our objective function resulting in the values seen in the table.  Without $\sigma^2$, our GP regressors would attempt to perfectly interpolate through the data.

    \begin{table}[!b]
        \centering
        \begin{tabular}{|c|c||c|c|c|c|c|}
            \hline
            D & Noise & $\beta$ & $\ell_{Lap}$ & $\ell_{Gaus}$ & $c$ & $\sigma^2$ \\
            \hline
            \multirow[c]{2}{*}{$2$} & No & 8480.751 & 1276.249 & 0.6543 & 32.495 & --\\
             & Yes & 0.000010 & 1.0569 & 0.8414 & 0.2734 & 0.7910 \\
             \hline
            \multirow[c]{2}{*}{$3$} & No & 453.111 & 765.133 & 0.6540 & 21.515 & -- \\
             & Yes & 0.000014 & 1.0046 & 0.8242 & 0.2722 & 0.7916 \\
             \hline
            \multirow[c]{2}{*}{$10$} & No & 469.676 & 0.3702 & 0.6249 & 6.6946 & --\\
             & Yes & 0.000010 & 0.3026 & 0.2152 & 0.2350 & 0.8177 \\
             \hline
        \end{tabular}
        \caption{\setlinespacing{1.1} Kernel hyperparameter results for posterior mean matching with inputs in $\Ss^1$.}
        \label{tab:parametricSdHyp}
    \end{table}   
    
    Lastly, Table \ref{tab:parametricSdHyp} shows the trained kernel hyperparameters after the posterior mean matching procedure.  As noted in Section \ref{sec:equivalence}, we see that the Laplace length-scale decreases as NTK depth increases.  The Gaussian length-scale stays approximately the same regardless of depth when the data is non-noisy.  Both the Laplace and Gaussian length-scales decrease as depth increases when the data is noisy which may be explained by the additional smoothness added by the noise term $\sigma^2$.
    
    \section{Synthetic 2D Input Case}\label{sec:2D}
    
    \begin{table}[!b]
        \centering
        \begin{tabular}{|p{0.19\linewidth}|p{0.75\linewidth}|}
            \hline
            \centering Dataset & \centering $f(x_1,x_2)$ \tabularnewline
            \hline
            \centering Ackley \citep{ackley2012connectionist} & \centering\small $-20\exp\left[-\frac{1}{5} \sqrt{\frac{1}{2}(x_1^2 + x_2^2)}\right]- \exp\left[\frac{1}{2}(\cos 2\pi x_1 + \cos 2\pi x_2)\right] + e + 20$ \tabularnewline
            \hline
            \centering Franke \citep{franke1979critical} & \centering\small $0.75\exp\left(-\frac{(9x_1-2)^2}{4} -\frac{(9x_2-2)^2}{4}\right) + 0.75\exp\left(-\frac{(9x_1+1)^2}{49} -\frac{9x_2+1}{10}\right) + 0.5\exp\left(-\frac{(9x_1-7)^2}{4} -\frac{(9x_2-3)^2}{4}\right) - 0.2\exp\left(-(9x_1-4)^2 -(9x_2-7)^2\right)$ \tabularnewline
            \hline
            \centering Nonpoly. \citep{welch1992screening} & \centering\small $\frac{1}{6}\left[(30 + 5x_1\sin(5x_1))(4+\exp(-5x_2)) - 100\right]$ \tabularnewline
            \hline
        \end{tabular}
        \caption{\setlinespacing{1.1} Three 2D input functions with $(x_1,x_2)\in\R^2$.}
        \label{tab:2Dfuncs}
    \end{table}
    \begin{table}[t]
        \centering
        \begin{tabular}{|c|c|c|}
            \hline
            Ackley & Franke & Nonpolynomial \\
            \hline
            $x_1,x_2\sim U[1,7]$ & $x_1,x_2\sim U[-0.5, 1]$ & $x_1,x_2\sim U[0,2]$ \\
            \hline
            $\epsilon\sim \mathcal{N}(0, 0.75^2)$ & $\epsilon\sim \mathcal{N}(0, 0.10^2)$ & $\epsilon\sim \mathcal{N}(0, 1)$ \\
            \hline
        \end{tabular}
        \caption{\setlinespacing{1.1} 2D input function input sampling distributions and noise used for the noisy experiments.}
        \label{tab:2Dfuncparams}
    \end{table}
    
    In this section we look at more complicated functions in the 2D input space.  Since the Laplace kernel and NTK share the same RKHS in $\Sd$, it is of interest to determine how these kernels behave in the space of $\Rd$.  We consider 3 different functions found in literature summarized in Tables \ref{tab:2Dfuncs} and \ref{tab:2Dfuncparams}.
    
    During the calculation of the NTK, the computational space complexity scales exceptionally poorly due to the recursive nature of the kernel which caused issues on our hardware. Thus this required a way to minimize the dataset size while retaining enough information to generalize our model.  For the following experiments, we utilize Latin hypercube sampling \citep{iman1981lhs}. To understand Latin hypercube sampling it is simpler to define Latin square sampling first. \textit{Latin square sampling} is where only one sample is chosen in each row and column of a square grid.  As a result, \textit{Latin hypercube sampling} is a generalization of the Latin square giving a sample that tries to yield approximately equidistant points in given boundaries.  We use it because it better captures the overall surface of these functions during training whilst utilizing fewer training points.
    
    \begin{figure}[!b]
        \centering
        \makebox[\textwidth]{\includegraphics[width=\textwidth]{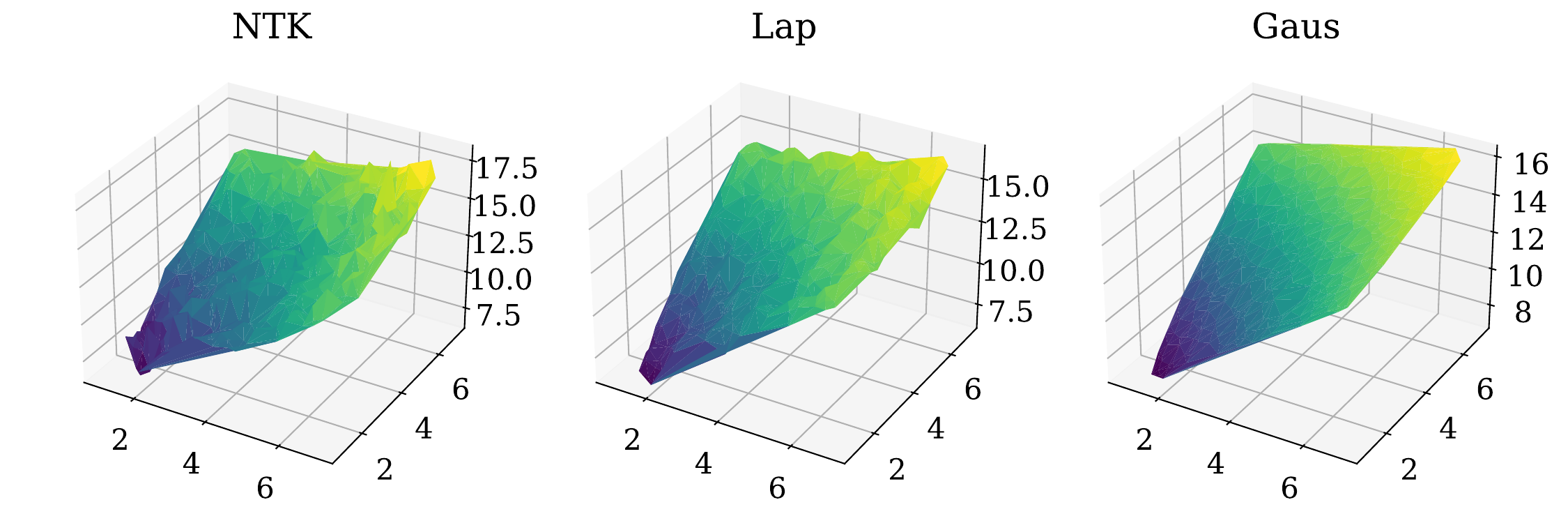}}
        \caption{\setlinespacing{1.1} Posterior means for the noisy the Ackley function in $\R^2$ for NTK depth $D=2$ with test size $m=500$.  GPs were trained using $n=500$ inputs $x_1,x_2 \in [1,7]$ generated using Latin hypercube sampling.}
        \label{fig:2dRdNoisyD2}
    \end{figure}
    
    In these experiments we have 2 groups ($\R^2$ and normalized to $\Ss^1$) each containing a non-noisy and noisy dataset for a total of 4 datasets per function.  For each function we sample a grid of $1000$ points using Latin hypercube sampling.  We split the sample into a training dataset size $n=500$ and a testing dataset size $m=500$.  For all other parameters we adopt the defaults laid out in Table \ref{tab:expparams}.
    
    \begin{table}[!b]
        \centering
        \begin{tabular}{|c|c|c||c|c||c|c||c|c|}
        \cline{4-9}
        \multicolumn{3}{c|}{} & \multicolumn{2}{c||}{$D=2$} & \multicolumn{2}{c||}{$D=3$} & \multicolumn{2}{c|}{$D=10$} \\
        \hline
        Metrics & Dataset & Noise & Lap & Gaus & Lap & Gaus & Lap & Gaus \\
        \hline
        \multirow[c]{6}{*}{RMSE} & \multirow[c]{2}{*}{Ackley} & No & 0.7881 & 0.6695 & 0.7868 & 0.6693 & 0.7790 & 0.6682 \\
        &  & Yes & 0.3296 & 0.3115 & 0.3370 & 0.3206 & 0.3816 & 0.3712 \\
        \cline{2-9} 
        & \multirow[c]{2}{*}{Franke} & No & 0.0947 & 0.0957 & 0.0946 & 0.0968 & 0.0935 & 0.0965 \\
        &  & Yes & 0.0487 & 0.0532 & 0.0496 & 0.0541 & 0.0551 & 0.0595 \\
        \cline{2-9}
        & \multirow[c]{2}{*}{Nonpoly} & No & 2.6526 & 2.5916 & 2.6526 & 2.5916 & 2.6513 & 2.5905 \\
        &  & Yes & 0.8275 & 0.8411 & 0.8307 & 0.8448 & 0.8461 & 0.8640 \\
        \hline
        \hline
        \multirow[c]{6}{*}{$\rho$} & \multirow[c]{2}{*}{Ackley} & No & 0.9407 & 0.9558 & 0.9408 & 0.9559 & 0.9419 & 0.9560 \\
        &  & Yes & 0.9890 & 0.9900 & 0.9884 & 0.9894 & 0.9849 & 0.9857 \\
        \cline{2-9}
        & \multirow[c]{2}{*}{Franke} & No & 0.9330 & 0.9328 & 0.9332 & 0.9310 & 0.9341 & 0.9315 \\
        &  & Yes & 0.9791 & 0.9753 & 0.9781 & 0.9744 & 0.9724 & 0.9688 \\
        \cline{2-9}
        & \multirow[c]{2}{*}{Nonpoly} & No & 0.3760 & 0.4255 & 0.3760 & 0.4255 & 0.3761 & 0.4254 \\
         &  & Yes & 0.8000 & 0.7994 & 0.7985 & 0.7967 & 0.7883 & 0.7809 \\
        \hline
        \end{tabular}
        \caption{\setlinespacing{1.1} Posterior mean matching results for the 2D input surface datasets in $\R^2$ with test size $m=500$.  GPs were trained using $n=500$ inputs $x_1,x_2$ generated using Latin hypercube sampling over the respective function domains.}
        \label{tab:2dRd}
    \end{table}
    
    \begin{figure}[!htp]
        \centering
        \includegraphics[width=\textwidth]{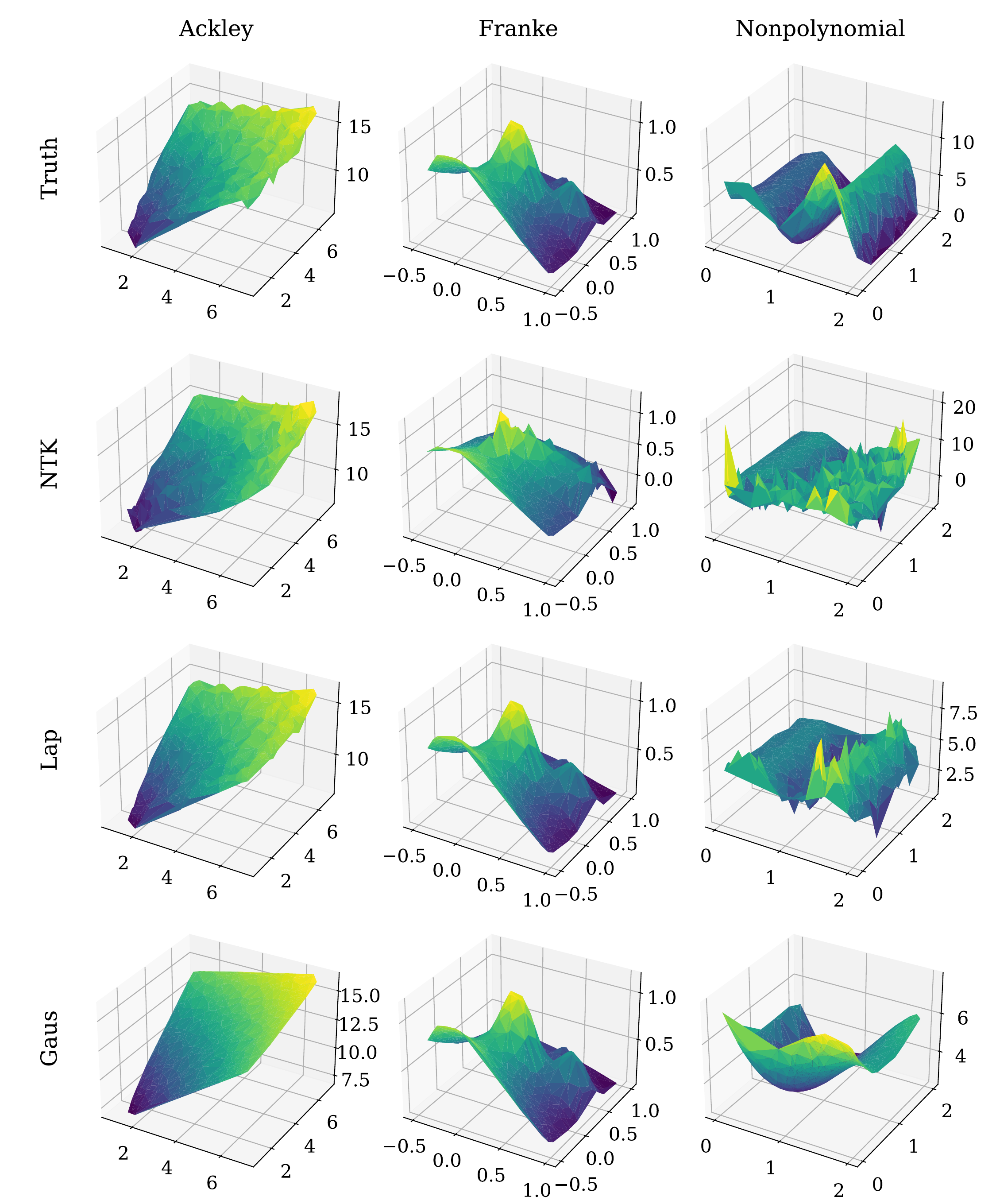}
        \caption{\setlinespacing{1.1} Posterior means of the non-noisy 2D input functions trained in $\R^2$ for NTK depth $D=3$ with test size $m=500$.  GPs were trained using $n=500$ inputs $x_1,x_2$ generated using Latin hypercube sampling over the respective function domains.}
        \label{fig:2dRd}
    \end{figure}
    
    \begin{figure}[!htp]
        \centering
        \includegraphics[width=\textwidth]{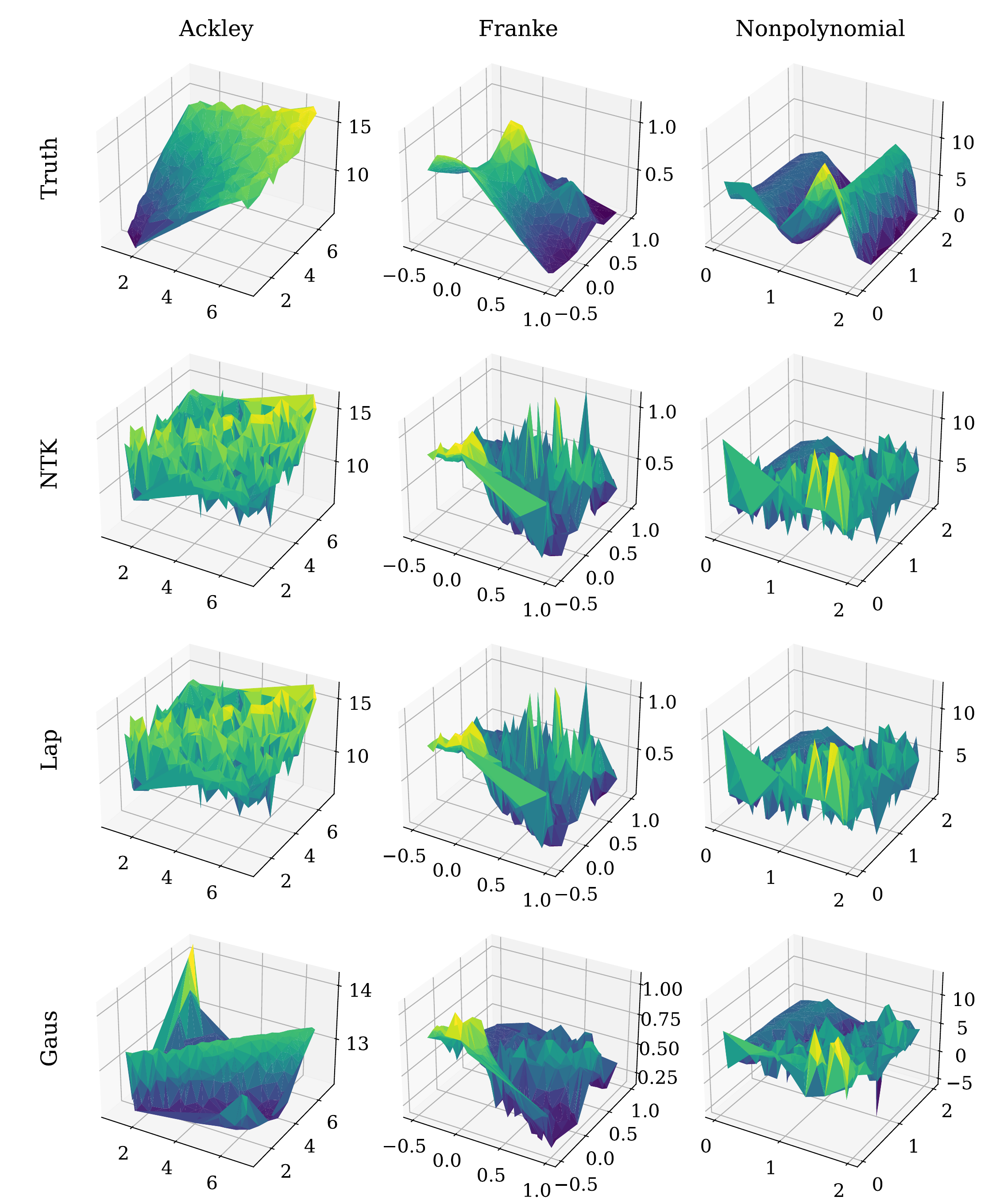}
        \caption{\setlinespacing{1.1} Posterior means of the non-noisy 2D input functions trained in $\Ss^1$ for NTK depth $D=3$ with test size $m=500$.  GPs were trained using $n=500$ inputs $x_1,x_2$ generated using Latin hypercube sampling over the respective function domains.  For visualization $(x_1, x_2)\in\R^2$.}
        \label{fig:2dSd}
    \end{figure}
    
    One reoccurring theme in these experiments is the white noise term $\sigma^2$ performing the role of smoothing out the predictions.  As mentioned in the previous section, taking data noise into account allows for more leeway in matching using Algorithm \ref{alg:length-scalebiasmatching}.  Figure \ref{fig:2dRdNoisyD2} illustrates this via the posterior means of the Ackley function.  The smoothed surface of the NTK is similar to that of a plane and thus closely matched to by the Laplace and Gaussian kernels.  Much like noisy data in $\Ss^1$, this effect is seen in all noisy $\R^2$ experiments where the posterior mean correlation is always higher compared to the corresponding non-noisy experiments.
    
    Looking at the $\R^2$ experiments in Table \ref{tab:2dRd} and Figure \ref{fig:2dRd}, it is apparent from the correlations of the posterior means that the Laplace kernel is not able to match the NTK.  This is the same case for the Gaussian kernel as well.  The nonpolynomial dataset seems to be especially hard to match $(\rho \approx 0.5)$.  The Laplace and Gaussian kernels both provide similar resulting correlations in all configurations in $\R^2$.  Based on these experiments, the conclusion is that the elements of the Laplace kernel's and NTK's RKHSs (e.g.\ the posterior means) do not fully overlap in the general space $\Rd$.  
    
    As mentioned previously, this is in contrast to the space of $\Sd$ where we observe nearly perfect correlation and low RMSE values during the posterior mean matching procedure for the Laplace kernel and NTK.  Details of these experiments are can be found in Table \ref{tab:2dSd} in Appendix \ref{ch:additionalfigstables}.  Figure \ref{fig:2dSd} shows that the Laplace kernel and NTK have effectively the same posterior means whereas the Gaussian kernel's posterior falls short of matching.  The two spaces presented in the figures have very different posterior means. Arguably, the space of $\R^2$ provides results more in-line with the ground truth than the space of $\Ss^1$.  This observation is further explore in the next section.  Lastly, we would like to draw attention to the noisy Ackley function in $\Ss^1$ which was a difficult dataset to fit utilizing the NTK.  We discuss the details of this in Figure \ref{fig:2dSdNoisyD2} in Appendix \ref{ch:additionalfigstables}.

    \section{Synthetic High Dimensional Cases}
    \label{sec:HD}
    
    Knowing that we can reliably match the posterior means of the Laplace kernel and NTK in $\Sd$, we now turn our focus to the quality of predictions while utilizing these kernels.  Specifically, we focus on the NTK's ability to generate meaningful posterior means.  We test this using the multidimensional Friedman datasets \citep{friedman1991multivariate, breiman1996bagging} which are used to benchmark and test linear regression models.  The datasets and their features are summarized in Tables \ref{tab:friedmandata} and \ref{tab:friedmandistribution} respectively.  
    
    \begin{table}[!b]
        \centering
        \begin{tabular}{|p{0.15\linewidth}|p{0.65\linewidth}|p{0.06\linewidth}|}
            \hline
            \centering Dataset & \centering $f(x_1, x_2,\dots, x_d)$ & \centering $d$ \tabularnewline
            \hline
            \centering Friedman 1 & \centering $10\sin(\pi x_1 x_2) + 20(x_3 - 0.5)^2 + 10x^4 + 5x^5$ & \centering 10 \tabularnewline
            \hline
            \centering Friedman 2 & \centering $\left(x_1^2 + \left(x_2 x_3 - (x_2 x_4)^{-2}\right)^2\right)^{1/2}$ & \centering 4 \tabularnewline
            \hline
            \centering Friedman 3 & \centering $\arctan\left(\dfrac{x_2 x_3 - (x_2 x_4)^{-1}}{x_1}\right)$ & \centering 4 \tabularnewline
            \hline
        \end{tabular}
        \caption{\setlinespacing{1.1} Friedman datasets along with their corresponding input dimensions where $(x_1,\dots,x_d)\in\R^d$.  Friedman 1 dataset's output is independent of the last five input variables hence why the dataset has a total of 10 input dimensions.}
        \label{tab:friedmandata}
    \end{table}
    
    \begin{table}[!b]
        \centering
        \begin{tabular}{|c|c|c|}
            \hline
            Friedman 1 & Friedman 2 & Friedman 3 \\
            \hline
            \multirow{4}{*}{$x_1,\dots, x_{10} \sim U[0, 1]$} & \multicolumn{2}{c|}{$x_1\sim U(0,100]$} \\
            & \multicolumn{2}{c|}{$x_2\sim U[40\pi, 560\pi]$} \\
            & \multicolumn{2}{c|}{$x_3\sim U[0, 1]$} \\
            & \multicolumn{2}{c|}{$x_4\sim U[1, 11]$} \\
            \hline
            $\epsilon\sim \mathcal{N}(0, 1.5)$ & $\epsilon\sim \mathcal{N}(0, 5)$ & $\epsilon\sim \mathcal{N}(0, 0.15)$ \\
            \hline
        \end{tabular}
        \caption{\setlinespacing{1.1} Friedman data feature distributions.  Friedman 2 and 3 share the same feature distributions.  The noise term $\epsilon$ is applied only for the noisy data cases.}
        \label{tab:friedmandistribution}
    \end{table}
    
    We maintain the same experiment setup as stated in Table \ref{tab:expparams}; however, we utilize the \textit{coefficient of determination} ($R^2$) as a new metric for this section as a way to assess regression performance:
    \begin{equation}\label{eq:r2}
        R^2 = 1 - \dfrac{\textit{residual sum of squares}}{\textit{total sum of squares}} =  1 - \dfrac{\sum_i (y_{*_i} - f_{*_i})^2}{\sum_i (y_{*_i} - \bar{y}_*)^2},
    \end{equation}
    where $y_{*_i}$ is the ground truth value from the testing set and $f_{*_i}$ is the predicted value.  $R^2$ can attain a maximum value of 1 indicating perfect interpolation, a value of 0 indicating performance equivalent to the mean of the ground truth $\bar{y}_*$, and arbitrary negative values indicating worse performance than $\bar{y}_*$.  Our goal in this section is to test the impact of rescaling inputs and/or outputs while using the NTK in both $\Rd$ and $\Sd$.  Alongside this, we continue to perform posterior mean matching to the NTK for the Mat\'ern kernels. 
    
    For each dataset we generate 200 samples with $n=100$ for the training set and $m=100$ for the testing set.  In the experiments where both normalization and rescaling is used on the input space, we first rescale our inputs and then normalize them to $\Sd$.  It is important to rescale first because doing otherwise does not guarantee that our inputs will lie in $\Sd$.  We also treat rescaling inputs ($X_{rescale}$) and outputs ($\mathbf{y}_{rescale}$) as separate experiments.  Output rescaling is only done during training as outlined in Table \ref{tab:expparams}.
    
    \begin{table}[!ht]
        \centering
        \begin{tabular}{|c|c||c|c|c|c||c|c|c|c|}
        \cline{3-10}
        \multicolumn{2}{c|}{} & \multicolumn{4}{c||}{Non-noisy} & \multicolumn{4}{c|}{Noisy} \\
        \hline
        $D$ & \small Sp. & None & \small $X_{rescale}$ & \small $\mathbf{y}_{rescale}$ & Both & None & \small $X_{rescale}$ & \small $\mathbf{y}_{rescale}$ & Both \\
        \hline
        \multirow[c]{2}{*}{2} & $\Ss^9$ & 0.6549 & 0.7792 & 0.6447 & 0.8054 & 0.5919 & 0.6993 & 0.5843 & 0.7154 \\
        & $\R^{10}$ & 0.7917 & 0.7857 & 0.7924 & 0.7954 & 0.6947 & 0.7000 & 0.7054 & 0.7086 \\
        \hline
        \multirow[c]{2}{*}{3} & $\Ss^9$ & 0.6481 & 0.7672 & 0.6353 & 0.7880 & 0.5924 & 0.6913 & 0.5820 & 0.7054 \\
        & $\R^{10}$ & 0.7860 & 0.7697 & 0.7833 & 0.7788 & 0.6901 & 0.6898 & 0.7010 & 0.7007 \\
        \hline
        \multirow[c]{2}{*}{10} & $\Ss^9$ & 0.5956 & 0.6302 & 0.5785 & 0.6565 & 0.5538 & 0.5781 & 0.5406 & 0.5992 \\
        & $\R^{10}$ & 0.7569 & 0.6301 & 0.7069 & 0.6595 & 0.6711 & 0.5727 & 0.6459 & 0.6048 \\
        \hline
        \end{tabular}
        \caption{\setlinespacing{1.1} Friedman 1 $R^2$ results for NTK posterior means with training done using various data transformations.  The None column is the baseline with no input or output rescaling, $X_{rescale}$ column is with input rescaling only, $\mathbf{y}_{rescale}$ column is with output rescaling during training only, and the Both column applies both types of rescaling.}
        \label{tab:friedman1NTKr2}
    \end{table}
    
    \begin{table}[!hb]
        \centering
        \begin{tabular}{|c||c|c|c|c||c|c|c|c|}
        \cline{2-9}
        \multicolumn{1}{c|}{} & \multicolumn{4}{c||}{Non-noisy} & \multicolumn{4}{c|}{Noisy} \\
        \hline
        $D$ & None & \small $X_{rescale}$ & \small $\mathbf{y}_{rescale}$ & Both & None & \small $X_{rescale}$ & \small $\mathbf{y}_{rescale}$ & Both \\
        \hline
        2 & 0.9974 & $\approx 1$ & 0.9981 & 0.9961 & 0.9985 & 0.9991 & 0.9995 & 0.9951 \\
        3 & 0.9983 & $\approx 1$ & 0.9991 & 0.9986 & 0.9988 & 0.9996 & 0.9992 & 0.9983 \\
        10 & 0.9985 & 0.9989 & 0.9992 & 0.9999 & 0.9984 & 0.9990 & 0.9993 & 0.9999 \\
        \hline
        \end{tabular}
        \caption{\setlinespacing{1.1} Friedman 1 $\rho$ results for Laplace kernel and NTK posterior mean matching in $\Ss^9$ with training done using various data transformations.}
        \label{tab:friedman1NTKcorr}
    \end{table}
    
    Our experiment results show that rescaling the inputs has a significant positive impact on regression results for the NTK.  This is best illustrated in Table \ref{tab:friedman1NTKr2} where we consistently see that input rescaling increases our $R^2$ value significantly.  On the other hand output rescaling does not have a significant impact on the quality of the regression as indicated by the None and $\mathbf{y}_{rescale}$ columns.  That said, output rescaling helps with hyperparameter optimization.  Without it we hit upper bounds during optimization for hyperparameters like constant value or NTK's bias.  Lastly, applying both input and output rescaling yields the best regression results.  The data transformations do not greatly impact Laplace kernel and NTK posterior mean matching as seen in Table \ref{tab:friedman1NTKcorr}.  We find similar results for the Friedman 2 and 3 datasets found in Tables \ref{tab:friedman2NTKr2}, \ref{tab:friedman3NTKr2}, \ref{tab:friedman2NTKcorr}, and \ref{tab:friedman3NTKcorr} in Appendix \ref{ch:additionalfigstables}.
    
    \begin{figure}[!hp]
    \vspace{-2em}
        \centering
        \includegraphics[width=0.9\textwidth]{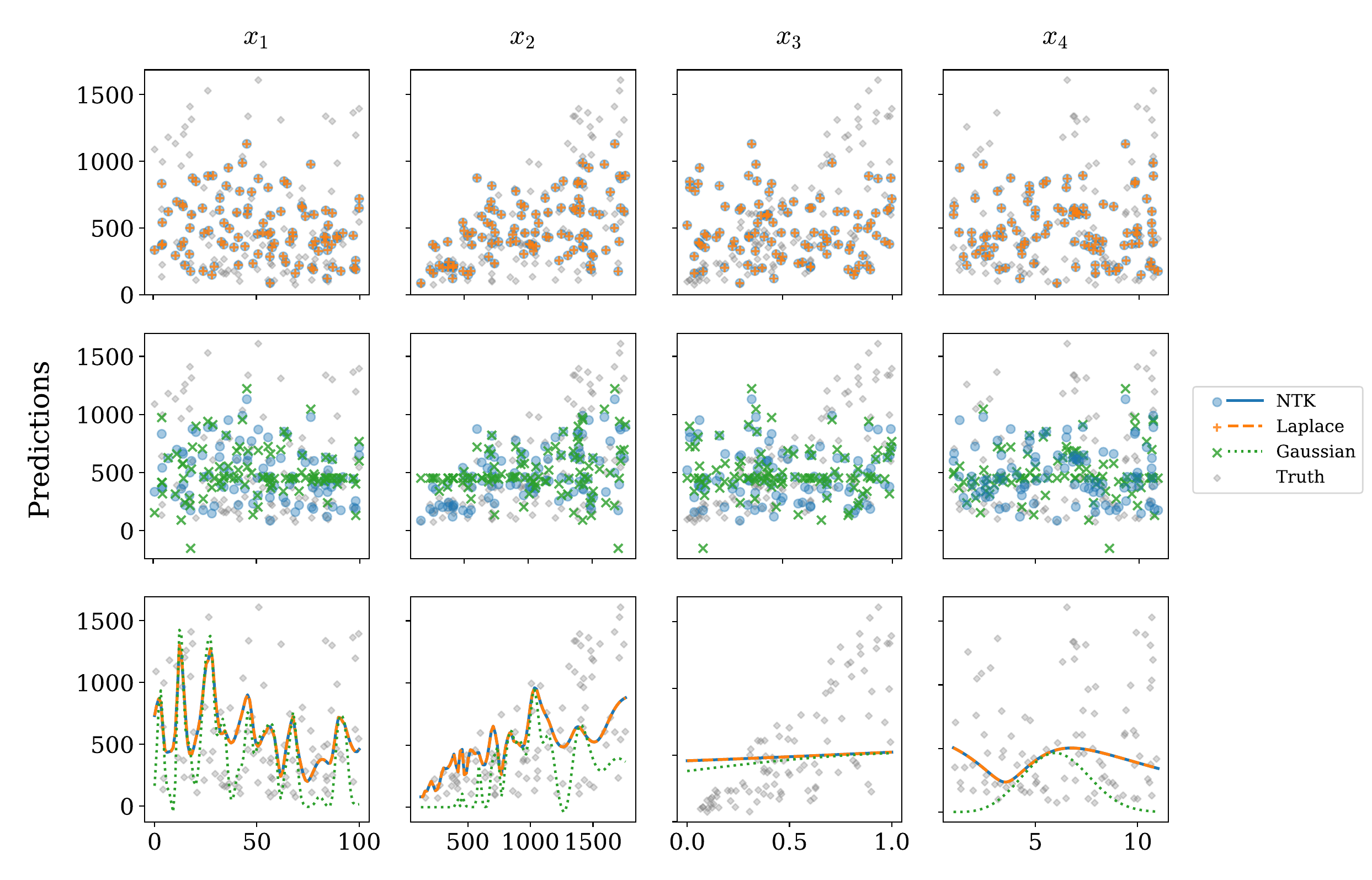}
        \caption{\setlinespacing{1.1} Predictions for non-noisy Friedman 2 in $\Ss^3$ for $D=2$ with $\mathbf{y}_{rescale}$. \textit{Top:} NTK and Laplace predictions overlayed. \textit{Middle:} NTK and Gaussian predictions overlayed. \textit{Bottom:} Averaged prediction plots of all kernels.}
        \label{fig:friedman2Sdy}
        
        \vspace{0.75em}
        
        \includegraphics[width=0.9\textwidth]{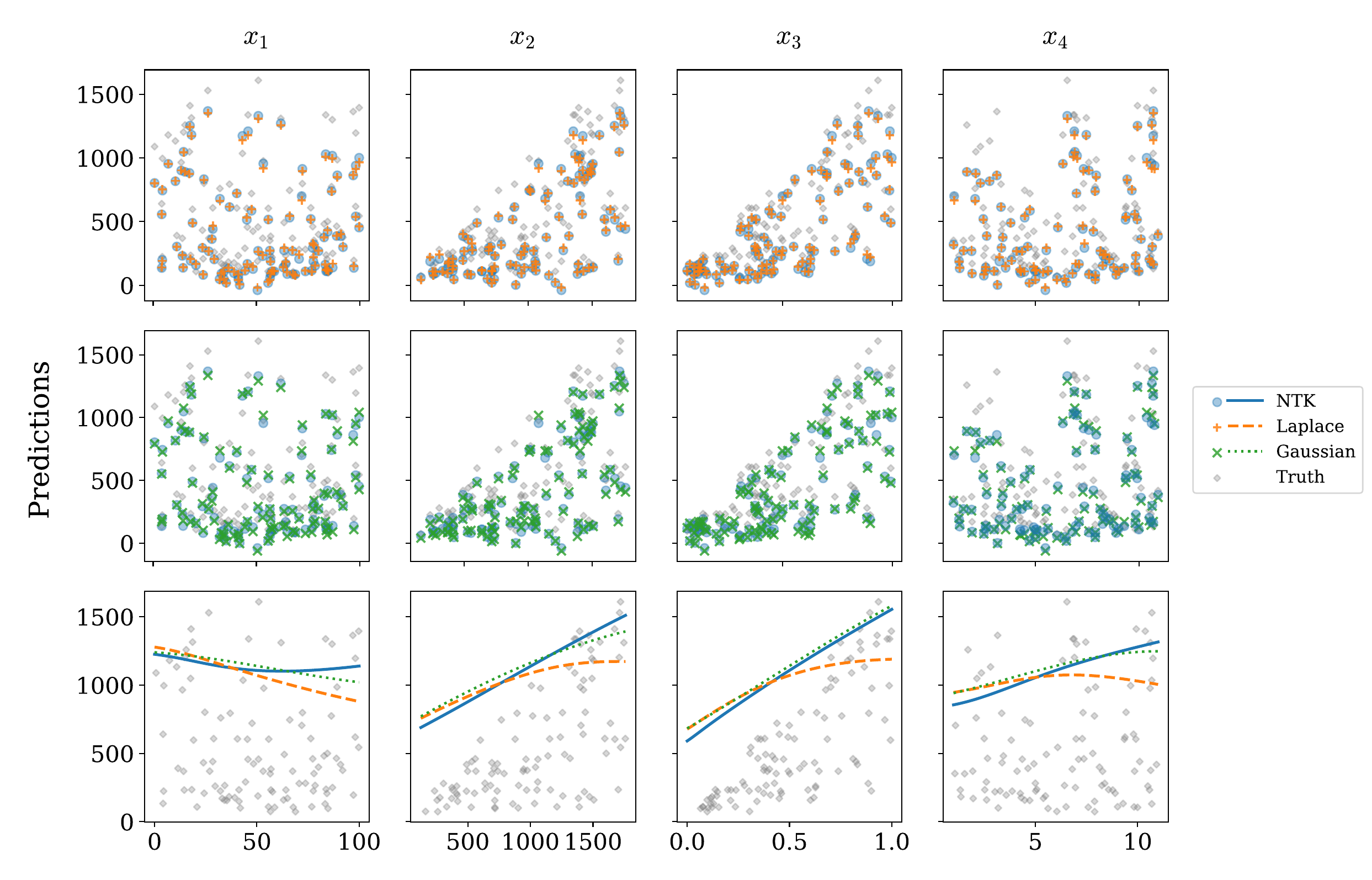}
        \caption{\setlinespacing{1.1} Predictions for non-noisy Friedman 2 in $\Ss^3$ for $D=2$ with $X_{rescale}$ and $\mathbf{y}_{rescale}$. The rows of the figure are laid out as in Figure \ref{fig:friedman2Sdy}.}
        \label{fig:friedman2SdXy}
    \end{figure}
    
    We move on to visualizations of the different data transformations.  In the remainder of the figures in this section, we include a bottom row that contains averaged prediction plots of all the kernels.  These plots are created by using a trained GP to predict over a dataset where we uniformly sample from one input dimension while the remaining inputs are set as the mean of their respective distribution (Table \ref{tab:friedmandistribution}).  As a result, we can visualize the sampled dimension as a 2D plot (e.g.\ $x_1$ vs $f_*$).
    
    We begin by differentiating the posterior results for scaled and unscaled inputs.  For this we turn to the non-noisy Friedman 2 dataset in $\Ss^3$ with outputs rescaled and NTK $D=2$.  Figures \ref{fig:friedman2Sdy} and \ref{fig:friedman2SdXy} showcase unscaled and scaled inputs respectively.  In the unscaled input case in Figure \ref{fig:friedman2Sdy}, the GP poorly generalizes for all inputs aside from $x_2$ ($R^2 \approx -0.071$, Table \ref{tab:friedman2NTKr2}).  On the other hand, the scaled inputs in Figure \ref{fig:friedman2SdXy} do a much better job in producing a good regression fit over the test data ($R^2 \approx 0.855$, Table \ref{tab:friedman2NTKr2}).  These results indicate that input rescaling seems to improve regression predictions and does not impact our ability to match posterior means.  We include results for the noisy Friedman 2 dataset with the same setup in Figures \ref{fig:friedman2noisySdy} and \ref{fig:friedman2noisySdXy} in Appendix \ref{ch:additionalfigstables}.
    
    \begin{figure}[htp]
        \vspace{-2em}
        \centering
        \includegraphics[width=0.9\textwidth]{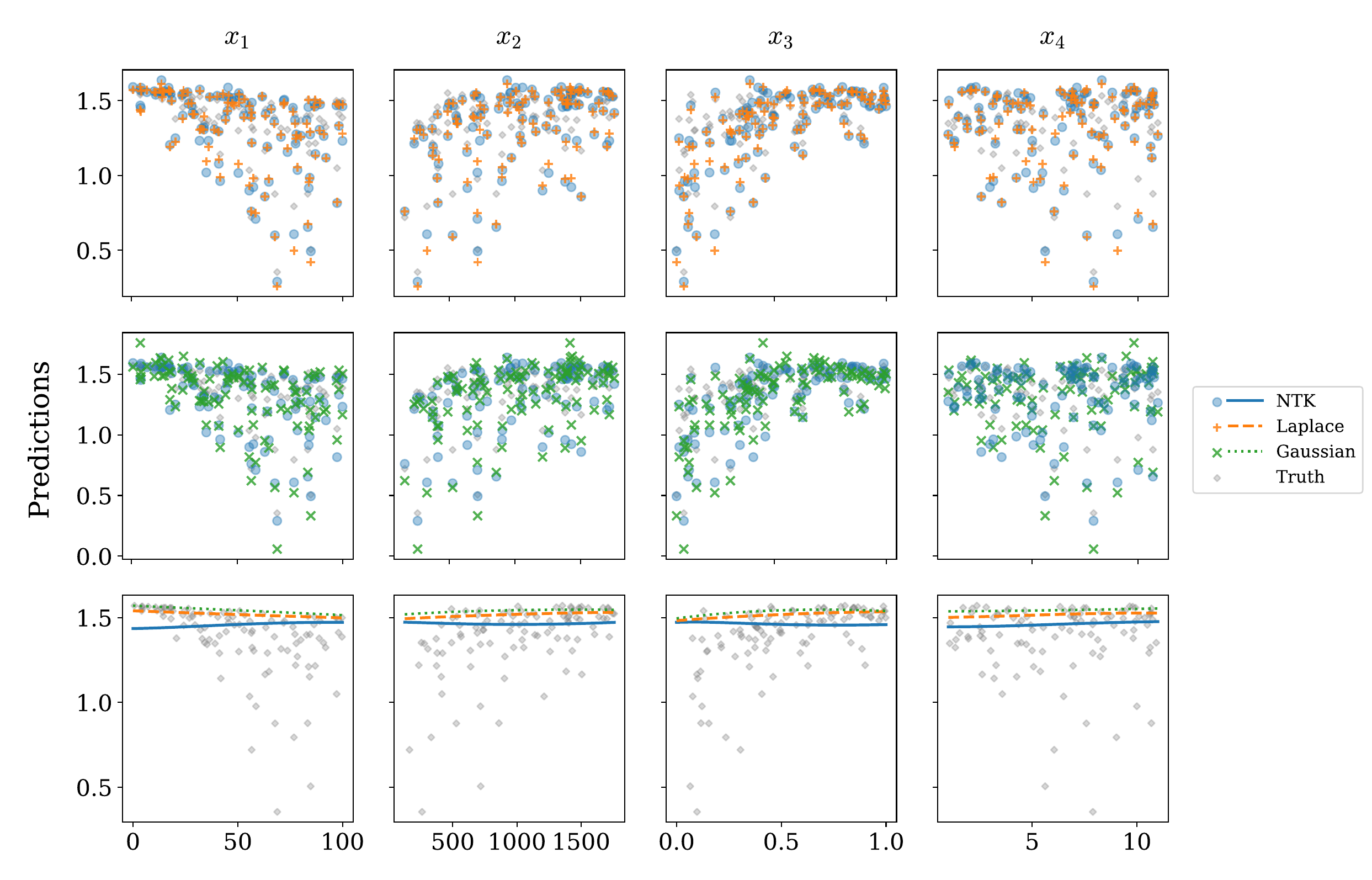}
        \caption{\setlinespacing{1.1} Predictions for non-noisy Friedman 3 in $\R^4$ for $D=2$ with $X_{rescale}$ and $\mathbf{y}_{rescale}$. The rows of the figure are laid out as in Figure \ref{fig:friedman2Sdy}.}
        \label{fig:friedman3RdBoth}
        
        \vspace{0.75em}
        
        \includegraphics[width=0.9\textwidth]{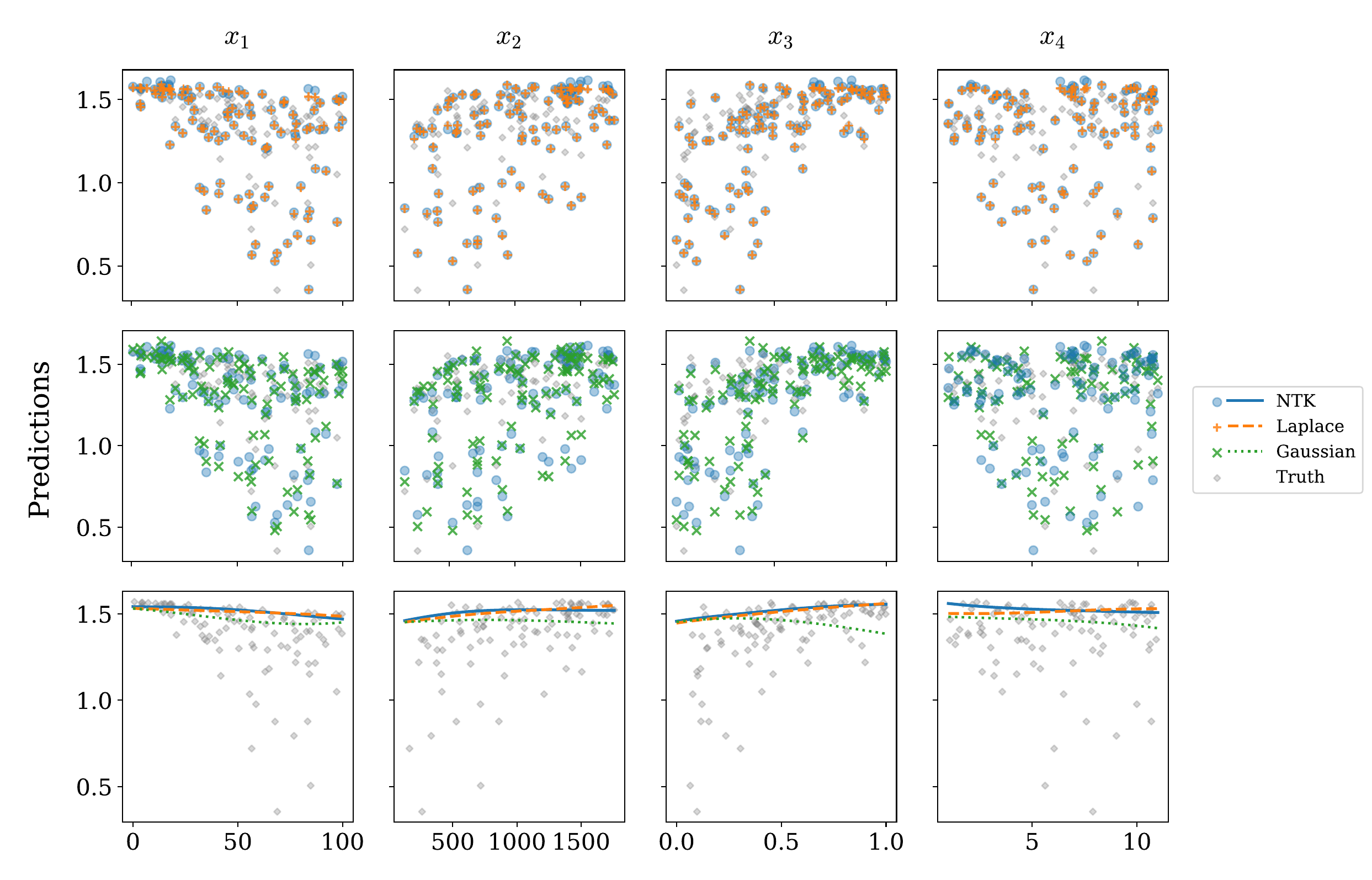}
        \caption{\setlinespacing{1.1} Predictions for non-noisy Friedman 3 in $\Ss^3$ for $D=2$ with $X_{rescale}$ and $\mathbf{y}_{rescale}$. The rows of the figure are laid out as in Figure \ref{fig:friedman2Sdy}.}
        \label{fig:friedman3SdBoth}
    \end{figure}
    
    We now backtrack a bit to address the high $R^2$ values in $\Rd$ as opposed to $\Sd$.  This is almost always the case regardless of dataset or configuration as seen in Tables \ref{tab:friedman1NTKr2} (Friedman 1), \ref{tab:friedman2NTKr2} (Friedman 2), and \ref{tab:friedman3NTKr2} (Friedman 3).  We provide a visualization of this phenomenon for the non-noisy Friedman 3 dataset with both inputs and outputs rescaled and NTK $D=2$.  Figures \ref{fig:friedman3RdBoth} and \ref{fig:friedman3SdBoth} on the following page showcase the posterior means of $\R^4$ ($R^2 \approx 0.692$) and $\Ss^3$ ($R^2 \approx 0.033$) respectively.  The low $R^2$ of $\Ss^3$ is attributed close values between the residual sum of squares and the total sum of squares in Equation \eqref{eq:r2}.  The differences of the predictions appear to be very minor between the spaces.  Looking at the first row of both figures it appears that the inputs in $\R^4$ have a slightly tighter fit over the test data.  This can best be seen on the $x_3$ input in both rows.  What is of interest is that the averaged prediction plots (row 3) for both figures appear to be very similar.  That said, this is in-line with the plots seen in Section \ref{sec:2D} where the posteriors trained in $\Rd$ were more related to the ground truth compared to posteriors trained in $\Sd$.  We include Figures \ref{fig:friedman3noisyRdBoth} and \ref{fig:friedman3noisySdBoth} in Appendix \ref{ch:additionalfigstables} which showcases the same setup for noisy Friedman 3.
    
    \chapter{Real World Experiments}\label{ch:realexp}
    
    In this chapter we use the lessons we learned from Chapter \ref{ch:synthexp} to evaluate the individual regression performance of the Laplace kernel, Gaussian kernel, and NTK individually on two real world datasets in a setting where the models are allowed to train to their own accord.  In Section \ref{sec:realworldsetup} we outline the experiment setup and datasets and in Section \ref{sec:realworldresults} we present our regression results.
    
    \section{Setup and Datasets}\label{sec:realworldsetup}
    
    We setup our experiments with the findings in Chapter \ref{ch:synthexp} in mind for the best possible GP fit. We summarize the experimental treatment in Table \ref{tab:expparamsreal}.  The key differences are that we perform input rescaling ($X_{rescale}$) in all experiments.  In addition, we split our data 75-25 into a training set $(X, \mathbf{y})$ and a testing set $(X_*, \mathbf{y}_*)$ respectively.  Our experiments test inputs in $\Rd$ or $\Sd$ resulting in only 2 configurations per dataset.  Since we are working with real world data, we assume that noise is present and apply the white noise kernel to our three kernels.  We use $R^2$ and RMSE as our regression metrics.
    
    \begin{table}[p]
        \centering
        \begin{tabular}{|c|c|c|}
            \hline
            Data & Description & Value \\
            \hline
            Normalization & Transforming inputs to $\Sd$ & -- \\
            $X_{rescale}$ & Input variable rescaling & \texttt{true} \\
            \hline
            \hline
            Optimization & Description & Value\\
            \hline
            $n_{restart}$ & Number of optimizer restarts & $9$ \\
            $\mathbf{y}_{rescale}$ & Response variable rescaling during training & \texttt{true} \\
            $\alpha$ & Value to ensure positive definiteness & $10^{-5}$ \\
            \hline
            \hline
            Kernel & Description & Value \\
            \hline
            $D$   & NTK depth & $2, 3, 10$ \\
            $\beta$ & NTK bias & Unfixed \\
            $\nu$ & Mat\'ern kernel smoothness & $\dfrac{1}{2}, \infty$ \\
            $\ell_{Lap}, \ell_{Gaus}$ & Length-scale & Unfixed \\
            $c_{NTK}, c_{Lap}, c_{Gaus}$ & Constant value & Unfixed \\
            $\sigma^2_{NTK}, \sigma^2_{Lap}, \sigma^2_{Gaus}$ & Data noise variance & Unfixed \\
            \hline
        \end{tabular}
        \caption{\setlinespacing{1.1} Summary of variables managed in all real world experiments.}
        \label{tab:expparamsreal}
    \end{table}
    \begin{algorithm}[p]
        \DontPrintSemicolon
        \SetKwInOut{Define}{Define}
        \SetKwInOut{Input}{Input}
        \SetKwInOut{Return}{Return}
        
        \Define{$n_{restart}$, $\mathbf{y}_{rescale}$, $\alpha$, $noise$}
        \Input{$\theta, X, \mathbf{y}, X_*$ (Train and test data pre-processed the same way)}

        $k(\theta) \leftarrow c\cdot k(\theta) + \sigma^2 \delta_{\mathbf{s} = \mathbf{t}}$\;
        $f_{k_{opt}}, \theta_{opt} \leftarrow$ \textsc{optimize}$\left\{ f_k \sim \mathcal{GP}\right(0, k(\theta)\left)|_{X, \mathbf{y}} : n_{restart}, \mathbf{y}_{rescale}, \alpha \right\}$\;
        $\mathbf{\bar{f}}_{k_*} \leftarrow f_{k_{opt}}(X_*)$\;

        \Return{$\mathbf{\bar{f}}_{k_*}$}
        
        \caption{\setlinespacing{1.1} Experiment procedure for fitting a GP to real world data. \textsc{optimize} refers to the GP optimization process of maximizing the marginal log likelihood.}
        \label{alg:realworld}
    \end{algorithm}
    
    Furthermore, we simplify the training procedure as we are not interested in posterior mean matching.  The three kernels considered follow the same GP fitting and predicting procedure outlined in Equation \eqref{eq:posteriordist} from Section \ref{sec:gp}.  We summarize the pseudocode in Algorithm \ref{alg:realworld}.
    
    We utilize two real world datasets: the concrete compressive strength dataset \citep{yeh1998modeling} and the forest fire dataset \citep{cortez2007data}.  We provide a brief summary of the data in Table \ref{tab:realworlddata}.  The Concrete dataset contains 1030 observations and 8 features which include 7 features describing the concrete's ingredients (kg/m$^3$) and 1 feature describing its age (days).  The Fire dataset contains 517 observations and a total of 12 features.  We drop all the features except for the 4 features describing the weather conditions (temperature (\textcelsius), relative humidity (\%), wind (km/h), and rain (mm/m$^2$)).  This is because \citet{cortez2007data} found that these features provided the best regression performance given their metrics and we found that additional features would result in kernels that would zero out resulting in meaningless posteriors.  It should be noted that we also take the advice of the authors of the Fire dataset and transform the response variable ``area'' using a log transform since it is heavily right-skewed:  $y_{area} := \log(y + 1)$.
    
    \begin{table}[t]
        \centering
        \begin{tabular}{|c|p{0.6\textwidth}|c|c|}
            \hline
            Data Name & \centering Task & $d$ & $n$ \tabularnewline
            \hline
            Concrete & \centering Determine concrete compressive strength & 8 & 1030 \tabularnewline
            \hline
            Fire & \centering Determine area of forest burned during fire & 4 & 517 \tabularnewline
            \hline
        \end{tabular}
        \caption{\setlinespacing{1.1} Summary of real world data.}
        \label{tab:realworlddata}
    \end{table}
    
    \section{Results}\label{sec:realworldresults}
    
    We find through our experiments that the NTK attains comparable $R^2$ values to the Laplace and Gaussian kernels for the Concrete dataset as seen in Table \ref{tab:realworldresults}.  In fact, we attain values of $R^2 \approx 0.9$ indicating a very close fit to the ground truth.  There is only a minute difference between experiments done in $\R^8$ versus $\Ss^7$.  We illustrate the concrete compression predictions in $\Ss^7$ in Figure \ref{fig:concreteNormTNoiseT}.  We note that the Laplace and Gaussian GP predictions seem to align closely to the NTK despite being trained separately. 
    
    \begin{table}[b!]
        \centering
        \begin{tabular}{|c|c|c||c|c|c||c|c|}
        \cline{4-8}
        \multicolumn{3}{c|}{} & \multicolumn{3}{c|}{NTK} & \multirow[c]{2}{*}{Lap} & \multirow[c]{2}{*}{Gaus} \\
        \cline{1-6}
        Data & Metric & Space & $D=2$ & $D=3$ & $D=10$ & & \\
        \hline
        \hline
        \multirow[c]{4}{*}{\small Concrete} & \multirow[c]{2}{*}{RMSE} & $\R^8$ & 4.9562 & 5.0614 & 5.6477 & 5.3210 & 5.3058 \\
        &  & $\Ss^7$ & 5.3571 & 5.2906 & 5.5710 & 5.4368 & 5.1869 \\
        \cline{2-8}
        & \multirow[c]{2}{*}{$R^2$} & $\R^8$ & 0.9041 & 0.9000 & 0.8754 & 0.8894 & 0.8901 \\
        &  & $\Ss^7$ & 0.8879 & 0.8907 & 0.8788 & 0.8846 & 0.8949 \\
        \hline
        \hline
        \multirow[c]{4}{*}{Fire} & \multirow[c]{2}{*}{RMSE} & $\R^4$ & 78.302 & 78.278 & 78.262 & 78.415 & 78.449 \\
        &  & $\Ss^3$ & 78.302 & 78.277 & 78.262 & 78.456 & 78.462 \\
        \cline{2-8}
        & \multirow[c]{2}{*}{$R^2$} & $\R^4$ & \small $-10594$ & \small $-8415$ & \small $-4154$ & \small $-50077$ & \small $-116771$ \\
        &  & $\Ss^3$ & \small $-11031$ & \small $-8539$ & \small $-4154$ & \small $-499246$ & \small $-1254525$ \\
        \hline
        \end{tabular}
        \caption{\setlinespacing{1.1} Results for real world experiments in $\Rd$ and $\Sd$. Metrics for the Fire dataset are calculated by first inverting the log-transformation.}
        \label{tab:realworldresults}
    \end{table}

    \begin{figure}[htp]
        \centering
        \includegraphics[width=\textwidth]{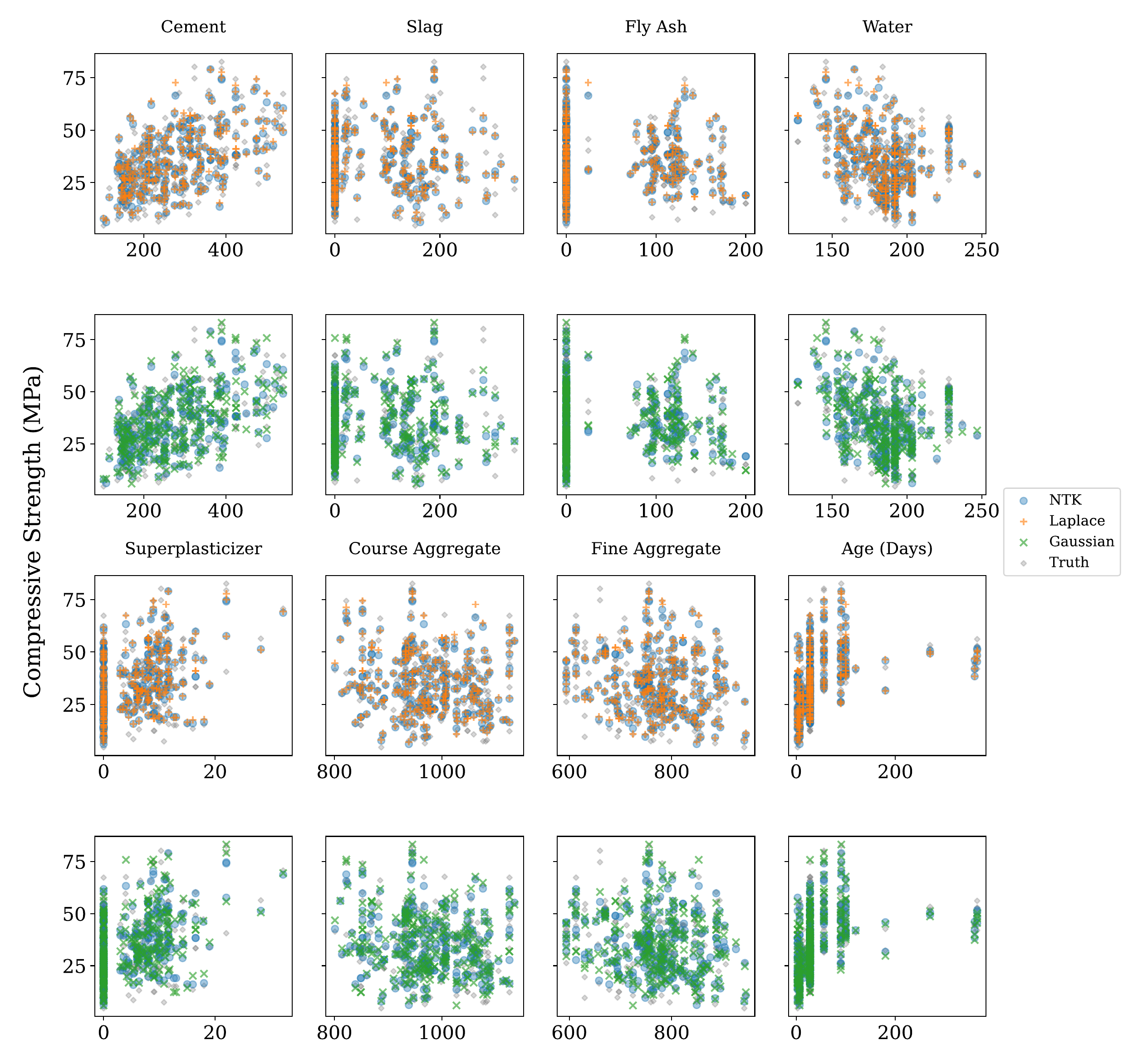}
        \caption{\setlinespacing{1.1} Concrete compression strength predictions over $\Ss^7$ with NTK depth $D=10$ overlayed. Inputs are shown in $\R^8$ for visualization.  The first and last two rows correspond to the input features all in kg/m$^3$ except for the Age.}
        \label{fig:concreteNormTNoiseT}
    \end{figure}
    
    On the other hand, the Fire dataset fails to produce predictions with $R^2$ values that are better than the baseline mean. In fact, the $R^2$ values are significantly worse than the baseline but do show an improvement with higher NTK depth as seen in Table \ref{tab:realworldresults}.  In Figure \ref{fig:firesNormTNoiseT} we can see the area predictions for the NTK attempt to approximate the ground truth but does not capture the general landscape of the data.  Upon further investigation this may be due to the kernel hyperparameter $\sigma^2$ overestimating the noise.  We include Figure \ref{fig:firesNormTNoiseF} of the predictions without the white noise kernel applied in Appendix \ref{ch:additionalfigstables}.  It is interesting to note that the Laplace and Gaussian predictions become constant.  This is because the constant value completely zeros out for both kernels while length-scale becomes large.  Lastly, we turn to the RMSE values for the Fire dataset.  According to \citet{cortez2007data}, their models attained RMSE values of approximately 64.7 whereas we get values close to 78 indicating a worse fit.  Overall, this dataset presented a significantly difficult regression task.
    \vfill
    \begin{figure}[hb!]
        \centering
        \includegraphics[width=\textwidth]{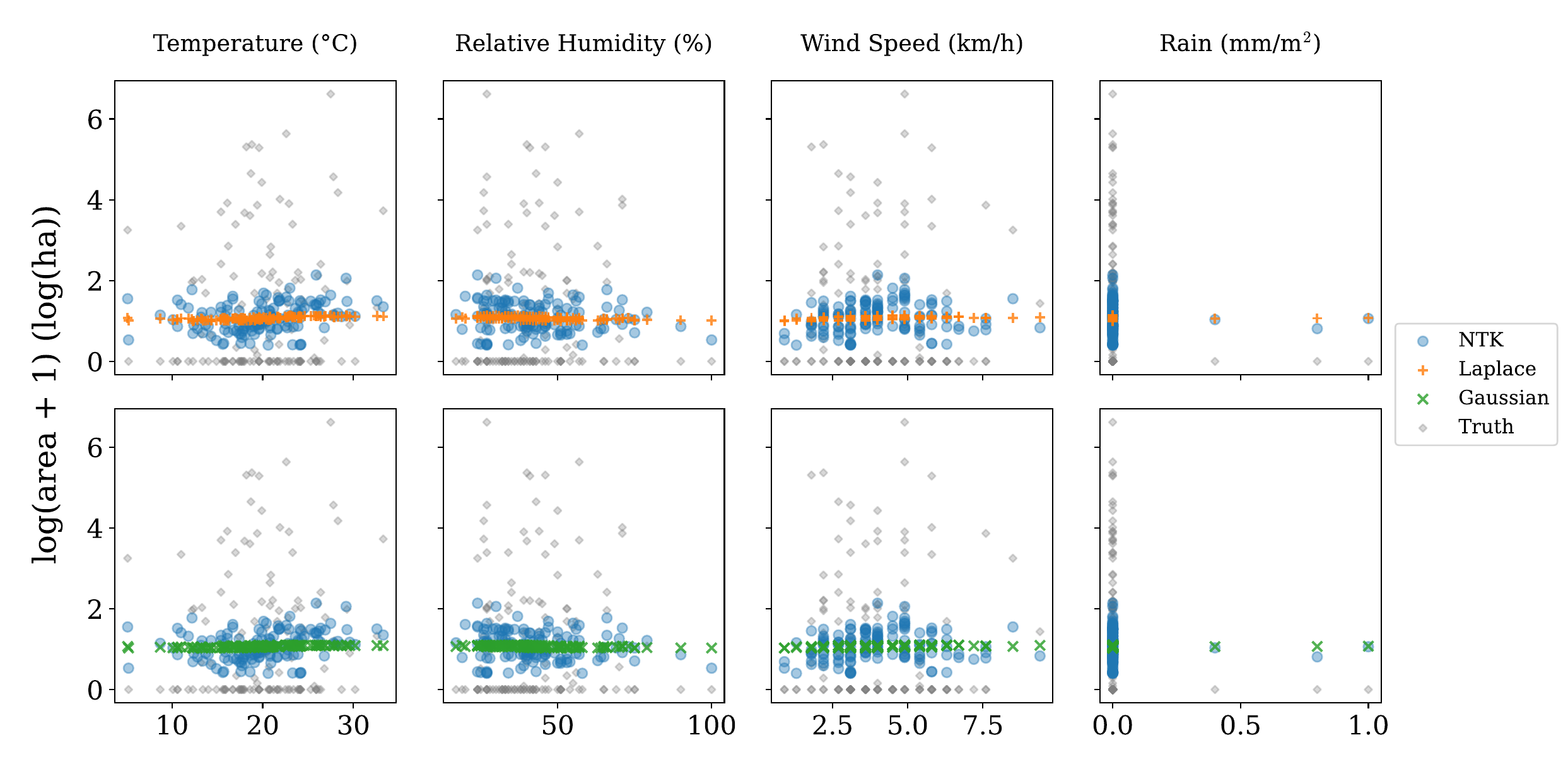}
        \caption{\setlinespacing{1.1} Fire area predictions over $\Ss^3$ with NTK depth $D=10$ overlayed.  Inputs are shown in $\R^4$ and output is log-transformed for visualization.}
        \label{fig:firesNormTNoiseT}
    \end{figure}
    \vfill

    \chapter{Conclusion}
    
    In this work, we explored the empirical connections between the Laplace kernel and NTK, most notably, the direct relationship between the NTK's parameterization and how it can be reconciled with the parameterization of the Laplace kernel.  We showed evidence for the importance of the bias parameter in equating the NTK to the Laplace kernel.  Further, we developed a way to to reliably match the posterior means (elements of an RKHS) generated by the Laplace kernel and NTK GPs trained on various datasets.  We further solidified the evidence that $\Rd$ is not a space in which the two kernels can be equal.  Finally, we showcased that although the Gaussian kernel shares some similarities to the Laplace kernel, it fails to have the flexibility to match up with the other two kernels.
    
    Though consequential, the NTK is a limited tool for machine learning tasks.  As layers increase, it requires a high overhead for computation and runs into numerical issues during hyperparameter optimization due to its recursive nature.  With the connections established to the Laplace kernel by \citet{geifman2020rkhs} and \citet{chen2021rkhs}, the NTK can be more easily applied to practical use and further motivate additional work in kernel methods.
    
    Future work to consider in this topic would be to fully develop the equality between two RKHSs.  As discussed in Section \ref{sec:inclusion}, parameters of two kernels limit the functions a shared RKHS can include.  As such, it would be interesting to explore how kernel parameters affect the resulting RKHS.  The Laplace kernel and NTK are especially of interest since they have entirely different parameterizations yet still share the same space.  In addition, it would be of interest to analyze the discretized nature of the NTK to see if it is possible to develop a direct parameter relation to the Laplace kernel other than pure optimization.  Lastly, it is mathematically interesting to do a full treatment of the asymptotics of the NTK's bias parameter (Appendix \ref{ch:asymtotics}).
    
    At the most essential level, RKHSs provide a robust theoretical framework for practical data modeling.  For a full treatment of this topic we highly recommend the text by \citet{berlinet2011reproducing}.  For machine learning with Gaussian processes \citet{rasmussen2006gpml} give an in-depth overview of the topic.  Lastly, \citet{Goodfellow2016deep} serves as an excellent introduction to modern deep learning.

    \SuppChap
    
    \newpage
    
    \addcontentsline{toc}{chapter}{Bibliography}
    \bibliographystyle{plainnat}
    \bibliography{refs.bib}
    
    \clearpage
    
    \AddAppendix
    
    \begin{appendices}
    
        \chapter{Neural Tangent Kernel Gradient with respect to \texorpdfstring{\boldmath$\beta$}{β}}\label{betagrad}
        
        We derive the gradient used in optimizing the normalized recursive NTK's $\beta$ bias parameter during GP training.  We begin with the following proposition:
        \begin{prop}
            Let $\beta \geq 0$ and define the normalized recursive NTK using Definition \ref{def:recurNTK} and Equation \eqref{eq:normNTK}: 
            \newline
            \begin{equation}\label{normrecur}
                \ddot{k}_{NTK}(L+1, \beta) = \dfrac{1}{(L+1)(\beta^2 + 1)} k_{NTK}(L+1, \beta) = \dfrac{1}{(L+1)(\beta^2 + 1)} \Theta^{(L)}
            \end{equation}
            \newline
            \noindent Then, the partial derivative with respect to $\beta$ of $\ddot{k}_{NTK}(L+1, \beta)$ is defined as
            \newline
            \begin{equation}\label{partialntk}
                \dfrac{\partial}{\partial\beta}\ddot{k}_{NTK}(L+1, \beta) = \dfrac{1}{(L+1)(\beta^2 + 1)}\left(\dfrac{\partial\Theta^{(L)}}{\partial\beta} - \dfrac{2\beta}{\beta^2 + 1}\Theta^{(L)}\right) \\
            \end{equation}
            where
            \begin{equation}
                \dfrac{\partial\Theta^{(L)}}{\partial\beta} = 2\beta \left(\prod^L_{i=1} \Sigma^{(i)}\right) \left(1 + \sum^L_{i=1} \left(\prod_{j=1}^i\dot{\Sigma}^{(j)}\right)^{-1}\right)
            \end{equation}
        \end{prop}
        \vspace{1em}
        Proving this Equation \ref{partialntk} requires the use of the product rule $(uv)' = u'v + uv'$ where we will let $u = \frac{1}{(L+1)(\beta^2 + 1)}$ and $v = \Theta^{(L)}$ from Equation \ref{normrecur}.  We will start with the substantially less involved task of deriving $u$ followed by deriving $v$.
        \begin{proof}
            We begin with $\beta \geq 0$ and Equation \ref{normrecur} and we set up the partial derivative with respect to $\beta$ as follows:
            \begin{equation*}
                \dfrac{\partial}{\partial\beta}\ddot{k}_{NTK}(L+1, \beta) = \dfrac{\partial}{\partial\beta}\left(\dfrac{1}{(L+1)(\beta^2 + 1)} \Theta^{(L)} \right)
            \end{equation*}
            We then let 
            \begin{equation*}
                \begin{aligned}
                    u &= \dfrac{1}{(L+1)(\beta^2 + 1)} \\
                    v^{(L)} &= \Theta^{(L)}
                \end{aligned}
            \end{equation*}
            so that we can utilize the product rule $(uv^{(L)})' = u'v^{(L)} + uv^{(L)'}$ in order to solve.  We begin with finding the derivative of $u$:
            \begin{align*}
                \dfrac{d u}{d\beta} &= \dfrac{d}{d\beta}\left(\dfrac{1}{(L+1)(\beta^2 + 1)}\right)\\
                &= \dfrac{d}{d\beta}\left((L+1)(\beta^2 + 1)\right)^{-1}\\
                &= -\left((L+1)(\beta^2+1)\right)^{-2}\left(2\beta(L+1)\right) \\
                &= - \dfrac{2\beta(L+1)}{(L+1)^{2}(\beta^2+1)^{2}}\\
                \dfrac{d u}{d\beta} &= -\dfrac{2\beta}{(L+1)(\beta^2+1)^{2}}
            \end{align*}
            Which completes the derivative of $u$. Now we continue onto the partial derivative with respect to $\beta$ for $v$ which is defined recursively.  As such, we begin with the base case and build our way towards the general case using the recursive formula from Definition \ref{def:recurNTK}:
            \begin{equation*}
                \dfrac{\partial v^{(0)}}{\partial \beta} = \dfrac{\partial \Theta^{(0)}}{\partial \beta} = \dfrac{\partial}{\partial\beta}\left(\Sigma^{(0)} + \beta^2\right) = 2\beta
            \end{equation*}
            \begin{equation*}
                \begin{aligned}
                    \dfrac{\partial v^{(1)}}{\partial \beta} &= 
                    \begin{aligned}[t]
                        \dfrac{\partial \Theta^{(1)}}{\partial \beta} &= \dfrac{\partial}{\partial \beta}\left(\Theta^{(0)}\dot{\Sigma}^{(1)} + \Sigma^{(1)} + \beta^2\right)\\
                        &= \dfrac{\partial\Theta^{(0)}}{\partial \beta}\dot{\Sigma}^{(1)} + \dfrac{\partial}{\partial \beta}\Sigma^{(1)} + \dfrac{\partial}{\partial \beta}\beta^2 \\
                        &= 2\beta\dot{\Sigma}^{(1)} + 2\beta \\
                        &= 2\beta\left(\dot{\Sigma}^{(1)} + 1\right) \\
                    \end{aligned}\\[1em]
                    \dfrac{\partial v^{(2)}}{\partial \beta} &= 
                    \begin{aligned}[t]
                        \dfrac{\partial \Theta^{(2)}}{\partial \beta} &= \dfrac{\partial}{\partial \beta}\left(\Theta^{(1)}\dot{\Sigma}^{(2)} + \Sigma^{(2)} + \beta^2\right)\\
                        &= \dfrac{\partial\Theta^{(1)}}{\partial \beta}\dot{\Sigma}^{(2)} + \dfrac{\partial}{\partial \beta}\Sigma^{(2)} + \dfrac{\partial}{\partial \beta}\beta^2 \\
                        &= 2\beta\left(\dot{\Sigma}^{(1)} + 1\right)\dot{\Sigma}^{(2)} + 2\beta \\
                        &= 2\beta(\dot{\Sigma}^{(1)}\dot{\Sigma}^{(2)} + \dot{\Sigma}^{(2)} + 1)
                    \end{aligned}\\[1em]
                    \dfrac{\partial v^{(3)}}{\partial \beta} &= 
                    \begin{aligned}[t]
                        \dfrac{\partial \Theta^{(3)}}{\partial \beta} &= \dfrac{\partial}{\partial \beta}\left(\Theta^{(2)}\dot{\Sigma}^{(3)} + \Sigma^{(3)} + \beta^2\right)\\
                        &= \dfrac{\partial\Theta^{(2)}}{\partial \beta}\dot{\Sigma}^{(3)} + \dfrac{\partial}{\partial \beta}\Sigma^{(3)} + \dfrac{\partial}{\partial \beta}\beta^2 \\
                        &= 2\beta\left(\dot{\Sigma}^{(1)}\dot{\Sigma}^{(2)} + \dot{\Sigma}^{(2)} + 1\right) \dot{\Sigma}^{(3)} + 2\beta \\
                        &= 2\beta\left(\dot{\Sigma}^{(1)}\dot{\Sigma}^{(2)}\dot{\Sigma}^{(3)} + \dot{\Sigma}^{(2)}\dot{\Sigma}^{(3)} + \dot{\Sigma}^{(3)} + 1\right)\\
                        &= 2\beta\left(\prod_{i=1}^3\dot{\Sigma}^{(i)} + \dfrac{\prod_{i=1}^3\dot{\Sigma}^{(i)}}{\dot{\Sigma}^{(1)}} + \dfrac{\prod_{i=1}^3\dot{\Sigma}^{(i)}}{\dot{\Sigma}^{(1)}\dot{\Sigma}^{(2)}} + \dfrac{\prod_{i=1}^3\dot{\Sigma}^{(i)}}{\prod_{i=1}^3\dot{\Sigma}^{(i)}}\right) \\
                        &= 2\beta \left(\prod_{i=1}^3\dot{\Sigma}^{(i)}\right) \left(1 + \dfrac{1}{\dot{\Sigma}^{(1)}} + \dfrac{1}{\dot{\Sigma}^{(1)}\dot{\Sigma}^{(2)}} + \dfrac{1}{\prod_{i=1}^3\dot{\Sigma}^{(i)}}\right) \\
                        &= 2\beta \left(\prod_{i=1}^3\dot{\Sigma}^{(i)}\right) \left(1 + \sum_{i=1}^3 \left(\prod_{j=1}^i\dot{\Sigma}^{(j)}\right)^{-1}\right)
                    \end{aligned}\\
                    &\hspace{8em}\vdots \\
                    \dfrac{\partial v^{(L)}}{\partial \beta} &= \dfrac{\partial \Theta^{(L)}}{\partial \beta} = 2\beta \left(\prod_{i=1}^L\dot{\Sigma}^{(i)}\right) \left(1 + \sum_{i=1}^L \left(\prod_{j=1}^i\dot{\Sigma}^{(j)}\right)^{-1}\right)
                \end{aligned}
            \end{equation*}
            Thus, the partial derivative of $v$ with respect to $\beta$ is established.  This leaves us to establish the full partial derivative of $\ddot{k}_{NTK}(L+1, \beta)$ via the product rule:
            \begin{equation*}
                \begin{aligned}
                    \dfrac{\partial}{\partial\beta}\ddot{k}_{NTK}(L+1, \beta) = \dfrac{\partial}{\partial \beta}\left(uv^{(L)}\right) &= \left(\dfrac{d u}{d \beta}\right)\left(v^{(L)}\right) + \left(u\right)\left(\dfrac{\partial v^{(L)}}{\partial \beta}\right)\\
                \end{aligned}
            \end{equation*}
            \begin{equation*}
                \begin{aligned}
                    &= \left(-\dfrac{2\beta}{(L+1)(\beta^2+1)^{2}}\right)\left(\Theta^{(L)}\right) + \left(\dfrac{1}{(L+1)(\beta^2 + 1)}\right) \left(\dfrac{\partial \Theta^{(L)}}{\partial \beta}\right) \\
                    &= \dfrac{1}{(L+1)(\beta^2 + 1)} \left(-\dfrac{2\beta\Theta^{(L)}}{\beta^2+1} + \dfrac{\partial \Theta^{(L)}}{\partial \beta}\right) \\
                    &= \dfrac{1}{(L+1)(\beta^2 + 1)} \left(\dfrac{\partial \Theta^{(L)}}{\partial \beta} -\dfrac{2\beta\Theta^{(L)}}{\beta^2+1}\right)
                \end{aligned}
            \end{equation*}
            This completes the full derivation of the partial with respect to $\beta$ for the normalized recursive NTK.
        \end{proof}
        
        \chapter{Asymptotics of \texorpdfstring{\boldmath$\beta$}{β}}\label{ch:asymtotics}
        
        In this Appendix, we derive the limit of the NTK's parameter $\beta$ when depth $D=1$.
        
        \begin{prop}
            Let $\ddot{k}_{NTK}(1,\beta)$ be the shallow normalized neural tangent kernel as defined in Equation \eqref{eq:normNTK}.  Then the limit of $\beta\to\infty$ is as follows:
            \begin{equation}
            \lim_{\beta\to\infty} \ddot{k}_{NTK}(1,\beta) = 2 - \dfrac{\arccos\left(\lambda^{(0)}\right)}{\pi}
            \end{equation}
        \end{prop}
        
        \begin{proof}
            We begin with the normalized NTK from Equation \eqref{eq:normNTK} with 1 hidden layer.
            \begin{equation}
                \ddot{k}_{NTK}(1,\beta) = \dfrac{1}{(\beta^2 + 1)} k_{NTK}(1,\beta) = \dfrac{1}{(\beta^2 + 1)} \Theta^{(1)}
            \end{equation}
            Then, by substituting, we get 
            \begin{equation*}
                \begin{aligned}
                    \dfrac{1}{(\beta^2 + 1)} \Theta^{(1)} &= \dfrac{1}{(\beta^2 + 1)}\left(\Theta^{(0)}\dot{\Sigma}^{(1)} + \Sigma^{(1)} + \beta^2\right) \\ 
                    &= \dfrac{1}{(\beta^2 + 1)}\left((\Sigma^{(0)} + \beta^2)\kappa_0\left(\lambda^{(0)}\right) + \kappa_1\left(\lambda^{(0)}\right)\sqrt{x^\top xz^\top z} + \beta^2\right) \\
                    &= \dfrac{(\Sigma^{(0)} + \beta^2)\kappa_0\left(\lambda^{(0)}\right)}{\beta^2 + 1} + \dfrac{\kappa_1\left(\lambda^{(0)}\right)\sqrt{x^\top xz^\top z}}{\beta^2 + 1} + \dfrac{\beta^2}{\beta^2 + 1} \\
                    &= \dfrac{\Sigma^{(0)}\kappa_0\left(\lambda^{(0)}\right)}{\beta^2 + 1} + \dfrac{\beta^2\kappa_0\left(\lambda^{(0)}\right)}{\beta^2 + 1} + \dfrac{\kappa_1\left(\lambda^{(0)}\right)\sqrt{x^\top xz^\top z}}{\beta^2 + 1} + \dfrac{\beta^2}{\beta^2 + 1},
                \end{aligned}
            \end{equation*}
            then by taking the limit with respect to $\beta$ towards infinity:
            \begin{gather*}
                \lim_{\beta\to \infty} \left(\dfrac{\Sigma^{(0)}\kappa_0\left(\lambda^{(0)}\right)}{\beta^2 + 1} + \dfrac{\beta^2\kappa_0\left(\lambda^{(0)}\right)}{\beta^2 + 1} + \dfrac{\kappa_1\left(\lambda^{(0)}\right)\sqrt{x^\top xz^\top z}}{\beta^2 + 1} + \dfrac{\beta^2}{\beta^2 + 1}\right) \\
                = 0 + \kappa_0\left(\lambda^{(0)}\right) + 0 + 1 \\
                = \dfrac{1}{\pi}(\pi - \arccos(\lambda^{(0)})) + 1 \\
                = 1 - \dfrac{\arccos(\lambda^{(0)})}{\pi} + 1 \\
                = 2 - \dfrac{\arccos(\lambda^{(0)})}{\pi}
            \end{gather*}
            thus completing the formulation.
        \end{proof}

    \chapter{Additional Figures and Tables}
    \label{ch:additionalfigstables}
        
    \begin{figure}[p]
        \centering
        \includegraphics[width=\textwidth]{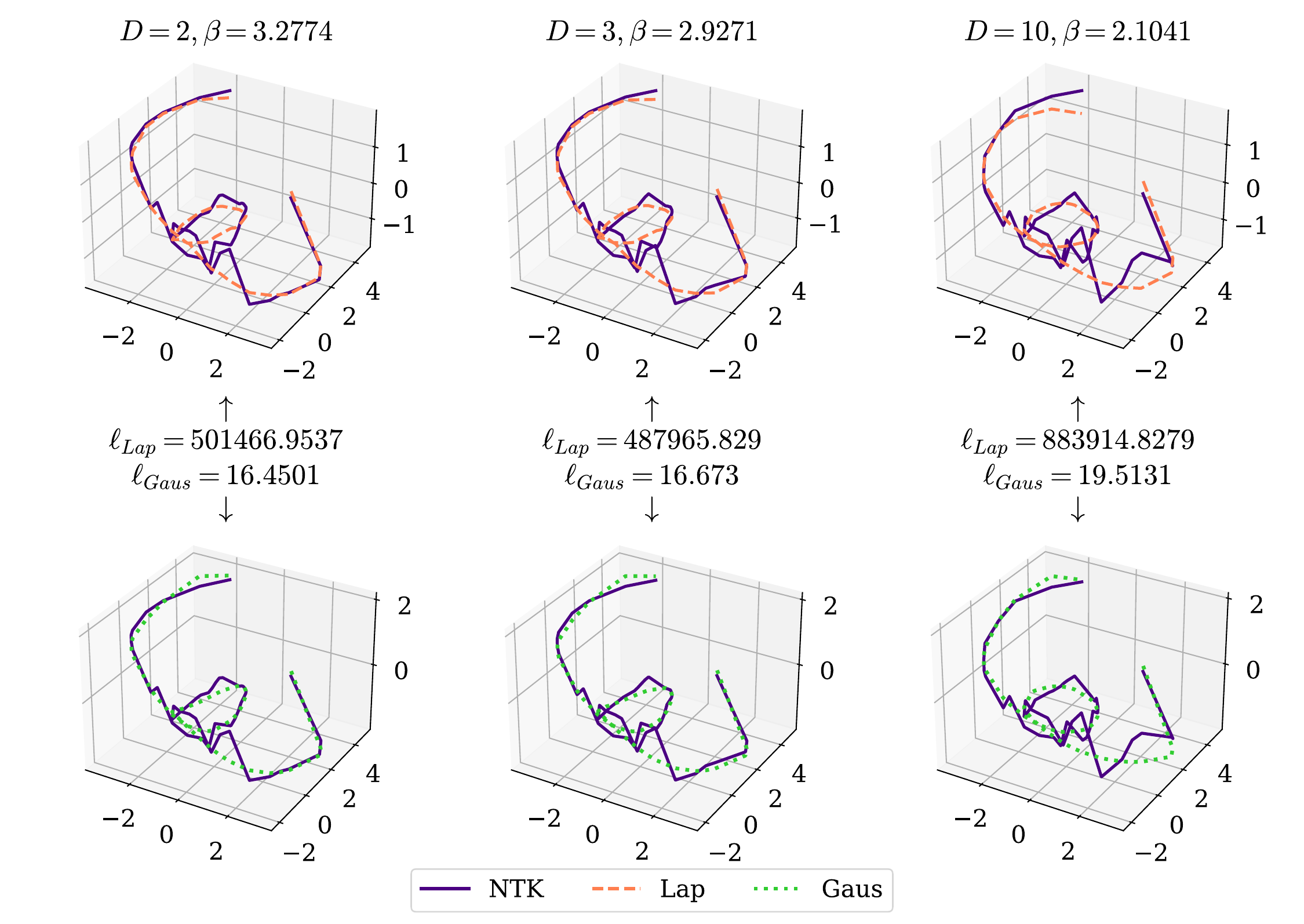}
        \caption{\setlinespacing{1.1} Posterior means generated by fitting to data in $\R^2$ and predicted on out of sample data in $\R^2$.  All kernels seem to be approximating the loop in the curve.  The Laplace and Gaussian kernels provide almost the same predictions between them.  In addition, the kernels seem to do a better job that $\Ss^1$ of approximating the underlying parametric curve.  \textit{Top:} NTK with Laplace kernel overlayed. \textit{Bottom:} NTK with Gaussian kernel overlayed.}
        \label{fig:parametricRd}
    \end{figure}
    
    \begin{figure}[p]
        \centering
        \includegraphics[width=0.45\textwidth]{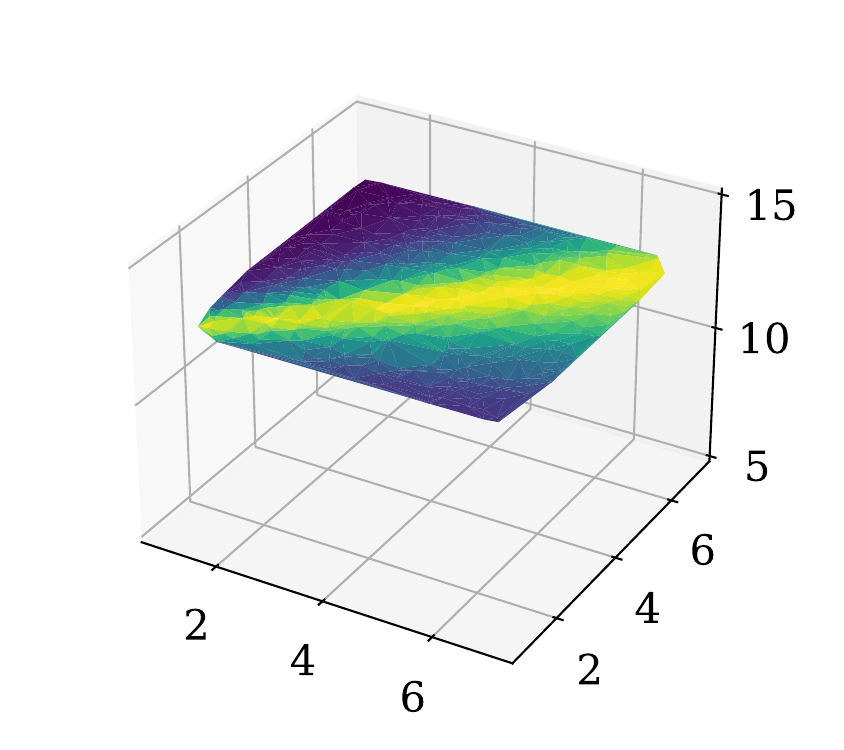}
        \caption{\setlinespacing{1.1} NTK posterior mean of the noisy the Ackley function in $\Ss^1$ for NTK depth $D=2$. The GP was trained using $n=500$ inputs $x_1, x_2\in [1, 7]$ generated using Latin hypercube sampling. The $y$ values are all concentrated around $\approx12.67$ with a difference between the minimum and maximum being $\approx 10^{-8}$ indicating that the posterior mean has zeroed out.  This is due to the kernel's constant value optimizing close to zero.  Attempting to manually fit the GP while controlling the constant value and white noise provides similar results.}
        \label{fig:2dSdNoisyD2}
    \end{figure}
    
    \begin{table}[p]
        \centering
        \begin{tabular}{|c|c|c||c|c||c|c||c|c|}
        \cline{4-9}
        \multicolumn{3}{c|}{} & \multicolumn{2}{c||}{$D=2$} & \multicolumn{2}{c||}{$D=3$} & \multicolumn{2}{c|}{$D=10$} \\
        \hline
        Metrics & Dataset & Noise & Lap & Gaus & Lap & Gaus & Lap & Gaus \\
        \hline
        \multirow[c]{6}{*}{RMSE} & \multirow[c]{2}{*}{Ackley} & No &  $\approx 0$ & 1.7505 & 0.0001 & 1.7505 & 0.0001 & 1.7504 \\
        &  & Yes &  $\approx 0$ &  $\approx 0$ &  $\approx 0$ &  $\approx 0$ &  $\approx 0$ &  $\approx 0$ \\
        \cline{2-9}
        & \multirow[c]{2}{*}{Franke} & No & 0.0101 & 0.1413 & 0.0030 & 0.1412 & 0.0001 & 0.1408 \\
        &  & Yes & 0.0046 & 0.0166 & 0.0028 & 0.0151 & 0.0031 & 0.0155 \\
        \cline{2-9}
        & \multirow[c]{2}{*}{Nonpoly} & No &  $\approx 0$ & 1.5981 & $\approx 0$ & 1.5975 & 0.0001 & 1.5955 \\
        &  & Yes & 0.0216 & 0.1924 & 0.0141 & 0.1866 & 0.0137 & 0.1685 \\
        \hline
        \hline
        \multirow[c]{6}{*}{$\rho$} & \multirow[c]{2}{*}{Ackley} & No & $\approx 1$ & 0.3526 & $\approx 1$ & 0.3526 & $\approx 1$ & 0.3526 \\
        &  & Yes & 0.9926 & 0.9028 & 0.9968 & 0.8939 & 0.9985 & 0.8738 \\
        \cline{2-9}
        & \multirow[c]{2}{*}{Franke} & No & 0.9990 & 0.7782 & 0.9999 & 0.7757 & $\approx 1$ & 0.7766 \\
        &  & Yes & 0.9997 & 0.9943 & 0.9999 & 0.9952 & $\approx 1$ & 0.9951 \\
        \cline{2-9}
        & \multirow[c]{2}{*}{Nonpoly} & No & $\approx 1$ & 0.7492 & $\approx 1$ & 0.7495 & $\approx 1$ & 0.7502 \\
        &  & Yes & 0.9999 & 0.9871 & 0.9999 & 0.9878 & $\approx 1$ & 0.9911 \\
        \hline
        \end{tabular}
        \caption{\setlinespacing{1.1} Posterior mean matching results for the 2D input surface datasets in $\Ss^1$.  As noted in Section \ref{sec:2D}, the noisy the Ackley function was difficult to properly fit resulting in a constant and meaningless posterior thus the zero RMSE should be looked at skeptically.}
        \label{tab:2dSd}
    \end{table}
    
    \begin{table}[p]
        \centering
        \begin{tabular}{|c|c||c|c|c|c||c|c|c|c|}
        \cline{3-10}
        \multicolumn{2}{c|}{} & \multicolumn{4}{c||}{Non-noisy} & \multicolumn{4}{c|}{Noisy} \\
        \hline
        $D$ & \small Sp. & None & \small $X_{rescale}$ & \small $\mathbf{y}_{rescale}$ & Both & None & \small $X_{rescale}$ & \small $\mathbf{y}_{rescale}$ & Both \\
        \hline
        \multirow[c]{2}{*}{2} & $\Ss^9$ & -0.069 & 0.852 & -0.071 & 0.855 & 0.136 & 0.849 & 0.126 & 0.852 \\
        & $\R^{10}$ & 0.229 & 0.855 & 0.240 & 0.935 & 0.232 & 0.852 & 0.242 & 0.935 \\
        \hline
        \multirow[c]{2}{*}{3} & $\Ss^9$ & -0.070 & 0.851 & -0.071 & 0.855 & 0.137 & 0.850 & 0.127 & 0.852 \\
        & $\R^{10}$ & 0.226 & 0.927 & 0.235 & 0.933 & 0.229 & 0.925 & 0.238 & 0.932 \\
        \hline
        \multirow[c]{2}{*}{10} & $\Ss^9$ & -0.072 & 0.815 & -0.073 & 0.822 & 0.128 & 0.814 & 0.129 & 0.821 \\
        & $\R^{10}$ & 0.218 & 0.810 & 0.221 & 0.894 & 0.221 & 0.873 & 0.224 & 0.892 \\
        \hline
        \end{tabular}
        \caption{\setlinespacing{1.1} Friedman 2 $R^2$ results for NTK posterior means with training done using various data transformations.}
        \label{tab:friedman2NTKr2}
        
        \vspace{0.75em}
        
        \begin{tabular}{|c|c||c|c|c|c||c|c|c|c|}
        \cline{3-10}
        \multicolumn{2}{c|}{} & \multicolumn{4}{c||}{Non-noisy} & \multicolumn{4}{c|}{Noisy} \\
        \hline
        $D$ & \small Sp. & None & \small $X_{rescale}$ & \small $\mathbf{y}_{rescale}$ & Both & None & \small $X_{rescale}$ & \small $\mathbf{y}_{rescale}$ & Both \\
        \hline
        \multirow[c]{2}{*}{2} & $\Ss^9$ & -0.198 & 0.023 & -0.198 & 0.033 & 0.063 & 0.044 & 0.075 & 0.117 \\
        & $\R^{10}$ & -0.187 & 0.667 & -0.187 & 0.692 & 0.057 & 0.317 & 0.058 & 0.378 \\
        \hline
        \multirow[c]{2}{*}{3} & $\Ss^9$ & -0.200 & 0.040 & -0.200 & 0.055 & 0.064 & 0.025 & 0.066 & 0.125 \\
        & $\R^{10}$ & -0.191 & 0.659 & -0.190 & 0.696 & 0.051 & 0.313 & 0.117 & 0.374 \\
        \hline
        \multirow[c]{2}{*}{10} & $\Ss^9$ & -0.211 & 0.209 & -0.211 & 0.260 & 0.024 & 0.039 & 0.029 & 0.177 \\
        & $\R^{10}$ & -0.195 & 0.532 & -0.207 & 0.735 & 0.022 & 0.280 & 0.025 & 0.431 \\
        \hline
        \end{tabular}
        \caption{\setlinespacing{1.1} Friedman 3 $R^2$ results for NTK posterior means with training done using various data transformations.}
        \label{tab:friedman3NTKr2}
    \end{table}
    
    \begin{table}[p]
        \centering
        \begin{tabular}{|c||c|c|c|c||c|c|c|c|}
        \cline{2-9}
        \multicolumn{1}{c|}{} & \multicolumn{4}{c||}{Non-noisy} & \multicolumn{4}{c|}{Noisy} \\
        \hline
        $D$ & N/a & \small $X_{rescale}$ & \small $\mathbf{y}_{rescale}$ & Both & N/a & \small $X_{rescale}$ & \small $\mathbf{y}_{rescale}$ & Both \\
        \hline
        2 & $\approx 1$ & 0.9995 & $\approx 1$ & 0.9995 & 0.9999 & 0.9993 & 0.9998 & 0.9993 \\
        3 & $\approx 1$ & 0.9998 & $\approx 1$ & 0.9998 & $\approx 1$ & 0.9998 & 0.9999 & 0.9997 \\
        10 & $\approx 1$ & 0.9996 & $\approx 1$ & $\approx 1$ & $\approx 1$ & 0.9996 & $\approx 1$ & $\approx 1$ \\
        \hline
        \end{tabular}
        \caption{\setlinespacing{1.1} Friedman 2 $\rho$ results for Laplace kernel and NTK posterior mean matching in $\Ss^9$ with training done using various data transformations.}
        \label{tab:friedman2NTKcorr}
        
        \vspace{0.75em}
        
        \begin{tabular}{|c||c|c|c|c||c|c|c|c|}
        \cline{2-9}
        \multicolumn{1}{c|}{} & \multicolumn{4}{c||}{Non-noisy} & \multicolumn{4}{c|}{Noisy} \\
        \hline
        $D$ & N/a & \small $X_{rescale}$ & \small $\mathbf{y}_{rescale}$ & Both & N/a & \small $X_{rescale}$ & \small $\mathbf{y}_{rescale}$ & Both \\
        \hline
        2 & $\approx 1$ & 0.9993 & $\approx 1$ & 0.9993 & $\approx 1$ & 0.9997 & $\approx 1$ & 0.9980 \\
        3 & $\approx 1$ & 0.9999 & $\approx 1$ & 0.9999 & $\approx 1$ & 0.9999 & $\approx 1$ & 0.9991 \\
        10 & $\approx 1$ & 0.9983 & $\approx 1$ & 0.9997 & $\approx 1$ & 0.9953 & $\approx 1$ & 0.9993 \\
        \hline
        \end{tabular}
        \caption{\setlinespacing{1.1} Friedman 3 $\rho$ results for Laplace kernel and NTK posterior mean matching in $\Ss^9$ with training done using various data transformations.}
        \label{tab:friedman3NTKcorr}
    \end{table}
    
    \begin{figure}[p]
    \vspace{-2em}
        \centering
        \includegraphics[width=0.9\textwidth]{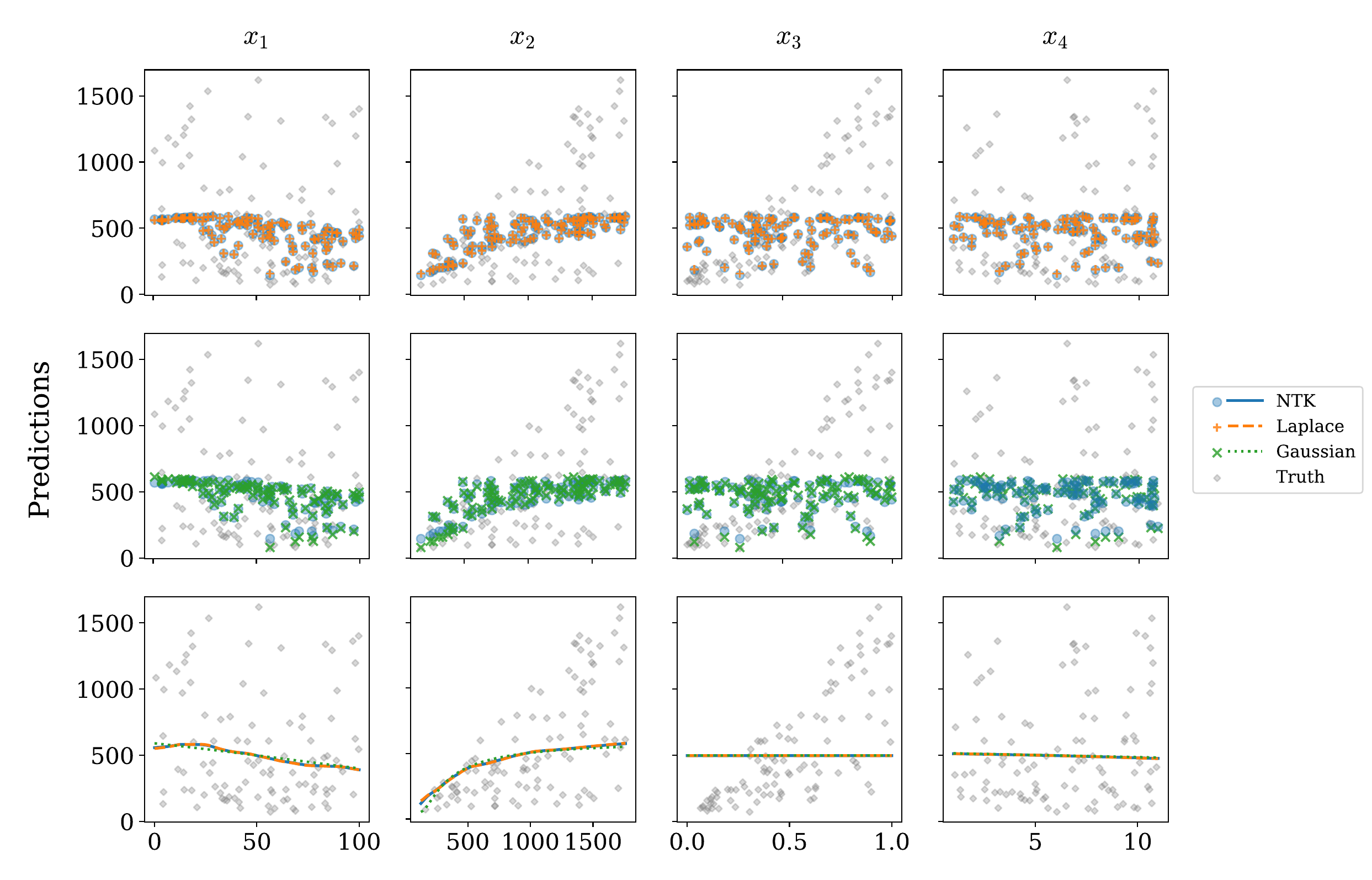}
        \caption{\setlinespacing{1.1} Predictions for noisy Friedman 2 in $\Ss^3$ for $D=2$ with $\mathbf{y}_{rescale}$. \textit{Top:} NTK and Laplace predictions overlayed. \textit{Middle:} NTK and Gaussian predictions overlayed. \textit{Bottom:} Averaged prediction plots of all kernels.}
        \label{fig:friedman2noisySdy}
        
        \vspace{0.75em}
        
        \includegraphics[width=0.9\textwidth]{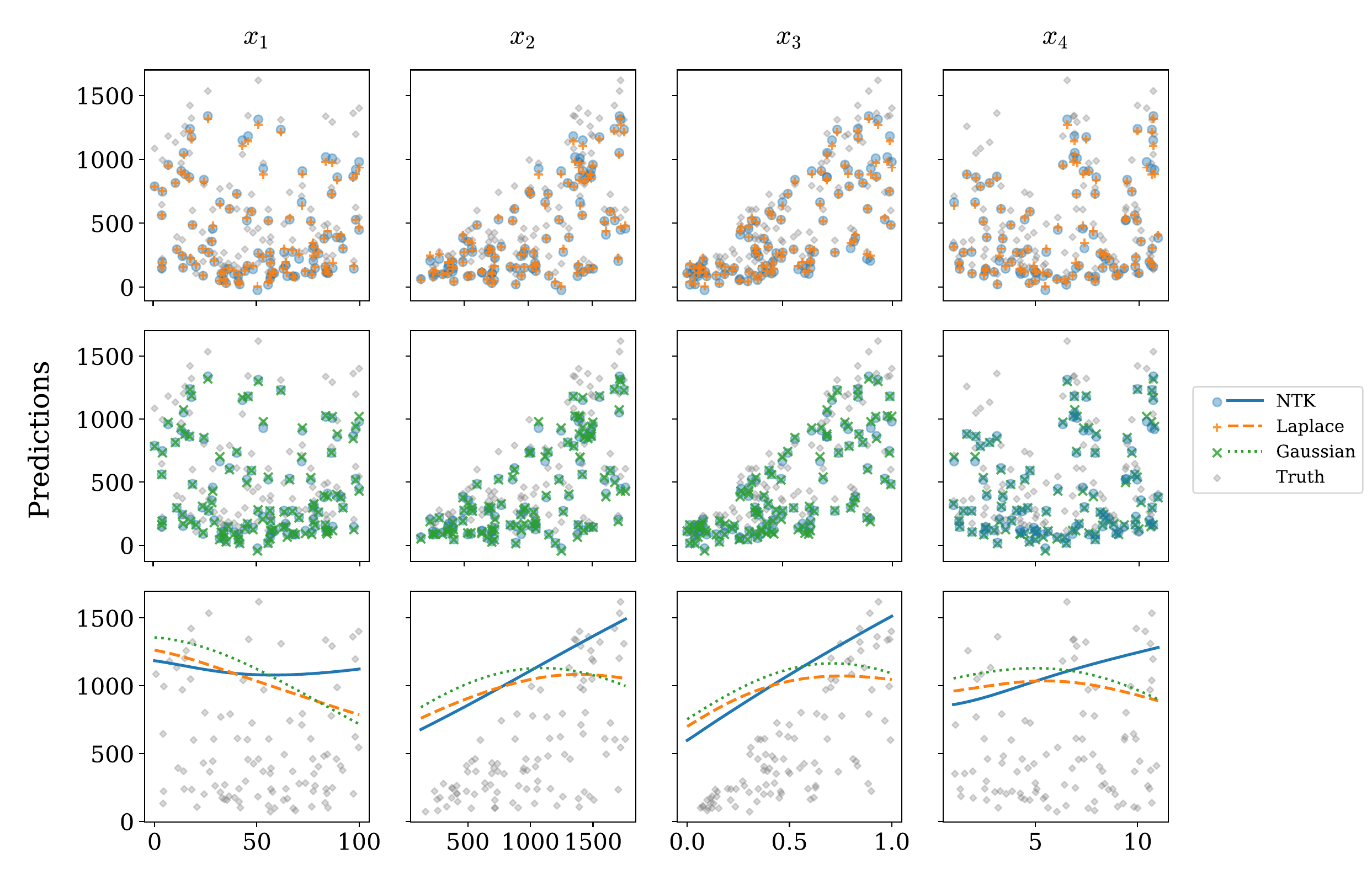}
        \caption{\setlinespacing{1.1} Predictions for noisy Friedman 2 in $\Ss^3$ for $D=2$ with $X_{rescale}$ and $\mathbf{y}_{rescale}$.  The rows of the figure are laid out as in Figure \ref{fig:friedman2noisySdy}.}
        \label{fig:friedman2noisySdXy}
    \end{figure}

    \begin{figure}[p]
        \vspace{-2em}
        \centering
        \includegraphics[width=0.9\textwidth]{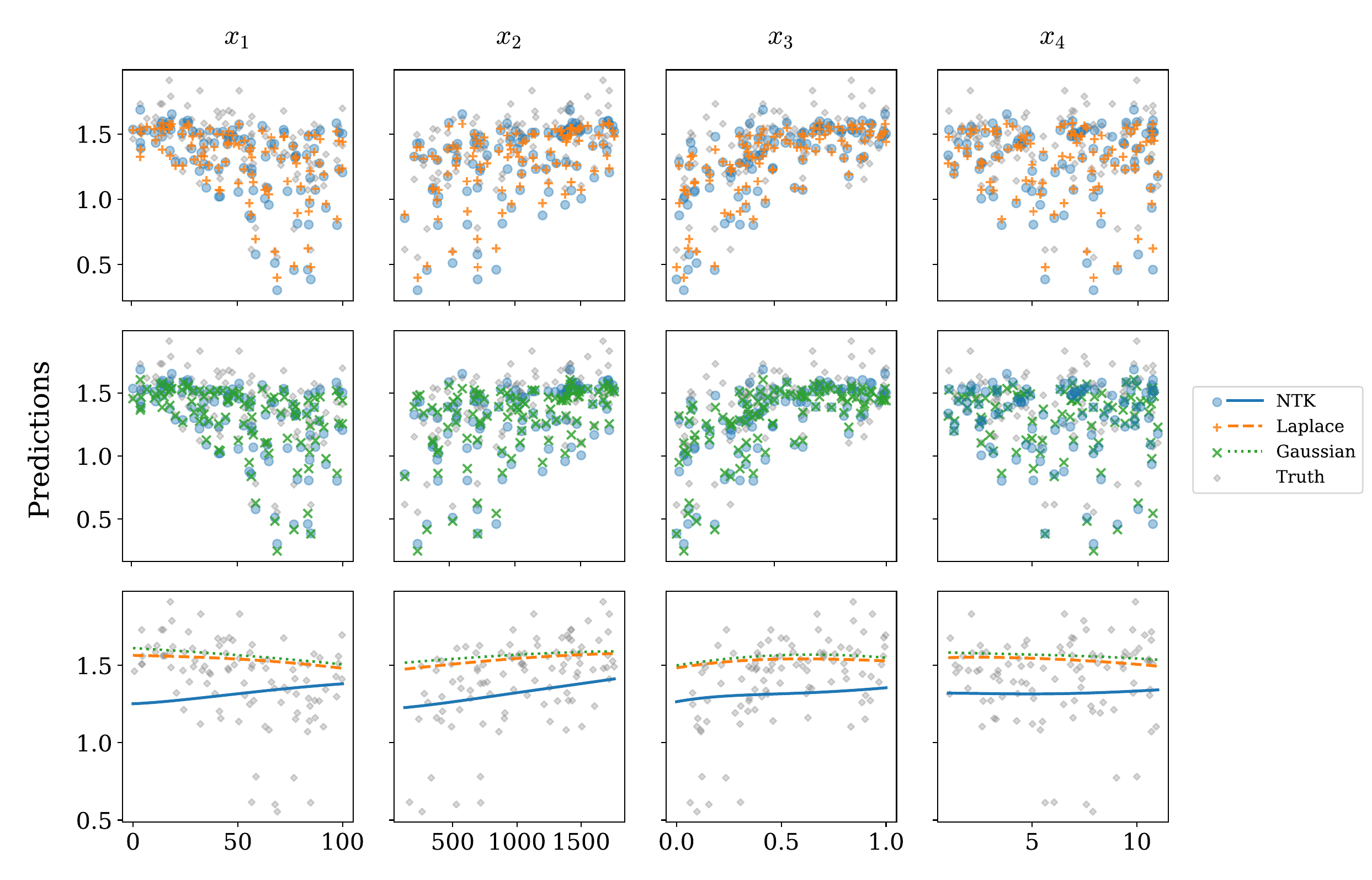}
        \caption{\setlinespacing{1.1} Predictions for noisy Friedman 3 in $\R^4$ for $D=2$ with $X_{rescale}$ and $\mathbf{y}_{rescale}$. The rows of the figure are laid out as in Figure \ref{fig:friedman2noisySdy}.}
        \label{fig:friedman3noisyRdBoth}
        
        \vspace{0.75em}
        
        \includegraphics[width=0.9\textwidth]{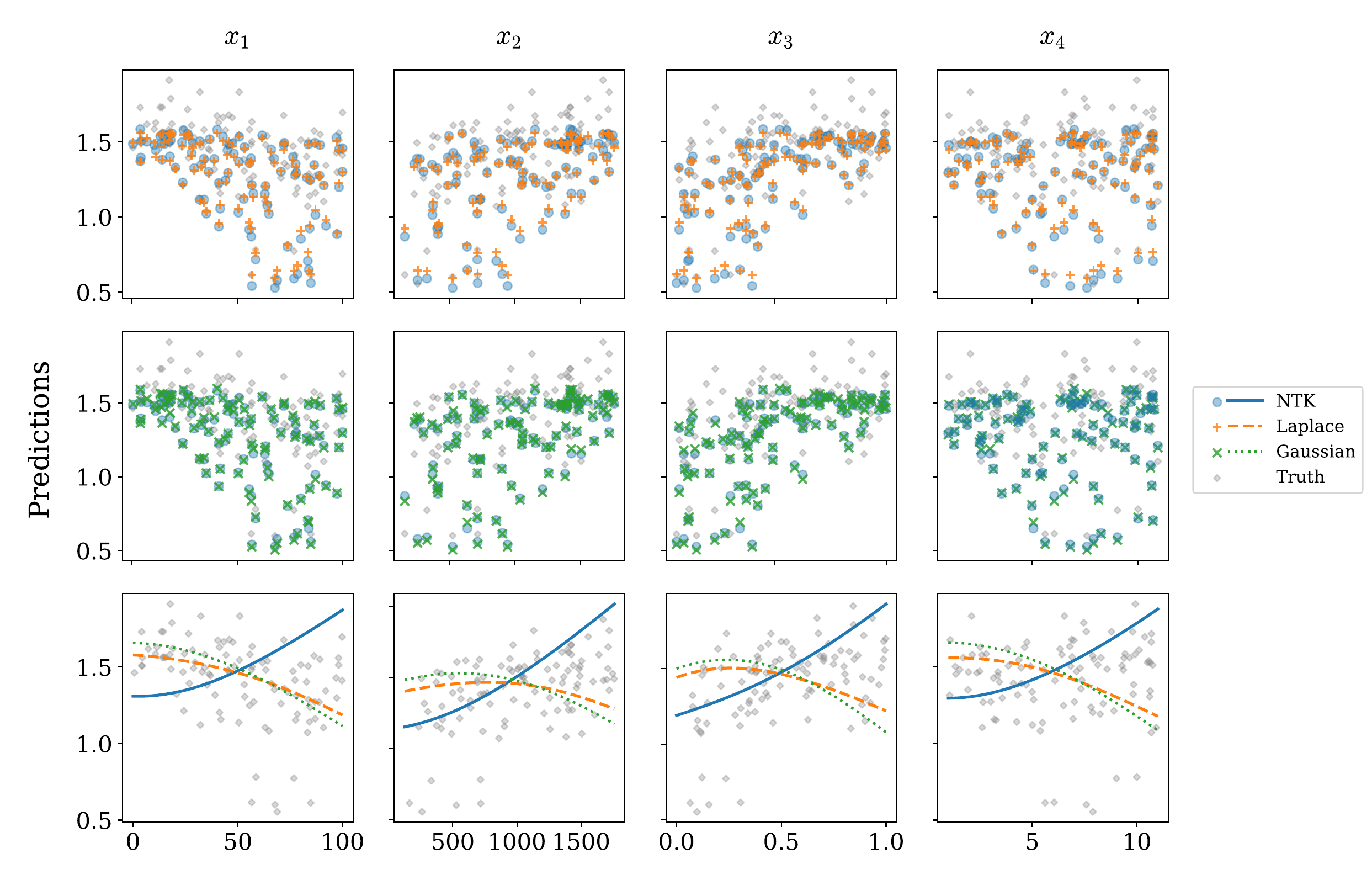}
        \caption{\setlinespacing{1.1} Predictions for noisy Friedman 3 in $\Ss^3$ for $D=2$ with $X_{rescale}$ and $\mathbf{y}_{rescale}$. The rows of the figure are laid out as in Figure \ref{fig:friedman2noisySdy}.}
        \label{fig:friedman3noisySdBoth}
    \end{figure}
    
    \begin{figure}[p]
        \centering
        \includegraphics[width=\textwidth]{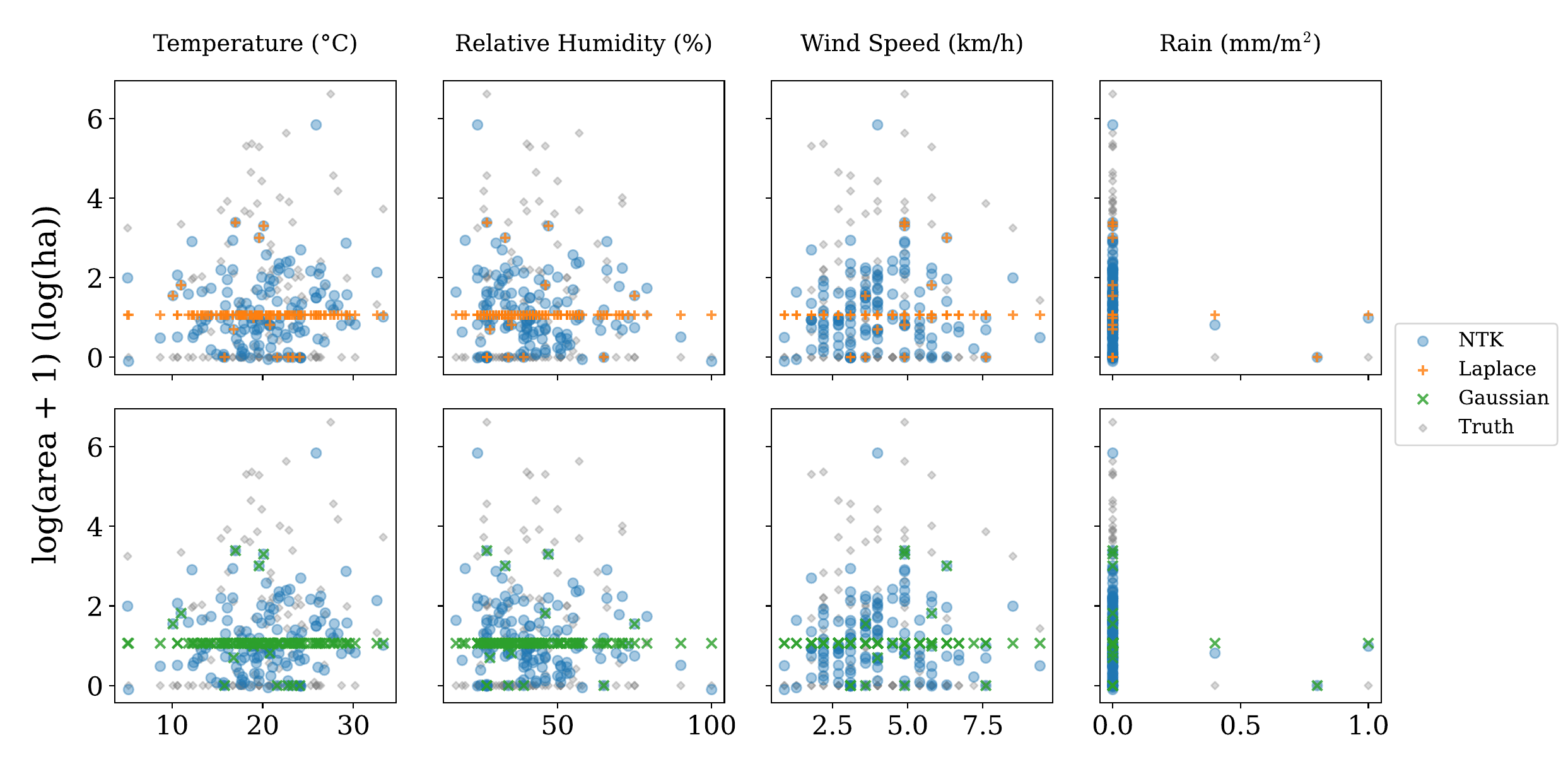}
        \caption{\setlinespacing{1.1} Fire area predictions over $\Ss^3$ without white noise term with NTK $D=10$.  Inputs are shown in $\R^4$ and output is log-transformed for visualization.  This is interesting since the NTK seems to do a better job of aligning the predictions to the ground truth in comparison to the GPs fit with a white noise term in Figure \ref{fig:firesNormTNoiseT}.}
        \label{fig:firesNormTNoiseF}
    \end{figure}
    
    \end{appendices}
\end{document}